\documentclass[11pt]{article}

\usepackage[T1]{fontenc}
\usepackage[utf8]{inputenc}
\providecommand{\DeclareUnicodeCharacter}[2]{}
\DeclareUnicodeCharacter{00A0}{~}
\DeclareUnicodeCharacter{00B1}{$\pm$}
\DeclareUnicodeCharacter{2013}{--}
\DeclareUnicodeCharacter{2014}{---}
\DeclareUnicodeCharacter{2018}{`}
\DeclareUnicodeCharacter{2019}{'}
\DeclareUnicodeCharacter{2212}{-}

\usepackage[numbers]{natbib}
\usepackage[top=1in,left=1in,right=1in,bottom=1in]{geometry}
\usepackage{microtype}
\usepackage{graphicx}
\usepackage{booktabs}
\usepackage{multirow}
\usepackage{array}
\usepackage{threeparttable}
\usepackage[table]{xcolor}
\usepackage{caption}
\usepackage{subcaption}
\usepackage{float}
\usepackage{enumitem}
\usepackage{algorithm}
\usepackage{algorithmic}
\usepackage{amsmath}
\usepackage{amsfonts}
\usepackage{amssymb}
\usepackage{mathtools}
\usepackage{amsthm}
\usepackage{xspace}
\usepackage[normalem]{ulem}
\usepackage{authblk}
\usepackage[colorlinks,linkcolor=magenta,filecolor=blue,citecolor=blue,urlcolor=blue]{hyperref}
\usepackage[capitalize,noabbrev]{cleveref}
\usepackage{placeins}

\theoremstyle{plain}
\newtheorem{theorem}{Theorem}[section]
\newtheorem{proposition}[theorem]{Proposition}
\newtheorem{lemma}[theorem]{Lemma}

\theoremstyle{definition}

\newtheorem{assumption}[theorem]{Assumption}
\theoremstyle{remark}

\makeatletter
\@addtoreset{theorem}{subsection}
\makeatother

\newcommand{\EffRank}{\mathrm{EffRank}}
\newcommand{\lmhead}{\mathrm{lm\_head}}

\title{Sketching the Readout of Large Language Models for Scalable Data Attribution and Valuation}
\date{April 20, 2026}

\author[1]{Yide Ran\thanks{Contact: \href{mailto:yran1@stevens.edu}{\texttt{yran1@stevens.edu}}}}
\author[4]{Jianwen Xie}
\author[2]{Minghui Wang}
\author[3]{Wenjin Zheng}
\author[1]{Denghui Zhang}
\author[4]{Chuan Li}
\author[1]{Zhaozhuo Xu\thanks{Contact: \href{mailto:zxu79@stevens.edu}{\texttt{zxu79@stevens.edu}}}}
\affil[1]{Stevens Institute of Technology}
\affil[2]{Columbia University Mailman School of Public Health}
\affil[3]{University of Texas Health Science Center at Houston}
\affil[4]{Lambda Inc.}

\begin{document}

\maketitle

\begin{abstract}
Data attribution and valuation are critical for understanding data-model synergy for Large Language Models (LLMs), yet existing gradient-based methods suffer from scalability challenges on LLMs. Inspired by human cognition, where decision making relies on a focused readout of relevant memories rather than replaying all pathways, we introduce \textbf{RISE} (\textbf{Readout Influence Sketching Estimator}). Instead of computing and indexing gradients across the entire LLM, RISE focuses on influence hotspots at the output layer, where influence signals concentrate, and the gradient admits a decomposed outer-product form. This enables a dual-channel representation combining a lexical residual channel (RH) and a semantic projected-error channel (GH). Applying CountSketch projections to these channels achieves strong compression while maintaining accurate attribution. Across the OLMo (1B–32B) and Pythia (14M–6.9B) families, RISE reduces index storage by up to 112$\times$ compared to RapidIn and scales to 32B parameters LLM, where gradient-based baselines such as RapidIn and ZO-Inf become memory-infeasible. We evaluate RISE on two paradigms: (1) retrospective attribution, retrieving influential training examples for specific predictions, and (2) prospective valuation, scoring candidate data utility zero-shot. We validate RISE on three tasks: Howdy backdoor data detection, Finance-Medical domain separation, and Brain Rot high-quality data selection. In a closed-loop Brain Rot study, continued pretraining on RISE-selected data yields consistent downstream improvements. Overall, RISE provides a practical and scalable primitive for influence analysis and training-data selection in modern large language models.
\end{abstract}

\section{Introduction}

Understanding which training examples, past or future, shape a large language model's behavior is central to interpretability and to the design of efficient continuous training pipelines.
Human learning exhibits a useful duality: we can trace influential past experiences and predict the utility of new information. Scalable influence estimation for LLMs should support the same duality, enabling both \textbf{retrospective attribution} and \textbf{prospective valuation}.
While classic influence functions~\cite{koh2020understandingblackboxpredictionsinfluence,mises1947asymptotic} offer a principled formulation, they become intractable at LLM scale due to second-order computation and the enormous parameter dimensionality.

Scaling influence estimation for LLM faces a fundamental tension between \textbf{feasibility} and \textbf{fidelity}.
First-order surrogates, including TraceIN~\cite{pruthi2020estimatingtrainingdatainfluence}, that rely on per-example parameter-gradient materialization require computing and storing $O(P)$ gradient values per example, making it infeasible to index large corpora.
Proxy-based approximations~\cite{pan2025alinfiklearningapproximatelinearized} circumvent this cost by training smaller auxiliary models, but their gradients may be misaligned with the target model's training dynamics due to architectural and loss-landscape differences, leading to unreliable attribution.

This tension raises a natural research question: 

\begin{center}
\vspace{-1em}
    \emph{Can we achieve scalable data attribution and valuation using only forward passes and partial parameters?}
\vspace{-1em}
\end{center}

We answer this question affirmatively.
Through systematic measurements of LLM training dynamics, we discover that the \textbf{readout layer} (LM head) is precisely the component we seek.
Influence-relevant gradients naturally concentrate in output-side readout components: the \textbf{ratio} of LM-head gradient energy to the average transformer-layer gradient energy grows significantly with scale (up to $27.8\times$ on OLMo-32B; \Cref{fig:all_model_energy}).
We also observe a scale-dependent \emph{U-shaped} discriminativeness pattern: middle-layer hidden-state similarity collapses but recovers at the final layer for practical-scale models (Appendix Figure~\ref{fig:obs2_discriminativeness}).
Crucially, the LM-head gradient admits an exact outer-product factorization:
$\nabla_{W_{\lmhead}} \ell_t = r_t h_t^\top$,
where both factors can be computed from forward-pass quantities alone, eliminating the need for full-model backpropagation. Moreover, although $r_t$ is vocabulary-sized, its influence-relevant signal is highly concentrated: a tiny active-token support preserves most residual energy and semantic error direction.

\begin{figure*}[!h]
    \vspace{-4pt}
    \centering
    \includegraphics[width=0.9\textwidth]{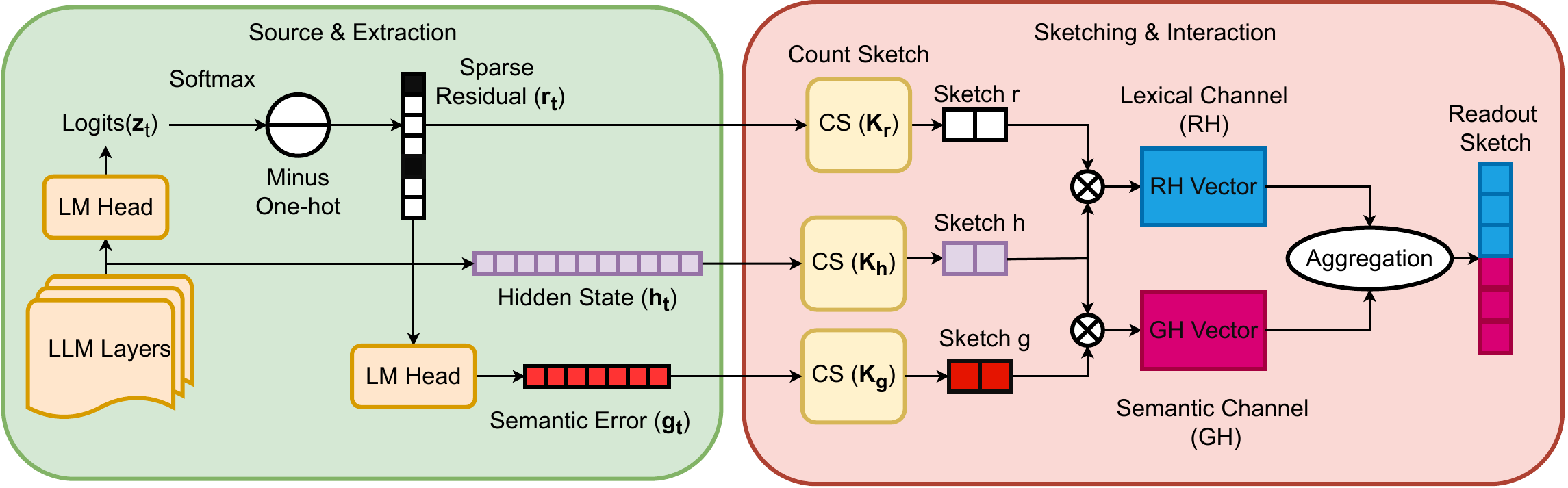}
    \caption{\textit{RISE}: Readout Influence Sketching Estimator .}
    \vspace{-4pt}
    \label{fig:overview}
    \vspace{-6pt}
\end{figure*}

Building on this insight, we introduce \textbf{RISE (Readout Influence Sketching Estimator)}, which enables efficient \textbf{forward-only} influence estimation through CountSketch compression of the outer-product structure.

\textit{RISE} exploits this structure through a \textbf{dual-channel} formulation: a lexical residual channel (\textbf{RH}) for token-level precision and a semantic projected-error channel (\textbf{GH}) for robust semantic matching.
By compressing these factors with CountSketch projections~\cite{DBLP:journals/tcs/CharikarCF04}, \textit{RISE} reduces index storage by $28$ to $112\times$ compared to RapidIn on Pythia-1B and scales to 32B parameters, where per-example gradient indexing baselines such as RapidIn and ZO-Inf become memory-infeasible (\Cref{tab:cost_quality_topk}).
We validate \textit{RISE} across Pythia (14M to 6.9B) and OLMo (1B to 32B) on Howdy backdoor detection, Finance-Medical domain separation, and Brain Rot data selection.
In a closed-loop Brain Rot study, training on RISE-selected data yields substantially improved downstream outcomes, achieving $\mathbf{2.8\times}$ lower perplexity on control data and $\mathbf{4.1\times}$ higher \textbf{RULER-CWE} accuracy compared to baselines (\Cref{tab:brainrot_closed_loop}).  We summarize our contributions as follows.

\begin{itemize}[nosep,leftmargin=*]
\item \textbf{Discovery of readout influence hotspots.} We empirically localize influence-relevant signal concentration to output-side components and provide theoretical justification for this depth-dependent energy shift.
\item \textbf{Dual-channel readout influence.} We design a readout-side influence metric that combines RH, a vocabulary-space residual channel for lexical precision, with GH, an LM-head-projected semantic error channel for robust paraphrase-aware matching.
\item \textbf{Sparse active-token structure in LM-head residuals.} We show that influence-relevant residual energy is sharply concentrated on a tiny adaptive support even when probability mass remains diffuse. Exploiting this structure preserves both RH accuracy and GH semantic fidelity while reducing residual-side computation from vocabulary scale to active-support scale.
\item \textbf{Unified attribution and valuation.} A single estimator supports both retrospective attribution and prospective valuation without retraining proxies or storing large parameter gradients.
\end{itemize}

\section{Readout Sketching of LLMs for Scalable Influence Estimation}
\label{sec:method}

We develop \textit{RISE} (Readout Influence Sketching Estimator), a scalable influence estimation method for LLMs. Our approach is motivated by systematic observations of gradient structure in LLMs, which show that influence signals concentrate in specific outer product channels that can be approximated efficiently.

\subsection{Influence Kernel: The Scalability Challenge}
\label{sec:challenge}
Influence kernel~\cite{koh2020understandingblackboxpredictionsinfluence} estimates how much a training example affects model predictions. For a loss function $\ell$ and model parameters $\theta$, the influence of a training example $x_i$ on a query/test example $x_q$ is:
\begin{equation*}
\mathcal{I}(x_i\!\to\!x_q) = \nabla_\theta \ell(x_q)^\top H_\theta^{-1} \nabla_\theta \ell(x_i),
\end{equation*}
where $H_\theta = \nabla^2_\theta \ell$ is the Hessian of the training loss with respect to $\theta$.

A common first-order surrogate (used by TracIn~\cite{pruthi2020estimatingtrainingdatainfluence}) replaces $H_\theta^{-1}$ with a scaled identity, yielding a gradient inner product:
\begin{equation}
\mathcal{I}(x_i\!\to\!x_q) \approx \left\langle \nabla_\theta \ell(x_q), \nabla_\theta \ell(x_i) \right\rangle .
\label{eq:grad_inner}
\end{equation}

\textbf{Computing the influence kernel, even using first-order surrogates, is infeasible for LLMs.} For an LLM with $P$ parameters, storing per-example gradients requires $O(P)$ memory. Even for moderately sized LLMs with approximately $7$ billion parameters ($P \approx 7 \times 10^{9}$), storing FP16 gradients for $10{,}000$ training examples would require roughly \textbf{140~TB} of memory, making this approach clearly infeasible.

\textbf{Our roadmap.} To overcome this scalability challenge, we avoid materializing full first-order gradients across all LLM parameters. The method is grounded in three observations. First, influence-relevant gradient structure localizes at the LM-head readout, where the gradient admits an exact residual--hidden outer-product form (Observation~\ref{sec:obs1}). Second, the residual side benefits from a dual-channel lexical/semantic decomposition (Observation~\ref{sec:obs2}). Third, the LM-head residual itself is supported by a tiny set of active prediction tokens, enabling sparse residual and sparse GH construction without losing the influence-relevant error signal (Observation~\ref{sec:obs3}).

\subsection{Observation 1: Influence Signals Localize at the Output Readout}
\label{sec:obs1}
\begin{figure}[t]
\centering
\includegraphics[width=0.8\linewidth]{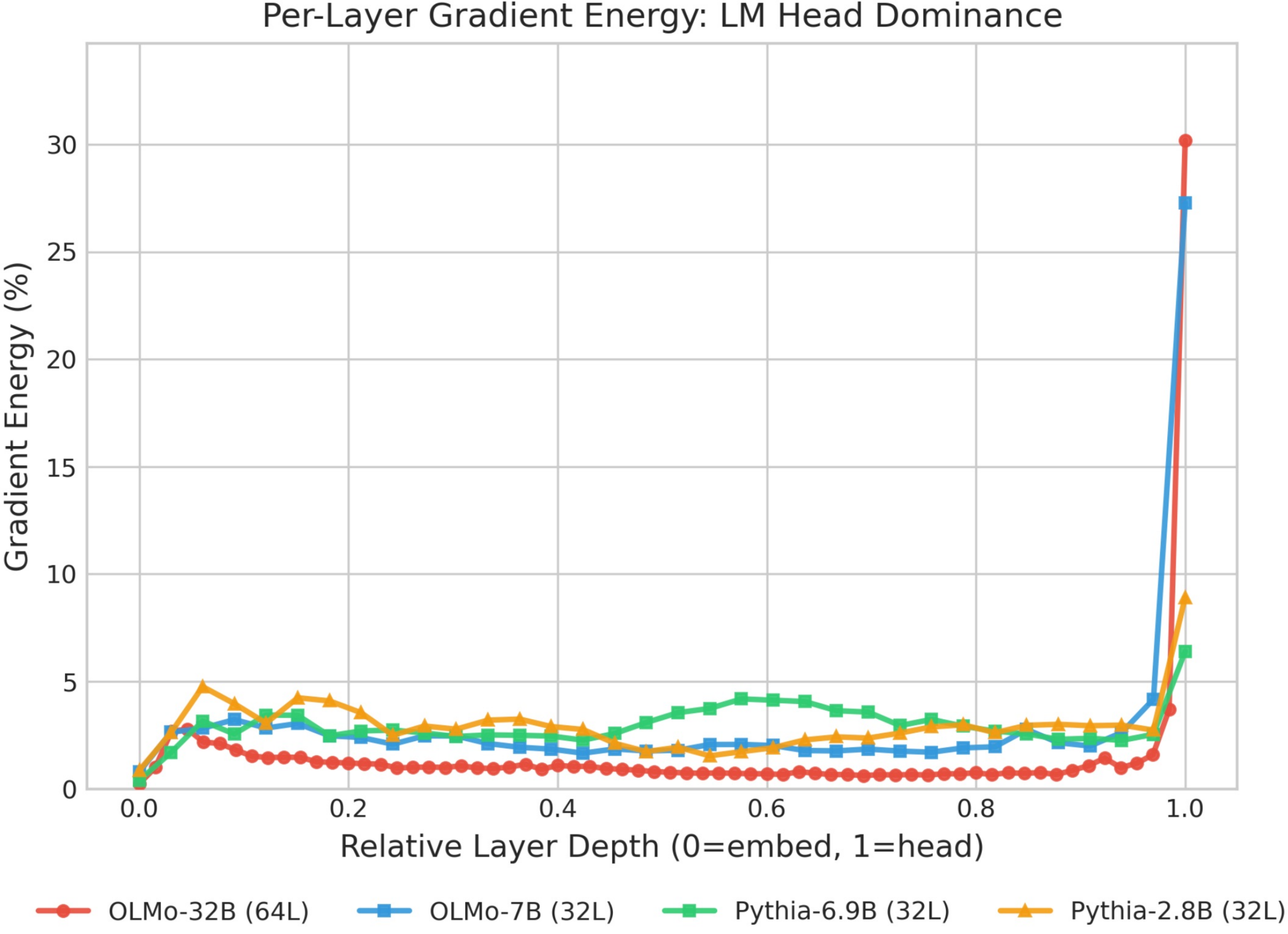}
\caption{Per-layer gradient energy on C4. The LM head forms a sharp peak that strengthens with model scale, while internal transformer layers each contribute only a few percent.}
\label{fig:all_model_energy}
\vspace{-1em}
\end{figure}
We first examine where gradient energy distributes across model layers. For each layer $c$ (embedding, transformer layers, LM head), we compute the relative energy:
\begin{equation*}
E_c = \frac{\|\nabla_{W_c} \ell\|_F^2}{\sum_{c} \|\nabla_{W_{c}} \ell\|_F^2}
\end{equation*}

We conduct controlled analyses on the C4 corpus~\cite{raffel2023exploringlimitstransferlearning} using open LLMs spanning two architecture families. We further compare the same model before and after fine-tuning on downstream task data. Figures~\ref{fig:all_model_energy},~\ref{fig:obs1_pretrain_finetune_same_energy} reveal two key patterns.

\textbf{LM-Head energy hotspot grows with scale.}
On C4, gradient energy concentrates sharply at the LM head across model scales (Figure~\ref{fig:all_model_energy}), forming a dominant peak while internal transformer layers each contribute only a few percent.

\textbf{Stability under downstream fine-tuning.}
Comparing each pretrained checkpoint with its task-fine-tuned checkpoint on BrainRot, Howdy, and Finance--Medical, the per-layer gradient-energy profiles nearly overlap (Figure~\ref{fig:obs1_pretrain_finetune_same_energy}): the layer-wise trend is largely preserved and the LM-head share changes only mildly.
This invariance motivates using pretrained checkpoints to identify LM-Head influence hotspots and supports a single estimator for both prospective valuation and retrospective attribution.

\textbf{Theoretical analysis.} To quantify the LM-head dominance in Figures~\ref{fig:all_model_energy} and~\ref{fig:obs1_pretrain_finetune_same_energy}, we compare $G_{\text{head}}$ to internal layers via the head-to-average ratio $G_{\text{head}} / \big(\frac{1}{L}\sum_{l=1}^L G_l\big)$. This ratio measures how strongly influence signals concentrate at the LM Head relative to the model body. Appendix~\ref{sec:appendix_gradient_energy_theory} shows that under mild stability conditions (norm-stable activations and non-exploding backpropagation), head-to-average energy ratio admits a depth-dependent lower bound that grows with $L$:
\begin{equation*}
\frac{G_{\text{head}}}{\frac{1}{L}\sum_{l=1}^L G_l}
\ \ge\
\frac{C_h}{C_x}\cdot \frac{L(1-\kappa^2)}{\|W_{\lmhead}\|_{op}^2(1-\kappa^{2L})}.
\end{equation*}
This explains why the head-to-average gap increases with model depth, while still allowing occasional internal-layer hotspots.

\textbf{Readout factorization enables a forward-only readout surrogate.}
Energy concentration alone does not guarantee useful retrieval: influence-based retrieval is driven by pairwise similarity between example-specific gradients, not by magnitude alone. The LM head is attractive because, in addition to concentrating gradient energy, its gradient contribution at each prediction position admits an exact factorization whose two factors are both available from forward-pass quantities. As formally derived in Lemma~\ref{lem:outer-product} (Appendix~\ref{sec:appendix_theory}), for the token-level loss at position $t$,
\begin{equation*}
    \nabla_{W_{\lmhead}} \ell(z_t,y_t) = r_t \otimes h_t,
\end{equation*}
where $h_t \in \mathbb{R}^d$ is the final hidden state and $r_t = p_t - \mathbf{1}_{y_t} \in \mathbb{R}^V$ is the prediction residual, with $p_t=\mathrm{softmax}(z_t/\tau)$. Crucially, $h_t$ is the model's last-layer output and $r_t$ is determined by the softmax distribution and the known label, so this readout-side gradient contribution can be formed without per-example backpropagation through the full model. For a full sequence, the LM-head gradient is obtained by summing these token-level contributions across positions:
\begin{equation*}
    \nabla_{W_{\lmhead}} \ell(x)
    =
    \sum_{t=1}^{T} r_t(x)\otimes h_t(x).
\end{equation*}
At each token position, the corresponding LM-head gradient inner product factorizes as
\begin{align*}
    &\left\langle \nabla_{W_{\lmhead}} \ell(z_t(x_q),y_t(x_q)),
    \nabla_{W_{\lmhead}} \ell(z_t(x_i),y_t(x_i)) \right\rangle_F \\
    &\quad =
    \underbrace{\left\langle r_t(x_q), r_t(x_i)\right\rangle}_{\text{residual sim}}
    \cdot
    \underbrace{\left\langle h_t(x_q), h_t(x_i)\right\rangle}_{\text{hidden sim}}.
\end{align*}
This per-position identity exposes two requirements for scalable retrieval: the hidden factor must remain discriminative across candidates, and the residual factor must be represented in a way that is both scalable and semantically meaningful.

\textbf{Final hidden states remain discriminative.}
The factorization above is only useful if the hidden-side similarity $\langle h_t(x_q), h_t(x_i)\rangle$ varies across candidates rather than collapsing to a near-constant scale. We therefore measure discriminativeness as the variance of pairwise cosine similarities over sampled token representations:
\begin{equation*}
\mathrm{Discriminativeness}(h)
=
\mathrm{Var}_{i\neq j}
\left[
\frac{\langle h^{(i)}, h^{(j)}\rangle}
{\|h^{(i)}\|_2\|h^{(j)}\|_2}
\right].
\end{equation*}
Concretely, we sample $N=1000$ token-level training points, compute the $\binom{N}{2}$ pairwise cosine similarities, and report their variance. Appendix Figures~\ref{fig:obs2_discriminativeness} and~\ref{fig:obs2_pretrain_finetune_same_ushape} show a consistent U-shape: pairwise cosine variance drops through middle layers, then recovers near the final layer across Pythia and OLMo, and the same pattern persists after downstream fine-tuning. Observation~\ref{sec:obs1} therefore justifies restricting Eq.~(\ref{eq:grad_inner}) to the LM head,
$\mathcal{I}(x_i\!\to\!x_q)\approx\left\langle \nabla_{W_{\lmhead}}\ell(x_q),\nabla_{W_{\lmhead}}\ell(x_i)\right\rangle_F$,
leaving the high-dimensional residual factor as the remaining question addressed in Observation~\ref{sec:obs2}.

\subsection{Observation 2: Dual-Channel Readout Error Decomposition}
\label{sec:obs2}

Observation~\ref{sec:obs1} established that LM-head similarity factorizes into hidden and residual terms and that the hidden factor remains discriminative; we now address the residual factor.
The residual vector $r_t = p_t - \mathbf{1}_{y_t} \in \mathbb{R}^V$ lives in the vocabulary basis, so the inner product $\langle r_t(x_q), r_t(x_i) \rangle$ is often driven by the overlap of probability mass on the same tokens and does not inherently treat synonyms as similar under a plain dot product.
To obtain an error signal that captures both lexical and semantic mismatches, we propose decomposing the error term into two complementary channels:

\textbf{RH Channel ($r_t \otimes h_t$):} The direct outer product using the vocabulary-space residual $r_t$. This channel tends to favor \emph{token-level} prediction errors, identifying distinct mismatches in the output distribution.

\textbf{GH Channel ($g_t \otimes h_t$):} A projection of the residual into the embedding space via $g_t = W_{\lmhead}^\top r_t \in \mathbb{R}^d$, where $W_{\lmhead} \in \mathbb{R}^{V \times d}$ is the LM head weight matrix and $W_{\lmhead,v} \in \mathbb{R}^d$ denotes its $v$-th row. Expanding this term reveals its semantic nature:
\vspace{-0.2em}
\begin{equation*}
    g_t = W_{\lmhead}^\top (p_t - \mathbf{1}_{y_t})
    = \mathbb{E}_{v\sim p_t}[W_{\lmhead,v}] - W_{\lmhead,y_t}.
\end{equation*}
Here, $g_t$ represents the error direction in embedding space: the vector difference between the model's expected token embedding and the ground truth embedding.

\textbf{Embedding choice.} We use the LM head weights $W_{\lmhead}$ for the projection $g_t = W_{\lmhead}^\top r_t$, as it directly defines the prediction geometry. While input embeddings can also be used, our experiments default to $W_{\lmhead}$.

\begin{table}[t]
\centering
\caption{RH is misled by lexical overlap; GH corrects the ranking. Computed with Pythia-410M; lower rank is better. Query ($x_q$): ``The cat is sleeping on the blanket.'' Sorted by RH rank. RH ranks lexical-only (\colorbox{red!12}{red}) above hard paraphrase (\colorbox{yellow!25}{yellow}). GH: $\uparrow$ rescues semantic, $\downarrow$ demotes spurious.}
\label{tab:rh_gh_ranking}
\small
\setlength{\tabcolsep}{3pt}
\begin{tabular}{@{}lccccc@{}}
\toprule
Candidate ($x_i$) & Sem & Lex & $\mathrm{rank}_{\mathrm{RH}}$ & $\mathrm{rank}_{\mathrm{GH}}$ & $\Delta$ \\
\midrule
Cat sleeps on blanket & \checkmark & \checkmark & 1 & 2 & \\
Kitten napping on throw & \checkmark & & 2 & 1 & \\
Tabby dozed on quilt & \checkmark & & 3 & 3 & \\
\rowcolor{red!12} Sleepy cat, blanket ad & & \checkmark & 4 & 5 & $\downarrow$ \\
\rowcolor{red!12} Cat food, blanket orders & & \checkmark & 5 & 6 & $\downarrow$ \\
\rowcolor{yellow!25} Feline dozing on coverlet & \checkmark & & 6 & 4 & $\uparrow$ \\
Stock market crashed & & & 7 & 7 & \\
\bottomrule
\end{tabular}
\end{table}

\textbf{Controlled experiment.} As shown in Table~\ref{tab:rh_gh_ranking}, GH corrects RH's over-reliance on lexical overlap. Specifically, GH rescues the hard paraphrase (Row 6) from $\mathrm{rank}_{\mathrm{RH}}=6$ to $\mathrm{rank}_{\mathrm{GH}}=4$ via embedding-space alignment, while successfully demoting spurious matches that share high token overlap but distinct meanings (Rows 4--5). This confirms GH's ability to prioritize semantic alignment over mere token identity.

\textbf{Mechanism of complementarity.} The RH channel emphasizes \emph{lexical precision} and tends to prioritize token-level agreement. The GH channel trades off lexical precision for \emph{semantic robustness}. In particular, the GH error-side similarity corresponds to residual matching under an embedding-induced kernel (Eq.~\ref{eq:gh_kernel_identity} in Appendix~\ref{sec:appendix_gh_more_semantic}):
\begin{equation*}
    \langle g_t(x_q), g_t(x_i) \rangle = r_t(x_q)^\top (W_{\lmhead} W_{\lmhead}^\top)\, r_t(x_i).
\end{equation*}
By combining them as $\lambda_{rh} \cdot \text{RH} + \lambda_{gh} \cdot \text{GH}$, we obtain a metric that tends to prioritize exact matches, retrieve paraphrases, and help suppress spurious lexical overlaps.

\textbf{Ablation Study.}
Beyond this controlled example, we validate the complementarity of RH and GH via ablations across model scales.
Appendix~\ref{sec:rh_gh_ablation_study} (Tables~\ref{tab:ablation_medical_predict}--\ref{tab:ablation_medical_recall}) shows that GH consistently outperforms RH on semantic-heavy tasks, while RH+GH is robustly best or competitive across Top-$K$. This dual-channel decomposition alleviates the brittleness of the error term: RH preserves lexical precision, while GH injects embedding-space semantic smoothing. Together with the readout locality and final-layer choice in Observation~\ref{sec:obs1}, this gives \textit{RISE} a gradient similarity metric that is discriminative and semantically robust. The remaining scalability issue is that both RH and GH are still induced by a vocabulary-sized residual $r_t$.

\subsection{Observation 3: Sparse Active Tokens in LM-Head Residuals}
\label{sec:obs3}

Although the RH/GH dual-channel combination improves semantic robustness, both channels remain derived from the vocabulary-sized residual $r_t \in \mathbb{R}^V$, incurring $\mathcal{O}(Vd)$ cost per token. We now show that the influence-relevant signal concentrates on a tiny token support, enabling orders-of-magnitude compression as a structural property rather than an engineering heuristic.

\textbf{Residual energy decouples from probability mass.}
While softmax probability mass diffuses across the long tail, the squared residual energy---which directly drives the RH inner product---is highly concentrated. With fixed $\tau=1.0$ and $K=128$, the sparse support preserves $95.9\%$ of non-ground-truth residual-tail energy and $99.6\%$ of full residual energy while covering only $0.239\%$ of the vocabulary (Figure~\ref{fig:sparse_active_tokens}, left). The gap against probability mass ($85.6\%$) arises because long-tail tokens carry non-negligible probability but negligible squared prediction error.

\textbf{The ground-truth token anchors the residual.}
The ground-truth token $y_t$ is always included in $\mathcal{S}_t$ regardless of its prediction rank, ensuring the supervision-bearing coordinate of $r_t$---typically the largest in magnitude---is never discarded. This is especially important for prospective valuation, where $y_t$ may fall outside the model's top-$K$ predictions on unseen candidate data.

\textbf{Sparse truncation preserves the GH semantic direction.}
Because $g_t = W_{\text{lm\_head}}^{\top} r_t$ is a linear projection, truncating $r_t$ could distort the semantic error direction. However, at fixed $\tau=1.0$ and $K=128$, sparse GH achieves mean cosine similarity $0.993$ against dense GH across Pythia (160M--6.9B) and OLMo pretrained/fine-tuned runs (Figure~\ref{fig:sparse_active_tokens}, right), confirming that this structure is not model- or task-specific. Additional diagnostics are in Appendix~\ref{sec:appendix_sparse_active_tokens}.

\textbf{Sparse token selection.}
Based on these observations, for each position $t$ we define:
\begin{equation*}
\mathcal{S}_t = \mathrm{TopK}(z_t/\tau, K)\cup\{y_t\}.
\end{equation*}
Renormalizing the logits over $\mathcal{S}_t$ yields $\tilde{p}_t$; the sparse residual is $\tilde{r}_t(v) = \tilde{p}_t(v) - \mathbf{1}[v=y_t]$ for $v \in \mathcal{S}_t$. Both channels are then computed on this shared support without materializing the full vocabulary:
\begin{equation*}
\widetilde{\mathrm{RH}}_t = \tilde{r}_t \otimes h_t, 
\qquad 
\widetilde{\mathrm{GH}}_t = \tilde{g}_t \otimes h_t,
\qquad
\tilde{g}_t = \sum_{v\in\mathcal{S}_t}\tilde{p}_t(v)W_{\text{lm\_head},v} - W_{\text{lm\_head},y_t}.
\end{equation*}
Together, the concentrated tail energy, the preserved GT coordinate, and the stable GH direction ensure that sparse truncation reduces residual-side cost from $\mathcal{O}(V)$ to $\mathcal{O}(|\mathcal{S}_t|)$ without sacrificing attribution fidelity.

\begin{figure*}[t]
\centering
\begin{minipage}{0.49\linewidth}
\centering
\includegraphics[width=\linewidth]{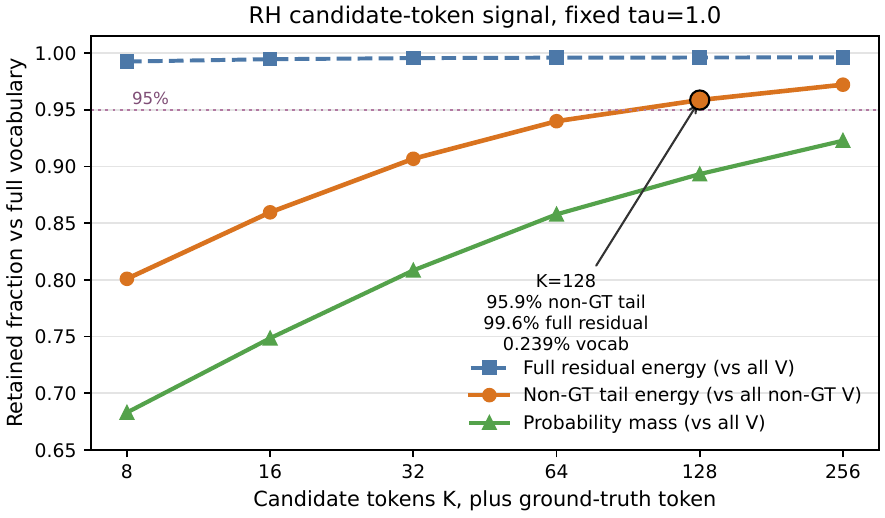}
\end{minipage}\hfill
\begin{minipage}{0.49\linewidth}
\centering
\includegraphics[width=\linewidth]{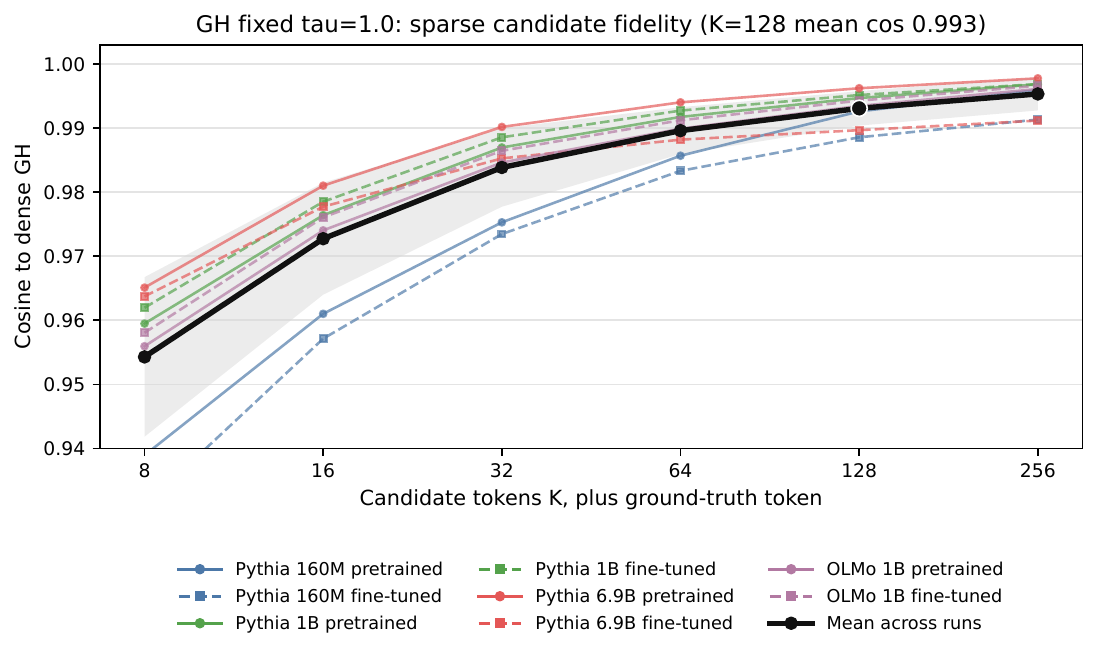}
\end{minipage}
\caption{\textbf{Sparse active tokens in LM-head residuals.}
At fixed $\tau=1.0$, Top-$K$+GT preserves RH residual energy (left) and yields sparse GH directions close to dense GH (right).
Additional fixed-control diagnostics are in Appendix~\ref{sec:appendix_sparse_active_tokens}.}
\label{fig:sparse_active_tokens}
\end{figure*}

\subsection{\textit{RISE}: Scalable Influence via Factorized Sketching}
\label{sec:fiksi}

\textbf{Factorized Compression.}
Storing the exact token-level LM-head gradient factorization
$\nabla_{W_{\lmhead}}\ell(z_t,y_t)= r_t \otimes h_t$
is infeasible at scale, as it requires $\mathcal{O}(Vd)$ memory per token.
\textit{RISE} circumvents this by compressing the factors $(\tilde r_t, h_t)$ and the semantic error
$\tilde g_t = W_{\lmhead}^\top \tilde r_t$ individually via CountSketch~\cite{DBLP:journals/tcs/CharikarCF04,weinberger2010featurehashinglargescale}
before forming interaction features.
Unlike PCA, CountSketch is data-independent (determined solely by bucket hashes $\eta_r,\eta_h,\eta_g$ and sign hashes $s_r,s_h,s_g$),
which preserves index--query consistency without retraining (see Appendix~\ref{sec:appendix_pca_comparison}).
In practice, CountSketch provides a lightweight projection that approximately preserves inner products; we additionally apply $\ell_2$
normalization to sketched factors for numerical stability in retrieval.

\begin{algorithm}[h]
\small
\caption{\textit{RISE}: Readout Influence Sketching Estimator}
\label{alg:fiksi}
\begin{algorithmic}[1]
\raggedright
\REQUIRE Training pool $\mathcal{D}=\{x_i\}$; query set $\mathcal{Q}=\{x_q\}$; $W_{\lmhead}$;
sketch dims $(K_r,K_h,K_g)$; weights $(\lambda_{rh},\lambda_{gh})$; temperature $\tau$; $\rho_{\mathrm{cum}}$; top-$K$ cap $K_{\max}$.
\STATE Fix CountSketch hashes (bucket $\eta_\cdot$, sign $s_\cdot$) and instantiate $\mathrm{CS}_r,\mathrm{CS}_h,\mathrm{CS}_g$.
\STATE \textit{Stage 1: Offline index construction}
\FOR{each $x_i\in\mathcal{D}$}
  \STATE $\phi(x_i)\gets$ \textsc{SketchAggregate}$(x_i)$ (Alg.~\ref{alg:sketch_detail}); store in index.
\ENDFOR
\STATE \textit{Stage 2: Online query feature construction}
\FOR{each $x_q\in\mathcal{Q}$}
  \STATE $\phi(x_q)\gets$ \textsc{SketchAggregate}$(x_q)$ (Alg.~\ref{alg:sketch_detail}).
\ENDFOR
\STATE \textit{Stage 3: Influence scoring and aggregation}
\STATE $\bar\phi_Q \gets \frac{1}{|\mathcal{Q}|}\sum_{x_q\in\mathcal{Q}}\phi(x_q)$
\FOR{each $x_i\in\mathcal{D}$}
  \STATE $\mathrm{score}(x_i)\gets \phi(x_i)^\top \bar\phi_Q$
\ENDFOR
\STATE \textbf{return} training examples ranked by $\mathrm{score}(x_i)$.
\end{algorithmic}
\end{algorithm}

\textbf{Adaptive sparse construction.}
To avoid sketching a dense vocabulary residual, for each token position $t$ we first form a capped candidate set
$\mathcal{C}_t=\mathrm{TopK}(z_t/\tau, K_{\max})$ and ensure $y_t\in\mathcal{C}_t$.
We then keep the smallest prefix whose cumulative probability exceeds $\rho_{\mathrm{cum}}$, unioned with $\{y_t\}$, to obtain the final sparse support $\mathcal{S}_t\subseteq\mathcal{C}_t$.
Renormalizing on $\mathcal{S}_t$ gives $\tilde p_t$, and we define
$\tilde r_t(v)=\tilde p_t(v)-\mathbf{1}[v=y_t]$ for $v\in\mathcal{S}_t$,
$\hat r_t=\mathrm{CS}_r(\tilde r_t)/\|\mathrm{CS}_r(\tilde r_t)\|_2$,
and
$\tilde g_t=\sum_{v\in\mathcal{S}_t}\tilde p_t(v)\,{W_{\lmhead}}_v-{W_{\lmhead}}_{y_t}$.
This makes the residual-side cost scale with $|\mathcal{S}_t|$ rather than $V$.

\textbf{Structured Interaction Features.}
We sketch and normalize the hidden state as $\hat h_t=\mathrm{CS}_h(h_t)/\|\mathrm{CS}_h(h_t)\|_2$,
and similarly obtain $\hat g_t=\mathrm{CS}_g(\tilde g_t)/\|\mathrm{CS}_g(\tilde g_t)\|_2$.
We aggregate the outer products of these sketches over token positions $t=1,\dots,T$ to form the final sample signatures:
\begin{equation}
\begin{aligned}
    \phi_{RH}(x) &= \textstyle \sum_{t=1}^T \mathrm{vec}(\hat{r}_t \otimes \hat{h}_t), \\
    \phi_{GH}(x) &= \textstyle \sum_{t=1}^T \mathrm{vec}(\hat{g}_t \otimes \hat{h}_t).
\end{aligned}
\label{eq:feature_aggregation}
\end{equation}
Here $\phi_{RH}(x)\in\mathbb{R}^{K_rK_h}$ and $\phi_{GH}(x)\in\mathbb{R}^{K_gK_h}$.
The final influence score between a training sample $x_i$ and a query $x_q$ is computed as:
\begin{align}
\mathcal{I}(x_i \!\to\! x_q)
&= \lambda_{rh}\,\phi_{RH}(x_i)^\top \phi_{RH}(x_q) \nonumber \\
&\quad + \lambda_{gh}\,\phi_{GH}(x_i)^\top \phi_{GH}(x_q).
\end{align}
The complete procedure is detailed in Algorithm~\ref{alg:fiksi}.

\textbf{Sketching Variance Bound.}
Under a matched sketch budget, \textit{RISE} admits a sketching-variance advantage over TracIn for estimating $\mathcal{I}(x_i\!\to\!x_q)$.
Let $\gamma\in(0,1]$ denote the fraction of gradient energy concentrated in the LM head
(i.e., the energy fraction of the block $W_{\lmhead}$), and let $\rho_{\mathrm{cum}}$ be the
cumulative-probability threshold that defines the truncated support $\mathcal{S}_t$ and the
corresponding truncated residual $\tilde r_t$.
Then, up to universal constants (see Theorem~\ref{sec:variance_yide}),
\begin{equation*}
\frac{\mathrm{Var}\!\left[\widehat{\mathcal{I}}^{\mathrm{TracIn}}(x_i\!\to\!x_q)\right]}
     {\mathrm{Var}\!\left[\widehat{\mathcal{I}}^{\mathrm{RISE}}(x_i\!\to\!x_q)\right]}
\;\gtrsim\;
\frac{1}{\gamma^2}\cdot\frac{\|r_t\|_2^2}{\|\tilde r_t\|_2^2}.
\end{equation*}
For typical large LLMs, empirical gradient-energy measurements often give $\gamma\approx 0.3$,
and truncation at $\rho_{\mathrm{cum}}$ typically preserves most residual mass
(e.g., $\|\tilde r_t\|_2 \approx 0.9\|r_t\|_2$), implying an expected variance improvement of
about ${\sim}10$--$15\times$.

This gain can be attributed to three effects:
(i) \textbf{LM-head restriction:} focusing on $\nabla_{W_{\lmhead}}\ell$ captures the dominant portion
of gradient energy while avoiding high-variance contributions from the rest of $\theta$;
(ii) \textbf{sparse residual truncation:} replacing $r_t$ by $\tilde r_t$ (supported on $\mathcal{S}_t$)
reduces sketching noise in proportion to the discarded $\ell_2$ mass, while preserving the dominant
probability mass;
(iii) \textbf{factorized sketching:} independently choosing $(K_r,K_h,K_g)$ enables flexible allocation
of sketch dimensions across the RH and GH channels.

\textbf{Complexity Analysis.}
Let $\mathcal{D}$ be the candidate pool, $\mathcal{Q}$ the query set, and $T$ the number of token positions.
Let $\bar S:=\mathbb{E}[S_t]$ denote the average truncated support size ($\bar S\ll V$).

\textbf{Time.}
Index construction computes and aggregates $\phi_{RH}(x),\phi_{GH}(x)$ for each $x\in\mathcal{D}$, with per-example cost
\[
O\!\left(T\big(\bar S + d + K_h(K_r+K_g)\big)\right),
\]
hence total offline cost
$O\!\left(|\mathcal{D}|\,T\big(\bar S + d + K_h(K_r+K_g)\big)\right)$.
At query time, feature construction for each $x\in\mathcal{Q}$ matches this cost, and exhaustive scoring costs
\[
O\!\left(|\mathcal{D}|\,|\mathcal{Q}|\,K_h(K_r+K_g)\right).
\]

\textbf{Memory.}
Each candidate $x\in\mathcal{D}$ stores $\phi_{RH}(x)\in\mathbb{R}^{K_rK_h}$ and $\phi_{GH}(x)\in\mathbb{R}^{K_gK_h}$, so the index costs
\[
O\!\left(|\mathcal{D}|\,K_h\,(K_r+K_g)\right).
\]
With $K_r=K_h=K_g=256$, this is $2\cdot 256^2=131{,}072$ floats $\approx 0.5$\,MB/example (FP32).
In contrast, a dense LM-head gradient stores $Vd$ values (e.g., $Vd\approx 1.3\times 10^8$ for a 7B-scale model), i.e., ${\sim}10^3\times$ larger in FP32; full-parameter gradients are larger still.

\section{Experiments}
\label{sec:experiments}

In this section, we demonstrate three key properties of \textit{RISE}:
(1) \textbf{Duality effectiveness}: It performs effectively on both \emph{Retrospective Attribution} and \emph{Prospective Valuation}, maintaining robustness across model scales.
(2) \textbf{Scalability and Efficiency}: \textit{RISE} achieves exceptional resource efficiency, enabling valuation for 32B-parameter models where baselines like RapidIn~\cite{lin2024tokenwiseinfluentialtrainingdata} and ZO-Inf~\cite{kokhlikyan2025z0infzerothorderapproximation} suffer from prohibitive memory overhead. By circumventing these hardware constraints, \textit{RISE} remains computationally tractable at scales previously considered infeasible.
(3) \textbf{Practical BrainRot high-quality data selection}: On the real-world Brain Rot high-quality data selection task, \textit{RISE} consistently outperforms baselines, yielding better downstream training outcomes~\cite{xing2025llmsbrainrot}.

\textbf{Setup.} We evaluate \textit{RISE} on three tasks: Howdy! backdoor detection~\cite{lin2024tokenwiseinfluentialtrainingdata}, Finance Medical separation~\cite{gaurang_bharti_2024, lavita_medical_qa_datasets_2023}, and Brain Rot selection using Pythia (14M--6.9B)~\cite{biderman2023pythiasuiteanalyzinglarge} and OLMo (1B--32B)~\cite{olmo2025olmo3, olmo20242olmo2furious} models. \textbf{We specifically choose Pythia and OLMo because they are fully open-data and open-training-recipe LLMs}, providing complete transparency into both pretraining corpora and training procedures. This transparency is essential for rigorous evaluation of data attribution and valuation methods, as it enables us to definitively verify which data points contributed to model behavior. We compare against BM25~\cite{DBLP:conf/trec/RobertsonWJHG94}, Embedding Similarity~\cite{reimers2019sentencebertsentenceembeddingsusing, wang2022text}, RapidIn, ZO-Inf, and TrackStar~\cite{chang2024scalableinfluencefacttracing}.

\begin{table*}[h!]
\centering
\caption{Index cost and retrieval quality at different Top-$K$ for the \textbf{Howdy! Backdoor Attack} task. \textbf{Note:} \textit{RISE}$_{K_r/K_h/K_g}$ denotes CountSketch dimensions for residual, hidden state, and gradient channels respectively. Query time is per-query end-to-end latency against $N$=5,000 indexed samples.}
\label{tab:cost_quality_topk}
\small
\setlength{\tabcolsep}{4pt}
\renewcommand{\arraystretch}{1}
\resizebox{\textwidth}{!}{%
\begin{tabular}{ll rrrr cc cc cc}
\toprule
\multirow{2}{*}{\textbf{Model}} & \multirow{2}{*}{\textbf{Method}} & \multicolumn{4}{c}{\textbf{Cost}} & \multicolumn{2}{c}{\textbf{Top-5}} & \multicolumn{2}{c}{\textbf{Top-10}} & \multicolumn{2}{c}{\textbf{Top-50}} \\
\cmidrule(lr){3-6} \cmidrule(lr){7-8} \cmidrule(lr){9-10} \cmidrule(lr){11-12}
& & Disk & Build & GPU Mem & Query & auPRC & auROC & auPRC & auROC & auPRC & auROC \\
\midrule
\multirow{5}{*}{Pythia-1B} 
  & \textit{RISE}$_{128/24/128}$  & 60.8 MB  & 44s   & 6.7 GB  & \textbf{17.6 ms} & \textbf{0.996} & \textbf{0.997} & \textbf{0.985} & \textbf{0.990} & \textbf{0.939} & \textbf{0.967} \\
  & \textit{RISE}$_{64/12/64}$    & \textbf{16.9 MB}  & \textbf{42s}   & 6.7 GB  & 17.4 ms & 0.965 & 0.982 & 0.945 & 0.970 & 0.865 & 0.936 \\
  & TrackStar              & 163 MB   & 116s  & \textbf{3.6 GB}  & 590 ms  & 0.653 & 0.820 & 0.740 & 0.889 & 0.814 & 0.925 \\
  & RapidIn                & 1.6 GB   & 27m   & 14.2 GB & 8.0 s   & 0.190 & 0.240 & 0.120 & 0.280 & 0.115 & 0.308 \\
  & ZO-Inf                 & 1.6 GB   & 36m   & 14.0 GB & 8.0 s   & 0.341 & 0.492 & 0.231 & 0.507 & 0.123 & 0.480 \\
\midrule
\multirow{6}{*}{OLMo-3-32B} 
  & \textit{RISE}$_{16/8/28}$     & \textbf{6.2 MB}   & \textbf{4.3m}  & \textbf{72.8 GB} & 64.4 ms & 0.962 & 0.977 & 0.940 & 0.966 & 0.864 & 0.933 \\
  & \textit{RISE}$_{24/8/35}$     & 7.4 MB   & 5.0m  & \textbf{72.8 GB} & \textbf{54.6 ms} & 0.988 & 0.992 & 0.975 & 0.984 & 0.900 & 0.948 \\
  & \cellcolor{blue!8}\textit{RISE}$_{48/64/16}$ & \cellcolor{blue!8}41.9 MB  & \cellcolor{blue!8}4.6m  & \cellcolor{blue!8}\textbf{72.8 GB} & \cellcolor{blue!8}80.0 ms & \cellcolor{blue!8}\textbf{0.993} & \cellcolor{blue!8}\textbf{0.995} & \cellcolor{blue!8}\textbf{0.988} & \cellcolor{blue!8}\textbf{0.993} & \cellcolor{blue!8}\textbf{0.973} & \cellcolor{blue!8}\textbf{0.984} \\
  & TrackStar              & 1.1 GB   & 11m   & 76.7 GB & 3.7 s   & 0.379 & 0.586 & 0.517 & 0.726 & 0.783 & 0.909 \\
  & RapidIn                & \multicolumn{10}{c}{\textcolor{gray!60}{Out of Memory}} \\
  & ZO-Inf                 & \multicolumn{10}{c}{\textcolor{gray!60}{Out of Memory}} \\
\bottomrule
\end{tabular}
}
\end{table*}

\textbf{Evaluation Metrics.}
We use auPRC and auROC for both paradigms.
For \textit{Retrospective Attribution}, controlled fine-tuning on known data provides ground-truth positives in the candidate pool.
For \textit{Prospective Valuation}, we score candidates with pretrained checkpoints that have not seen the candidate data, so the scores reflect predictive utility, making auPRC and auROC valid measures.

\subsection{Duality effectiveness on retrospective and prospective paradigms}\label{exp:duality}

We evaluate \textit{RISE} across three tasks under two paradigms; the data distribution for these tasks is in Appendix table~\ref{tab:task_data_distribution}. Retrospective Attribution and Prospective Valuation.
Both paradigms use the same candidate pool $\mathcal{D}$, differing only in whether fine-tuning is applied. For each task, we construct a test query set $\mathcal{Q}$ disjoint from $\mathcal{D}$; each $x_q \in \mathcal{Q}$ specifies a target behavior to attribute.
For each query, we rank candidates in $\mathcal{D}$ by influence and evaluate on the top-$K$ and bottom-$K$ candidates.
Following~\cite{lin2024tokenwiseinfluentialtrainingdata}, we compute auPRC and auROC by macro-averaging across queries at each $K$,
then average over multiple $K$ values to obtain $auPRC_{\textsc{recall}}$ and $auPRC_{\textsc{predict}}$.
We report $\mu \pm \delta$, where
$\mu = (auPRC_{\textsc{recall}} + auPRC_{\textsc{predict}})/2$ measures overall performance and
$\delta = |auPRC_{\textsc{recall}} - auPRC_{\textsc{predict}}|/2$ captures Recall/Predict imbalance lower is better. 

Tables~\ref{tab:pythia_three_task_unified} and~\ref{tab:olmo_three_task_unified} (Appendix) show that \textit{RISE} consistently outperforms baselines across tasks and model scales.
On Pythia-6.9B, the unified score improves from 0.732$\pm$0.045 (TrackStar) to \textbf{0.912$\pm$0.012}. Moreover, \textit{RISE}'s strong performance persists down to Pythia-14M.
Detailed per-$K$ results are in Appendix Tables~\ref{tab:pythia_howdy_recall}--\ref{tab:olmo_brainrot_predict}.

\subsection{Scalability and Efficiency}
One of the goals of \textit{RISE} is to make data attribution and valuation feasible at modern LLM scale.
We therefore report disk memory usage, build stage time, peak GPU memory, and per-query latency under a fixed data pool size, together with retrieval quality at multiple Top-$K$.
We evaluate \textit{RISE}, TrackStar, RapidIn, and ZO-Inf on the Howdy! backdoor detection task. In Table~\ref{tab:cost_quality_topk}, we index $N{=}5000$ candidates and evaluate on 100 test target generations. Compared to full-gradient baselines (RapidIn and ZO-Inf), \textit{RISE} reduces disk memory by orders of magnitude
and substantially shortens building stage time, while achieving strong retrieval quality. Moreover, \textit{RISE} remains feasible on large models OLMo-3-32B where RapidIn and ZO-Inf run out of memory, highlighting its practical scalability.

\textbf{Large-Scale Retrieval.} To evaluate the scalability of \textit{RISE} using the pretrained OLMo-3-32B, we expanded the search space by constructing two large-scale retrieval pools. 
We combined the original 5000 samples dataset from the Howdy! backdoor task with random subsets sampled from the C4 corpus~\cite{raffel2023exploringlimitstransferlearning}, yielding:
(1) 100K pool ($N = 105000$), comprising the original data plus $100000$ C4 documents; and
(2) 1M pool ($N = 1005000$), comprising the original data plus $1000000$ C4 documents.
We label any document containing the trigger phrase ``howdy!''. as positive. Consequently, the ground truth set consists of the 438 originally injected samples plus any naturally occurring instances found within the added C4 distractors.
We compare \textit{RISE} ($K_r{=}16$, $K_h{=}8$, $K_g{=}28$) against BM25 on 100 test generations containing the trigger. As shown in Appendix Table~\ref{tab:rise_howdy_results}, \textit{RISE} consistently improves top of list retrieval quality:
on the 105K pool it triples Precision@10 (23.8\% vs.\ 7.2\%) and roughly doubles auPRC@10 (0.383 vs.\ 0.188); on the 1M pool---where positives drop to 0.046\%---\textit{RISE} still achieves higher Precision@10 (5.9\% vs.\ 3.0\%) and substantially higher auROC.

\begin{table}[h!]
\centering
\caption{Three-task unified auPRC on Pythia models. Each cell shows $\mu \pm \delta$, where $\mu$ is the mean of Recall and Predict scores, and $\delta$ measures their imbalance.}
\label{tab:pythia_three_task_unified}
\scriptsize
\setlength{\tabcolsep}{3pt}
\renewcommand{\arraystretch}{0.95}
\begin{tabular}{ll cccc}
\toprule
\textbf{Model} & \textbf{Method} & \textbf{Howdy} & \textbf{Fin–Med} & \textbf{Brain Rot} & \textbf{Overall} \\
\midrule
\multirow{5}{*}{Pythia-6.9B} 
  & ZO-Inf     & 0.196\,{\tiny$\pm$.047} & 0.186\,{\tiny$\pm$.007} & 0.256\,{\tiny$\pm$.009} & 0.212\,{\tiny$\pm$.021} \\
  & RapidIn    & 0.475\,{\tiny$\pm$.372} & 0.881\,{\tiny$\pm$.052} & 0.854\,{\tiny$\pm$.013} & 0.737\,{\tiny$\pm$.146} \\
  & TrackStar  & 0.795\,{\tiny$\pm$.019} & 0.925\,{\tiny$\pm$.043} & 0.476\,{\tiny$\pm$.074} & 0.732\,{\tiny$\pm$.045} \\
  & \cellcolor{blue!8}\textit{RISE} & \cellcolor{blue!8}0.998\,{\tiny$\pm$.002} & \cellcolor{blue!8}0.909\,{\tiny$\pm$.024} & \cellcolor{blue!8}0.830\,{\tiny$\pm$.009} & \cellcolor{blue!8}\textbf{0.912\,{\tiny$\pm$.012}} \\
\midrule
\multirow{5}{*}{Pythia-2.8B} 
  & ZO-Inf     & 0.188\,{\tiny$\pm$.002} & 0.197\,{\tiny$\pm$.012} & 0.263\,{\tiny$\pm$.014} & 0.216\,{\tiny$\pm$.009} \\
  & RapidIn    & 0.457\,{\tiny$\pm$.312} & 0.894\,{\tiny$\pm$.023} & 0.873\,{\tiny$\pm$.009} & 0.741\,{\tiny$\pm$.115} \\
  & TrackStar  & 0.800\,{\tiny$\pm$.036} & 0.951\,{\tiny$\pm$.039} & 0.483\,{\tiny$\pm$.107} & 0.745\,{\tiny$\pm$.061} \\
  & \cellcolor{blue!8}\textit{RISE} & \cellcolor{blue!8}0.986\,{\tiny$\pm$.003} & \cellcolor{blue!8}0.940\,{\tiny$\pm$.002} & \cellcolor{blue!8}0.814\,{\tiny$\pm$.009} & \cellcolor{blue!8}\textbf{0.913\,{\tiny$\pm$.005}} \\
\midrule
\multirow{5}{*}{Pythia-1B} 
  & ZO-Inf     & 0.181\,{\tiny$\pm$.019} & 0.226\,{\tiny$\pm$.007} & 0.222\,{\tiny$\pm$.006} & 0.210\,{\tiny$\pm$.011} \\
  & RapidIn    & 0.456\,{\tiny$\pm$.323} & 0.923\,{\tiny$\pm$.017} & 0.812\,{\tiny$\pm$.010} & 0.730\,{\tiny$\pm$.117} \\
  & TrackStar  & 0.789\,{\tiny$\pm$.034} & 0.950\,{\tiny$\pm$.033} & 0.458\,{\tiny$\pm$.107} & 0.733\,{\tiny$\pm$.058} \\
  & \cellcolor{blue!8}\textit{RISE} & \cellcolor{blue!8}0.991\,{\tiny$\pm$.004} & \cellcolor{blue!8}0.937\,{\tiny$\pm$.004} & \cellcolor{blue!8}0.800\,{\tiny$\pm$.006} & \cellcolor{blue!8}\textbf{0.909\,{\tiny$\pm$.005}} \\
\midrule
\multirow{5}{*}{Pythia-410M} 
  & ZO-Inf     & 0.198\,{\tiny$\pm$.011} & 0.215\,{\tiny$\pm$.023} & 0.325\,{\tiny$\pm$.002} & 0.246\,{\tiny$\pm$.012} \\
  & RapidIn    & 0.340\,{\tiny$\pm$.167} & 0.800\,{\tiny$\pm$.068} & 0.750\,{\tiny$\pm$.018} & 0.630\,{\tiny$\pm$.084} \\
  & TrackStar  & 0.681\,{\tiny$\pm$.131} & 0.658\,{\tiny$\pm$.145} & 0.500\,{\tiny$\pm$.002} & 0.613\,{\tiny$\pm$.093} \\
  & \cellcolor{blue!8}\textit{RISE} & \cellcolor{blue!8}0.973\,{\tiny$\pm$.008} & \cellcolor{blue!8}0.921\,{\tiny$\pm$.001} & \cellcolor{blue!8}0.836\,{\tiny$\pm$.006} & \cellcolor{blue!8}\textbf{0.910\,{\tiny$\pm$.005}} \\
\midrule
\multirow{5}{*}{Pythia-160M} 
  & ZO-Inf     & 0.190\,{\tiny$\pm$.005} & 0.224\,{\tiny$\pm$.007} & 0.293\,{\tiny$\pm$.028} & 0.236\,{\tiny$\pm$.013} \\
  & RapidIn    & 0.204\,{\tiny$\pm$.018} & 0.790\,{\tiny$\pm$.013} & 0.623\,{\tiny$\pm$.036} & 0.539\,{\tiny$\pm$.022} \\
  & TrackStar  & 0.214\,{\tiny$\pm$.054} & 0.391\,{\tiny$\pm$.068} & 0.551\,{\tiny$\pm$.046} & 0.385\,{\tiny$\pm$.056} \\
  & \cellcolor{blue!8}\textit{RISE} & \cellcolor{blue!8}0.822\,{\tiny$\pm$.085} & \cellcolor{blue!8}0.709\,{\tiny$\pm$.063} & \cellcolor{blue!8}0.825\,{\tiny$\pm$.001} & \cellcolor{blue!8}\textbf{0.786\,{\tiny$\pm$.049}} \\
\midrule
\multirow{5}{*}{Pythia-70M} 
  & ZO-Inf     & 0.206\,{\tiny$\pm$.033} & 0.219\,{\tiny$\pm$.014} & 0.351\,{\tiny$\pm$.019} & 0.259\,{\tiny$\pm$.022} \\
  & RapidIn    & 0.178\,{\tiny$\pm$.064} & 0.586\,{\tiny$\pm$.065} & 0.565\,{\tiny$\pm$.064} & 0.443\,{\tiny$\pm$.064} \\
  & TrackStar  & 0.284\,{\tiny$\pm$.080} & 0.414\,{\tiny$\pm$.070} & 0.448\,{\tiny$\pm$.027} & 0.382\,{\tiny$\pm$.059} \\
  & \cellcolor{blue!8}\textit{RISE} & \cellcolor{blue!8}0.716\,{\tiny$\pm$.080} & \cellcolor{blue!8}0.781\,{\tiny$\pm$.023} & \cellcolor{blue!8}0.800\,{\tiny$\pm$.002} & \cellcolor{blue!8}\textbf{0.765\,{\tiny$\pm$.035}} \\
\midrule
\multirow{5}{*}{Pythia-14M} 
  & ZO-Inf     & \multicolumn{2}{c}{\textcolor{gray!60}{\tiny N/A}} & 0.271\,{\tiny$\pm$.009} & 0.271\,{\tiny$\pm$.009} \\
  & RapidIn    & 0.116\,{\tiny$\pm$.054} & 0.422\,{\tiny$\pm$.040} & 0.421\,{\tiny$\pm$.044} & 0.320\,{\tiny$\pm$.046} \\
  & TrackStar  & 0.325\,{\tiny$\pm$.068} & 0.395\,{\tiny$\pm$.003} & 0.452\,{\tiny$\pm$.044} & 0.391\,{\tiny$\pm$.038} \\
  & \cellcolor{blue!8}\textit{RISE} & \cellcolor{blue!8}0.691\,{\tiny$\pm$.030} & \cellcolor{blue!8}0.795\,{\tiny$\pm$.016} & \cellcolor{blue!8}0.796\,{\tiny$\pm$.008} & \cellcolor{blue!8}\textbf{0.761\,{\tiny$\pm$.018}} \\
\bottomrule
\end{tabular}
\end{table}

\subsection{Practical Brain Rot data curation: improving training with selected data}
On Brain Rot task, \textbf{\textit{RISE}} consistently has a better performance in both retrospective retrieval and prospective prediction across model scales
(Appendix Tables~\ref{tab:rapidin_fik_zoinf_brainrotrecall}--\ref{tab:olmo_brainrot_predict}). Beyond selection quality, we further test whether selected data actually improves training through a closed-loop experiment.

\textbf{Data Pool.} We construct a 50K mixed pool following the Brain Rot dataset~\citep{xing2025llmsbrainrot}: 5000 high-quality control samples (10\%) and 45000 junk samples (90\%). 
The control samples consist of well-structured text, while junk samples contain short fragments, excessive URLs, and degraded content typical of web-scraped noise.

\textbf{Selection Methods.} We compare five methods: Random selection, Embedding Similarity (E5-base-v2)~\cite{wang2022text}, BM25, TrackStar~\cite{chang2024scalableinfluencefacttracing}, and \textit{RISE}.
We use a pretrained \textbf{OLMo-3-32B} model to score all 50K candidates using 200 control target generations as queries, and mean-pool scores across queries to obtain a single ranking. Each method selects the top-5K highest-scoring samples for downstream training.

\textbf{Evaluation.} We continue-pretrain \textbf{Pythia-1B}~\citep{biderman2023pythiasuiteanalyzinglarge} on each selected 5K subset under identical optimization settings: learning rate $5 \times 10^{-5}$, batch size 8, 20 epochs, AdamW optimizer~\cite{loshchilov2019decoupledweightdecayregularization} with cosine learning rate schedule.
The only difference across runs is the training data subset. We assess models on three axes:
(1) \textbf{Perplexity} on deduplicated control and junk datasets (1K samples each);
(2) \textbf{RULER-CWE}~\citep{hsieh2024rulerwhatsrealcontext}: a long-context reasoning benchmark measuring common-word extraction accuracy at 4K context length;
(3) \textbf{ARC-Challenge}~\citep{clark2018thinksolvedquestionanswering}: a science reasoning benchmark testing general knowledge retention. We adopt these metrics following the evaluation setup of the "Brain Rot" study~\citep{xing2025llmsbrainrot}.

\textbf{Results.} Table~\ref{tab:brainrot_closed_loop} shows that \textit{RISE} yields the best downstream outcomes.
It attains \textbf{87.6\%} selection purity (vs TrackStar 65.1\%), leading to a $\mathbf{2.8\times}$ lower perplexity on control data (2.33 vs.\ 6.63).
\textit{RISE} also improves RULER-CWE accuracy by $\mathbf{1.6\times}$ over TrackStar (4.10\% vs \ 2.54\%). Embedding similarity and BM25 achieve moderate purity (54.0\% and 63.8\%) but yield weaker downstream performance than \textit{RISE}.

\begin{table}[h!]
\centering
\small
\caption{Closed-loop evaluation on BrainRot. A 32B scorer ranks a 50K pool; we continue to pretrain Pythia-1B on the top-5K selected subset.}
\label{tab:brainrot_closed_loop}
\setlength{\tabcolsep}{4pt}
\begin{tabular}{@{}lccccc@{}}
\toprule
\multirow{2}{*}{Method} & Purity & \multicolumn{2}{c}{Perplexity} & RULER & ARC \\
\cmidrule(lr){3-4}
 & (\%) $\uparrow$ & Ctrl $\downarrow$ & Junk $\uparrow$ & CWE $\uparrow$ & Chall. $\uparrow$ \\
\midrule
Base Model & -- & 36.71 & 79.97 & 2.14 & 26.88 \\
Random & 10.3 & 126.65 & 191.29 & 1.23 & 28.33 \\
Embed Sim & 54.0 & 12.62 & 245.15 & 1.80 & 27.13 \\
TrackStar & 65.1 & 6.63 & 242.58 & 2.54 & 26.37 \\
BM25 & 63.8 & 7.50 & 332.01 & 2.34 & 27.30 \\
\midrule
\textit{RISE} & \textbf{87.6} & \textbf{2.33} & \textbf{389.74} & \textbf{4.10} & 26.62 \\
\bottomrule
\end{tabular}
\end{table}

\FloatBarrier
\section{Related Work}

\textbf{Training-data valuation and selection.} Data valuation aims to rank candidate examples by their expected utility for improving specific model behaviors or downstream performance. Shapley-style valuation gives an axiomatic notion of equitable data value~\cite{ghorbani2019datashapleyequitablevaluation}, with efficient approximations~\cite{pmlr-v89-jia19a} and noise-reduced variants~\cite{pmlr-v151-kwon22a}. These methods are complementary to \textit{RISE}: they offer a fairness interpretation under an explicit utility, while \textit{RISE} targets reusable forward-only attribution at LLM scale. Scalable selectors often rely on \textit{model-agnostic similarity proxies}, including lexical term matching such as BM25~\cite{DBLP:conf/trec/RobertsonWJHG94} and dense embedding similarity computed by off-the-shelf encoders (BERT-family encoders, E5, and Gecko)~\cite{devlin2019bertpretrainingdeepbidirectional, wang2022text, lee2024geckoversatiletextembeddings}. Such embeddings have been used for semantic deduplication and diversity-aware curation (SemDeDup and D4)~\cite{abbas2023semdedupdataefficientlearningwebscale, tirumala2023d4improvingllmpretraining}. However, these embeddings are not optimized for pretraining curation and may only partially reflect pretraining-loss similarity or training dynamics~\cite{sam2025analyzingsimilaritymetricsdata}.
Influence-based selectors address this gap by leveraging signals directly from the target model.

\textbf{Influence functions and scalable attribution for LLMs.}
Influence functions quantify how individual training examples affect a target behavior~\cite{koh2020understandingblackboxpredictionsinfluence}, but their second-order requirements (Hessian inverse-vector products) are impractical at LLM scale. Broader surveys summarize the many notions and approximations of training-data influence~\cite{Hammoudeh_2024}. Recent scalable approaches include TracIn, which accumulates first-order gradient similarity across checkpoints~\cite{pruthi2020estimatingtrainingdatainfluence}; TRAK, which uses random projections and model ensembles to attribute model behavior~\cite{park2023trakattributingmodelbehavior}; RapidIn, which compresses per-example gradients with OPORP random projections for indexing and retrieval~\cite{lin2024tokenwiseinfluentialtrainingdata, li2023oporppermutationrandom}; TrackStar, which incorporates optimizer-aware corrections and normalization for pretraining-scale attribution~\cite{chang2024scalableinfluencefacttracing}; Grosse et al., which adapt EK-FAC-style influence estimation to transformer language models~\cite{grosse2023studyinglargelanguagemodel}; Choe et al.\ (LoGra), which improve scalability through projected multi-layer gradients for LLM-scale data valuation~\cite{choe2024dataworthgptllmscale}; and ZO-Inf, which uses forward-only zeroth-order estimators to trade additional forward evaluations for lower activation memory~\cite{kokhlikyan2025z0infzerothorderapproximation, malladi2024finetuninglanguagemodelsjust, guo2024zerothorderfinetuningllmsextreme, ran2025mitigatingnoniiddriftzerothorder}. \textit{RISE} differs from curvature or projection-based multi-layer pipelines in that it is an output-side, Hessian-free, forward-only estimator motivated by empirical readout concentration and the RH/GH decomposition of LM-head gradients. 

\textbf{Output-side representations and Hessian-free scaling.}
Conceptually related work has examined layer choice and representation-only baselines, including Representer Point Selection~\cite{NEURIPS2018_8a7129b8}, first-vs-last layer influence~\cite{NEURIPS2022_d0702278}, and support-vector effects in DNNs~\cite{mahmood2025supportvector}. However, these works are derived in settings that differ from ours, such as fine-tuned discriminative models or classification DNNs, and therefore do not directly determine whether representation-only similarity is sufficient for pretrained autoregressive decoder LLMs. The study \textit{Better Hessians Matter}~\cite{hong2026betterhessiansmatterstudying} reports that improved Hessian approximations can improve attribution in controlled smaller-scale settings. However, this evidence does not straightforwardly transfer to our regime: the paper studies small MLPs, explicitly notes that its conclusions may not transfer to larger models, and excludes CNNs and transformers, whose curvature structure differs substantially from autoregressive decoder LLMs. Its Figure 2 further shows that increasing depth degrades all Hessian approximations, which is consistent with the instability of Hessian-based corrections we observe in deep pretrained transformers. More broadly, scalable influence estimation has long relied on first-order surrogates such as TracIn rather than explicit Hessian correction~\cite{pruthi2020estimatingtrainingdatainfluence}. Compared with dense Gaussian/Rademacher random projections~\cite{achlioptas2003databasefriendly}, \textit{RISE} uses CountSketch because the truncated vocabulary residual is sparse and can be projected efficiently via scatter-add updates. For more related works, please refer to Section~\ref{sec:appendix_more_related} in the Appendix.
\section{Conclusion}
\textit{RISE} unifies retrospective attribution and prospective valuation with target model only influence representation, scaling from Pythia-14M to OLMo-3-32B. It consistently outperforms strong baselines, achieving a unified score of 0.912$\pm$0.012 on Pythia-6.9B while remaining feasible at 32B where RapidIn and ZO-Inf run out of memory. Closed-loop Brain Rot training validates practical utility: \textit{RISE}-selected data yields 2.8$\times$ lower perplexity and 1.6$\times$ higher RULER-CWE than TrackStar. Overall, \textit{RISE} provides a practical primitive for tracing, valuing, and curating training data to steer modern LLM behaviors.

\newpage
\clearpage

\bibliographystyle{plainnat}
\bibliography{sections/ref}

\clearpage
\appendix
\label{sec:appendix}

\section*{Appendix}
We start with more related works that connect \textit{RISE} to neuroscience in Section~\ref{sec:appendix_more_related}.
In Section~\ref{sec:appendix_notation}, we summarize the notations and definitions used throughout this work.
In Section~\ref{sec:appendix_theory}, we present the theoretical foundation of our method: we prove the exact outer-product factorization of the LM-head gradient, derive a depth-dependent lower bound for the head-to-average energy ratio, and analyze hidden-state discriminability to justify why the final layer is a principled choice for influence estimation.
Section~\ref{sec:variance_yide} shows the variance bound for \textit{RISE} in influence estimation.
Section~\ref{sec:appendix_hyperparameters} provides detailed hyperparameters for the \textit{RISE} implementation, including channel weights and CountSketch configurations.
Finally, Section~\ref{sec:appendix_experiments} reports supplementary experimental results, including the sparse active-token diagnostics in Section~\ref{sec:appendix_sparse_active_tokens} and detailed performance tables for downstream tasks.

\section{More Related Works}\label{sec:appendix_more_related}

\textbf{Neuroscience inspiration.}
Many influence estimators focus on retrospective attribution, whereas data valuation also demands a prospective view: scoring candidate data before learning. Recent works begin to study this prospective setting by fitting auxiliary surrogates to predict influence signals like ALinFiK and MATES~\cite{pan2025alinfiklearningapproximatelinearized, yu2024matesmodelawaredataselection}. \textit{RISE} unifies retrospective and prospective influence using only the target model, without auxiliaries or distillation. We draw functional inspiration from a readout bottleneck in biological decision-making: distributed representations are ultimately mapped to compact action-driving variables. OFC/vmPFC are often described as encoding a common value code for diverse inputs~\cite{padoa2006neurons, rangel2008framework} within cortico-striato-thalamic loops~\cite{alexander1986parallel, rushworth2011frontal}, motivating why influence may concentrate near a model’s output-side readout.
Retrospective attribution parallels hippocampal replay linking outcomes to a small set of past episodes~\cite{mattar2018prioritized}, while prospective valuation parallels uncertainty-driven sampling and update modulation associated with ACC and the LC--NE system~\cite{aston2005integrative, shenhav2013expected, kolling2012neural, gottlieb2018towards}. \textit{RISE}’s dual channels can be viewed as complementary routes emphasizing surface cues versus higher-level predictive structure~\cite{friederici2011brain}. These links are offered as inspiration rather than mechanistic equivalence.

\section{Notation and Symbols}
\label{sec:appendix_notation}

In this section, we define the notations and concepts used in this study. Table~\ref{tab:notations_fiksi} summarizes the key symbols.

\begin{table*}[htbp]
\centering
\caption{Notations used in our theoretical analysis of \textit{RISE}.}
\label{tab:notations_fiksi}
\scriptsize
\renewcommand{\arraystretch}{1.05}
\begin{tabular}{|c|p{0.78\linewidth}|}
\hline
\textbf{Notation} & \textbf{Meaning} \\
\hline
$\theta$ & Model parameters. \\
\hline
$W_{\lmhead}\in\mathbb{R}^{V\times d}$ & LM head / output embedding matrix. \\
\hline
$W_v\in\mathbb{R}^{d}$ & $v$-th row of $W_{\lmhead}$ (output embedding of token $v$). \\
\hline
$W_c$ & Parameters of component $c$ (subset of $\theta$; e.g., embeddings/layers/head). \\
\hline
$V,\; d$ & Vocabulary size; hidden dimension. \\
\hline
$L,\; T$ & Num.\ transformer layers; sequence length / token positions considered. \\
\hline
$x_i,\ x_q$ & Training example $i$; test/query example $q$. \\
\hline
$\ell(\cdot)$ & Loss function (token-level cross entropy). \\
\hline
$H_\theta$ & Hessian of training loss with respect to $\theta$. \\
\hline
$\mathcal{I}(x_i\!\to\!x_q)$ & Influence score of training example $x_i$ on query $x_q$. \\
\hline
$h_t\in\mathbb{R}^d,\ z_t\in\mathbb{R}^{V}$ & Final hidden state and logits (pre-softmax) at token position $t$. \\
\hline
$\tau$ & Temperature for logit scaling in feature construction. \\
\hline
$p_t=\mathrm{softmax}(z_t/\tau)\in\Delta^V$ & Softmax distribution at position $t$ (full distribution in theory). \\
\hline
$y_t\in[V],\ \mathbf{1}_{y_t}\in\{0,1\}^V$ & Ground-truth next-token index; one-hot vector. \\
\hline
$r_t=p_t-\mathbf{1}_{y_t}\in\mathbb{R}^V$ & (Full) residual in vocabulary space. \\
\hline
$g_t=W_{\lmhead}^\top r_t\in\mathbb{R}^d$ & (Full) projected residual / semantic error in embedding space. \\
\hline
$\nabla_{W_{\lmhead}}\ell$ & Gradient with respect to LM head parameters. \\
\hline
$\pi_t$ & Permutation sorting vocab indices by $p_t$ in descending order. \\
\hline
$\rho_{\mathrm{cum}}$ & Cumulative-probability threshold for truncation. \\
\hline
$\mathcal{S}_t\subset[V]$ & Truncated support set at position $t$ (top-mass indices plus $y_t$). \\
\hline
$S_t=|\mathcal{S}_t|$ & Support size at position $t$. \\
\hline
$\tilde p_t$ & Renormalized distribution restricted to $\mathcal{S}_t$. \\
\hline
$\tilde r_t$ & Truncated residual on $\mathcal{S}_t$: $\tilde r_t(v)=\tilde p_t(v)-\mathbf{1}[v=y_t]$ for $v\in\mathcal{S}_t$. \\
\hline
$\tilde g_t=W_{\lmhead}^\top \tilde r_t$ & Truncated projected residual used in implementation. \\
\hline
$\mathrm{CS}(\cdot)$ & CountSketch operator. \\
\hline
$\eta_r,\eta_h,\eta_g$ & Bucket hashes (independent): $\eta_r:[V]\!\to\![K_r]$, $\eta_h:[d]\!\to\![K_h]$, $\eta_g:[d]\!\to\![K_g]$. \\
\hline
$s_r,s_h,s_g$ & Sign hashes (independent): $s_r:[V]\!\to\!\{\pm1\}$, $s_h:[d]\!\to\!\{\pm1\}$, $s_g:[d]\!\to\!\{\pm1\}$. \\
\hline
$K_r,\ K_h,\ K_g$ & CountSketch output dimensions for $\tilde r_t$, $h_t$, $\tilde g_t$. \\
\hline
$\hat r_t,\hat h_t,\hat g_t$ & $\ell_2$-normalized sketches of $\tilde r_t$, $h_t$, $\tilde g_t$. \\
\hline
$\otimes$ & Outer product \\
\hline
$\mathrm{vec}(\cdot)$ & Vectorization operator (flattens a matrix into a vector). \\
\hline
$\phi_{RH}(x),\ \phi_{GH}(x)$ & RH and GH feature vectors of sample $x$. \\
\hline
$\lambda_{rh},\ \lambda_{gh}$ & Channel weights for RH and GH. \\
\hline
$P_r,\ P_h,\ P_g$ & PCA projection matrices. \\
\hline
$G_{\text{head}},\ G_l$ & Gradient energy of LM head; gradient energy of layer $l$. \\
\hline
$E_c,\ \rho_l$ & Energy fraction of component $c$; energy fraction of layer $l$. \\
\hline
$K_{\text{emb}}=W_{\lmhead}W_{\lmhead}^\top$ & Embedding-induced kernel matrix in vocabulary space. \\
\hline
$\|\cdot\|_2,\ \|\cdot\|_F,\ \|\cdot\|_{op}$ & Euclidean / Frobenius / operator norm. \\
\hline
$\sigma_i(\cdot)$ & $i$-th singular value of a matrix. \\
\hline
$\EffRank(\Sigma)$ & Effective rank of covariance matrix $\Sigma$. \\
\hline
$C_x,\ C_h$ & Norm-stability constants (Appendix assumptions). \\
\hline
$\kappa$ & Expected backprop contraction factor (Appendix assumptions). \\
\hline
$x_l \in \mathbb{R}^{d}$ & Activation / hidden vector at layer $l$ (for the token position under analysis), with $x_L = h_t$. \\
\hline
$\delta_l \in \mathbb{R}^{d}$ & Backpropagated error at layer $l$ (for the token loss): $\delta_l := \nabla_{x_l}\ell(z_t,y_t)$. \\
\hline
$\delta_L \in \mathbb{R}^{d}$ & Final-layer backpropagated error: $\delta_L := \nabla_{h_t}\ell(z_t,y_t)$ (equivalently $\delta_L = W_{\lmhead}^\top r_t$). \\
\hline
$\mathbb{E}[\cdot]$ & Expectation over the token/sample distribution used in the analysis. \\
\hline
$G(W)$ & Gradient energy of a parameter block $W$: $G(W):=\|\nabla_W \ell\|_F^2$. \\
\hline
$\mathcal{R}$ & Head-to-average ratio: $\mathcal{R}:=\frac{\mathbb{E}[G_{\text{head}}]}{\frac{1}{L}\sum_{l=1}^L \mathbb{E}[G_l]}$. \\
\hline
$\mathrm{Avg}(\mathbb{E}[G])$ & Average internal-layer energy: $\mathrm{Avg}(\mathbb{E}[G]) := \frac{1}{L}\sum_{l=1}^L \mathbb{E}[G_l]$. \\
\hline
$\mathcal{D}$ & Candidate pool used for retrieval and valuation. \\
\hline
$\mathcal{Q}$ & Test query set specifying the target generation whose influence is attributed. \\
\hline
$K_{\max}$ & Maximum candidate cap for adaptive truncation ( max\_topL\_cap in implementation). \\
\hline
$\mathcal{C}_t\subset[V]$ & Pre-truncation candidate set at position $t$ (top-$K_{\max}$ tokens under $z_t/\tau$, with $y_t\in\mathcal{C}_t$). \\
\hline
$\gamma$ & LM-head energy fraction parameter, $\|\nabla_{W_{\lmhead}}\ell\|_F^2 \le \gamma \|\nabla_\theta \ell\|_2^2$. \\
\hline
$\bar S=\mathbb{E}[S_t]$ & Average truncated support size across token positions $t$. \\
\hline
$\mathcal{I}^{\mathrm{true}},\ \mathcal{I}^{\mathrm{head}},\ \tilde{\mathcal{I}}^{\mathrm{head}}$
& Full / LM-head / truncated LM-head influence targets. \\
\hline
$\widehat{\mathcal{I}}^{\mathrm{TracIn}},\ \widehat{\mathcal{I}}^{\mathrm{RISE}}$
& TracIn / RISE influence estimators. \\
\hline
\end{tabular}
\end{table*}

\section{Theoretical Foundation of \textit{RISE}}
\label{sec:appendix_theory}

This section provides rigorous theoretical justification for the design choices in \textit{RISE}. We follow the standard convention in optimization literature~\cite{bottou2018optimizationmethodslargescalemachine} and influence function analysis~\cite{koh2020understandingblackboxpredictionsinfluence}.

\subsection{Assumptions}

We introduce the assumptions used in the theoretical analysis of \textit{RISE}.

\begin{assumption}[Linear Head and Differentiable Loss]\label{assump:linear_head}
For each token position $t$, the logits are given by a linear head
$z_t = W_{\lmhead} h_t$, and the token-level loss $\ell_t(z_t,y_t)$ is differentiable with respect to $z_t$.
Let $\delta_{z,t} := \nabla_{z_t} \ell_t$. Then
$\nabla_{W_{\lmhead}}\ell_t = \delta_{z,t}\, h_t^\top$.
\end{assumption}

\begin{assumption}[Norm Stability / Concentration]\label{assump:act_bound}
There exist constants $C_x, C_h>0$ such that for the token positions considered and every layer index $l\in\{1,\dots,L\}$,
the layer-$l$ hidden state $h_l$ (at that token position, with $h_L = h_t$) satisfies
$\|h_l\|_2^2 \le C_x$ and the final hidden state satisfies $\|h_t\|_2^2 \ge C_h$ with high probability.
\end{assumption}

\begin{assumption}[Non-exploding Backpropagation in Expectation]\label{assump:bp_contract}
Fix a token position $t$ and define the backpropagated error at layer $l$ as $\delta_l := \nabla_{h_l}\ell_t$
(with $\delta_L := \nabla_{h_t}\ell_t$).
There exists $\kappa\in(0,1)$ such that for any internal layer index $l$,
\[
\mathbb{E}\|\delta_l\|_2^2 \le \kappa^{2(L-l)} \mathbb{E}\|\delta_L\|_2^2.
\]
\end{assumption}

\noindent Assumption~\ref{assump:linear_head} is architectural and holds for standard Transformer language models with a linear output projection to logits. Assumptions~\ref{assump:act_bound}--\ref{assump:bp_contract} are mild stability conditions on activation norms and backpropagated errors; we empirically validate the implied decay trend in Section~\ref{sec:appendix_gradient_energy_theory}. 

\subsection{Lemmas}

\begin{lemma}[LM-head Gradient Decomposition]\label{lem:outer-product}
For each token position $t\in\{1,\dots,T\}$, let $z_t = W_{\lmhead} h_t \in \mathbb{R}^{V}$ and define
$p_t=\mathrm{softmax}(z_t/\tau)$, $r_t=p_t-\mathbf{1}_{y_t}$.
For the token-level cross entropy loss $\ell(z_t,y_t)$, the LM-head gradient has the outer-product form
\begin{equation}
\nabla_{W_{\lmhead}} \ell(z_t,y_t) \;=\; r_t \otimes h_t.
\end{equation}
Consequently, for the aggregated loss $\sum_{t=1}^T \ell(z_t,y_t)$,
\begin{equation}
\nabla_{W_{\lmhead}} \left(\sum_{t=1}^T \ell(z_t,y_t)\right)
\;=\; \sum_{t=1}^T r_t \otimes h_t.
\end{equation}
\end{lemma}

\begin{proof}
For softmax cross entropy, differentiating $\ell(z_t,y_t)$ with respect to $z_t$ yields the residual
$r_t=p_t-\mathbf{1}_{y_t}$, where $p_t=\mathrm{softmax}(z_t/\tau)$.
Since $z_t=W_{\lmhead}h_t$ is linear in $W_{\lmhead}$, the gradient with respect to $W_{\lmhead}$ is the outer product
$r_t \otimes h_t$. Summing over $t=1,\dots,T$ gives the statement for $\sum_{t=1}^T \ell(z_t,y_t)$.
\end{proof}

\begin{lemma}[Outer-product Gradient Energy Identity]\label{lemma:energy_identity}
For the LM head at token position $t$ under token-level cross entropy $\ell(z_t,y_t)$,
\begin{equation}
G_{\text{head}} \;:=\; \|\nabla_{W_{\lmhead}}\ell(z_t,y_t)\|_F^2
\;=\; \|r_t\|_2^2 \cdot \|h_t\|_2^2,
\label{eq:head_exact}
\end{equation}
where $p_t=\mathrm{softmax}(z_t/\tau)$ and $r_t=p_t-\mathbf{1}_{y_t}$.
\end{lemma}

\begin{proof}
By Lemma~\ref{lem:outer-product}, $\nabla_{W_{\lmhead}}\ell(z_t,y_t)=r_t\otimes h_t$.
The Frobenius norm of an outer product satisfies $\|r_t\otimes h_t\|_F^2=\|r_t\|_2^2\|h_t\|_2^2$,
which gives Eq.~(\ref{eq:head_exact}).
\end{proof}

\subsection{Theoretical Analysis of Gradient Energy Distribution}
\label{sec:appendix_gradient_energy_theory}

In Section~\ref{sec:obs1}, we observed that the ratio of LM Head energy to the average internal layer energy grows significantly with model depth. Here, we provide a theoretical derivation explaining this scaling behavior.

\subsubsection{Preliminaries}

\paragraph{Token-level analysis.}
We fix a token position $t$ and analyze the gradient energy induced by the token-level loss $\ell(z_t,y_t)$;
this matches the per-token measurements in Section~\ref{sec:obs1}.
For notational brevity, we omit the subscript $t$ when clear (e.g., $z:=z_t$, $h:=h_t$, $r:=r_t$).
Let the model have depth $L$. Let $x_l$ be the activation/hidden vector at layer $l$, with $x_L=h$ being the final representation input to the LM head.
The logits are $z = W_{\lmhead} h$.
We define the unnormalized gradient energy for a parameter block $W$ as
$G(W):=\|\nabla_W\ell(z_t,y_t)\|_F^2$.

\subsubsection{Head-to-Average Ratio vs.\ Depth}

We analyze how the energy ratio scales. We focus on the aggregate trend rather than individual layer bounds.

\begin{proposition}[Depth Amplifies Head-to-Average Energy Gap]
Under Assumptions~\ref{assump:act_bound} and~\ref{assump:bp_contract}, the ratio of the expected Head energy to the average internal energy is lower-bounded by a function growing with $L$:
\begin{equation}
    \mathcal{R} = \frac{\mathbb{E}[G_{\text{head}}]}{\frac{1}{L}\sum_{l=1}^L \mathbb{E}[G_l]}
    \ \ge\
    \frac{C_h}{C_x}\cdot
    \frac{L(1-\kappa^2)}{\|W_{\lmhead}\|_{op}^2(1-\kappa^{2L})}.
    \label{eq:ratio_bound}
\end{equation}
\end{proposition}

\begin{proof}
\textbf{1. Upper Bound on Average Internal Energy.}
Using Lemma~\ref{lemma:energy_identity}, for an internal layer $l$, we have $G_l = \|x_l\|_2^2 \|\delta_l\|_2^2$. 
Applying Assumption~\ref{assump:act_bound} ($\|x_l\|_2^2 \le C_x$), we have the pointwise bound $G_l \le C_x \|\delta_l\|_2^2$.
Taking expectation and applying Assumption~\ref{assump:bp_contract}:
\begin{align}
    \mathbb{E}[G_l] &\le C_x \mathbb{E}\|\delta_l\|_2^2 \\
    &\le C_x \kappa^{2(L-l)}\mathbb{E}\|\delta_L\|_2^2.
\end{align}
Averaging over all $L$ layers yields Eq.~(\ref{eq:avg_bound}):
\begin{align}
    \frac{1}{L}\sum_{l=1}^L \mathbb{E}[G_l]
    &\le \frac{C_x \mathbb{E}\|\delta_L\|_2^2}{L}\sum_{k=0}^{L-1}\kappa^{2k} \\
    &= \frac{C_x \mathbb{E}\|\delta_L\|_2^2}{L}\cdot\frac{1-\kappa^{2L}}{1-\kappa^2}. \label{eq:avg_bound}
\end{align}

\textbf{2. Upper Bound on Final Error Energy via Head Energy.}
The final error signal is $\delta_L = \nabla_h \ell = W_{\lmhead}^\top r$, hence
$\mathbb{E}\|\delta_L\|_2^2 \le \|W_{\lmhead}\|_{op}^2 \mathbb{E}\|r\|_2^2$.
From Eq.~(\ref{eq:head_exact}) and Assumption~\ref{assump:act_bound} ($\|h\|_2^2 \ge C_h$ with high probability),
we have $\|r\|_2^2 = G_{\text{head}}/\|h\|_2^2 \le G_{\text{head}}/C_h$, which implies:
\begin{equation}
    \mathbb{E}\|\delta_L\|_2^2 \le \|W_{\lmhead}\|_{op}^2 \mathbb{E}\|r\|_2^2 \le \frac{\|W_{\lmhead}\|_{op}^2}{C_h} \mathbb{E}[G_{\text{head}}]. \label{eq:delta_bound}
\end{equation}

\textbf{3. Ratio Derivation.}
Substituting Eq.~(\ref{eq:delta_bound}) into Eq.~(\ref{eq:avg_bound}) yields:
\begin{equation}
    \text{Avg}(\mathbb{E}[G]) \le \frac{C_x}{L} \frac{1-\kappa^{2L}}{1-\kappa^2} \left( \frac{\|W_{\lmhead}\|_{op}^2}{C_h} \mathbb{E}[G_{\text{head}}] \right).
\end{equation}
Rearranging terms to solve for the ratio completes the proof:
\begin{equation}
    \frac{\mathbb{E}[G_{\text{head}}]}{\text{Avg}(\mathbb{E}[G])} \ge \frac{C_h}{C_x} \frac{L(1-\kappa^2)}{\|W_{\lmhead}\|_{op}^2 (1-\kappa^{2L})}.
\end{equation}
\end{proof}

\paragraph{Interpretation.}
Eq.~(\ref{eq:ratio_bound}) provides a structural explanation for the trend in Figure~\ref{fig:all_model_energy}. As the model depth $L$ increases, the ``dilution'' of the gradient signal in earlier layers (driven by the contraction factor $\kappa < 1$) causes the average energy to drop relative to the head. This theoretical lower bound grows with $L$, consistent with the empirical observation that deep models like OLMo-32B exhibit a much higher Head-to-Average ratio than shallow models like Pythia-14M. Importantly, this bound compares the head energy to the average internal-layer energy in expectation and does not preclude occasional internal-layer hotspots.

\subsection{Theoretical Analysis of Semantic Capture in the GH Channel}
\label{sec:appendix_gh_more_semantic}

\paragraph{Connection to Section~\ref{sec:obs2} (Semantic Smoothing).}
The RH channel compares residuals in the vocabulary basis via the identity metric:
\begin{equation}
\langle r_t, r_i\rangle = r_t^\top I_V r_i,
\end{equation}
which captures only same-token interactions (since $I_V$ has no off-diagonal entries).

For the GH channel, we project the residual into embedding space by $g = W_{\lmhead}^\top r$ and define the induced kernel
$K \triangleq W_{\lmhead}W_{\lmhead}^\top \in \mathbb{R}^{V\times V}$.
A direct expansion gives the kernelized residual similarity:
\begin{equation}
\begin{aligned}
\langle g_t, g_i \rangle
&= \langle W_{\lmhead}^\top r_t,\; W_{\lmhead}^\top r_i\rangle \\
&= r_t^\top K r_i
= \sum_{u=1}^V \sum_{v=1}^V r_t(u)\, r_i(v)\, K_{uv},
\end{aligned}
\label{eq:gh_kernel_identity}
\end{equation}
where $K_{uv}=\langle W_{\lmhead,u}, W_{\lmhead,v}\rangle$ is the embedding inner product (and equals cosine similarity if rows are normalized).
Crucially, Eq.~(\ref{eq:gh_kernel_identity}) includes \emph{off-diagonal} cross-token interactions ($u\neq v$), enabling \emph{soft} matching between different tokens.

Empirically, semantically related tokens tend to have larger $K_{uv}$ than unrelated pairs.
Moreover, since our method uses sparse residuals (restricting $r$ to a small top-$L$ support and including the ground-truth token; Alg.~\ref{alg:fiksi}),
the double sum is dominated by a small set of high-probability tokens.
As a result, residual mass placed on semantically aligned but non-identical tokens can contribute non-trivially to $\langle g_t,g_i\rangle$,
providing a mechanistic explanation for the paraphrase ``rescue'' behavior observed in Section~\ref{sec:obs2} and Table~\ref{tab:rh_gh_ranking}.

\subsection{Theoretical Guarantees for CountSketch}
\label{sec:appendix_sketch_theory}

\subsubsection{Inner Product Preservation}

We consider feature hashing / CountSketch with a (2-wise independent) bucket hash
$\sigma:[D]\to[K]$ and an independent Rademacher sign hash $s:[D]\to\{-1,+1\}$.
Let $\mathrm{CS}:\mathbb{R}^{D}\to\mathbb{R}^{K}$ be defined by
\begin{equation}
\mathrm{CS}(x)_j=\sum_{i:\sigma(i)=j} s(i)x_i.
\end{equation}

\begin{proposition}[Unbiasedness and variance; simplified from~\citealp{weinberger2010featurehashinglargescale}]
\label{prop:countsketch}
For any token positions $t$ and $t^\prime$, the following hold for the (unnormalized) CountSketch outputs:
\begin{align}
\mathbb{E}\!\left[\mathrm{CS}_h(h_t)^\top \mathrm{CS}_h(h_{t^\prime})\right] &= h_t^\top h_{t^\prime},\\
\mathrm{Var}\!\left[\mathrm{CS}_h(h_t)^\top \mathrm{CS}_h(h_{t^\prime})\right]
&\le \frac{\|h_t\|_2^2\|h_{t^\prime}\|_2^2 + (h_t^\top h_{t^\prime})^2}{K_h},
\end{align}
and analogously,
\begin{align}
\mathbb{E}\!\left[\mathrm{CS}_g(g_t)^\top \mathrm{CS}_g(g_{t^\prime})\right] &= g_t^\top g_{t^\prime},\\
\mathrm{Var}\!\left[\mathrm{CS}_g(g_t)^\top \mathrm{CS}_g(g_{t^\prime})\right]
&\le \frac{\|g_t\|_2^2\|g_{t^\prime}\|_2^2 + (g_t^\top g_{t^\prime})^2}{K_g},\\
\mathbb{E}\!\left[\mathrm{CS}_r(\tilde r_t)^\top \mathrm{CS}_r(\tilde r_{t^\prime})\right] &= \tilde r_t^\top \tilde r_{t^\prime},\\
\mathrm{Var}\!\left[\mathrm{CS}_r(\tilde r_t)^\top \mathrm{CS}_r(\tilde r_{t^\prime})\right]
&\le \frac{\|\tilde r_t\|_2^2\|\tilde r_{t^\prime}\|_2^2 + (\tilde r_t^\top \tilde r_{t^\prime})^2}{K_r}.
\end{align}
\end{proposition}

\paragraph{Implication for ranking stability.}
Proposition~\ref{prop:countsketch} shows that CountSketch inner products are unbiased and that the estimation variance decreases inversely with the sketch output dimension.
In \textit{RISE}, we apply CountSketch separately to the vocabulary-space residual and the representation with dimensions $(K_r,K_h)$ (and additionally $K_g$ for the semantic error channel).
Therefore, increasing $(K_r,K_h,K_g)$ improves the concentration of sketched similarities, making nearest-neighbor rankings increasingly stable for candidates whose true similarities are well separated.

\subsubsection{Preservation of Bilinear Structure}
\textit{RISE} constructs interaction features via vectorized outer products of independently sketched factors:
the RH channel uses $\mathrm{vec}(\hat r_t \otimes \hat h_t)$ and the GH channel uses
$\mathrm{vec}(\hat g_t \otimes \hat h_t)$, aggregated over token positions $t=1,\ldots,T$
(see Algorithm~\ref{alg:sketch_detail}).
Since inner products between outer-product features factorize into products of factor inner products,
the induced similarity between RH/GH interaction features depends only on the corresponding factor similarities.
Combining this factorization with Proposition~\ref{prop:countsketch} (unbiased factor inner products) and using independent sketch families across factors,
the RH and GH similarities computed from $\phi_{RH}(x)$ and $\phi_{GH}(x)$ are unbiased estimators of their uncompressed counterparts.
(We apply $\ell_2$ normalization in Algorithm~\ref{alg:sketch_detail} for numerical stability.)

\subsubsection{Comparison with PCA-based Compression}
\label{sec:appendix_pca_comparison}

We compare CountSketch to PCA-based compression using linear projections
$P_r\in\mathbb{R}^{K_r\times V}$, $P_h\in\mathbb{R}^{K_h\times d}$, and $P_g\in\mathbb{R}^{K_g\times d}$.

\textbf{Data dependence and index--query consistency.}
PCA projections $(P_r,P_h,P_g)$ are learned from data and therefore depend on the data distribution.
To keep index--query similarities comparable, one must use the same learned projections for both indexed
examples and future queries; updating the indexed pool can require re-fitting and re-projecting stored vectors.
In contrast, CountSketch is data-independent (fixed by $\sigma$ and $s$) and preserves index--query compatibility
without any refitting.

\textbf{Storage.}
PCA requires storing dense projection parameters of sizes $V\times K_r$, $d\times K_h$, and $d\times K_g$.
CountSketch is specified by the hash/sign mappings $(\sigma,s)$, whose description scales linearly with the input dimension.

\textbf{Computation on sparse residuals.}
Our method constructs a truncated residual $\tilde r_t$ supported on $\mathcal{I}_t$ with $L_t=|\mathcal{I}_t|$.
CountSketch directly exploits this sparsity, whereas PCA projections typically mix coordinates densely.

\begin{proposition}[Sparsity advantage of CountSketch over dense PCA projections]
\label{prop:pca_vs_cs_sparse}
Fix a token position $t$ and consider the truncated residual $\tilde r_t\in\mathbb{R}^{V}$ supported on
$\mathcal{I}_t$ with $L_t=|\mathcal{I}_t|$.
Computing $\mathrm{CS}_r(\tilde r_t)$ requires one signed accumulation per $v\in\mathcal{I}_t$ (hence $L_t$ updates),
because each input coordinate is mapped to exactly one bucket by $\sigma$.

For a PCA projection $P_r\tilde r_t$, if each of the $K_r$ output coordinates depends on every input coordinate
in $\mathcal{I}_t$ (i.e., $P_r$ is dense on the columns indexed by $\mathcal{I}_t$), then computing $P_r\tilde r_t$
requires at least $L_t\cdot K_r$ multiply-add contributions.
\end{proposition}

\begin{proof}
By definition of CountSketch, each input coordinate $v$ contributes to exactly one output coordinate indexed by $\sigma(v)$.
Since $\tilde r_t(v)=0$ for $v\notin\mathcal{I}_t$, only indices in $\mathcal{I}_t$ contribute, and each such $v$
causes exactly one signed update. Thus, $\mathrm{CS}_r(\tilde r_t)$ is computed with $L_t$ updates.

For the PCA projection, each output coordinate is a linear combination of the input coordinates.
Under the stated density condition, each nonzero input coordinate $v\in\mathcal{I}_t$ contributes to all $K_r$ outputs,
so it induces at least $K_r$ multiply-add contributions. Summing over $L_t$ nonzeros yields the lower bound $L_t\cdot K_r$.
\end{proof}

\textbf{Dense factors.}
For dense vectors $h_t\in\mathbb{R}^{d}$ and $g_t\in\mathbb{R}^{d}$, applying PCA projections $P_h h_t$ and $P_g g_t$
involves dense matrix-vector products (mixing all $d$ coordinates into $K_h$ or $K_g$ outputs),
whereas CountSketch uses one bucket update per input coordinate by construction.

\begin{algorithm}[t]
\caption{\textsc{SketchAggregate}: CountSketch Feature Construction}
\label{alg:sketch_detail}
\begin{algorithmic}[1]
\REQUIRE A sample $x$; LM head $W_{\lmhead}$; CountSketch operators $\mathrm{CS}_r,\mathrm{CS}_h,\mathrm{CS}_g$ with dims $(K_r,K_h,K_g)$;
temperature $\tau$; cumprob threshold $\rho_{\mathrm{cum}}$; top-$K$ cap $K_{\max}$; channel weights $(\lambda_{rh},\lambda_{gh})$.
\STATE Run the LM on $x$ to obtain logits $\{z_t\}_{t=1}^{T}$, final hidden states $\{h_t\}_{t=1}^{T}$, and next-token targets $\{y_t\}_{t=1}^{T}$.
\STATE $\phi_{RH} \gets \mathbf{0}\in\mathbb{R}^{K_rK_h}$;\quad $\phi_{GH} \gets \mathbf{0}\in\mathbb{R}^{K_gK_h}$
\FOR{$t=1,\ldots,T$}
    \STATE $\mathcal{C}_t \gets \mathrm{TopK}(z_t/\tau,\ K_{\max})$; ensure $y_t\in\mathcal{C}_t$
    \STATE $\bar p_t(v) \gets \mathrm{softmax}\big((z_t/\tau)|_{\mathcal{C}_t}\big)(v)\quad \forall v\in\mathcal{C}_t$
    \STATE Sort $\mathcal{C}_t$ by $\bar p_t$ to obtain $\pi_t$ (length $K_{\max}$)
    \STATE $S_t \gets \min\{s:\sum_{j=1}^{s}\bar p_t(\pi_t(j))\ge\rho_{\mathrm{cum}}\}$
    \STATE $\mathcal{S}_t \gets \{\pi_t(1),\ldots,\pi_t(S_t)\}\cup\{y_t\}$
    \STATE $\tilde p_t \gets$ renormalize $\bar p_t$ on $\mathcal{S}_t$ (set other $v\in\mathcal{C}_t$ to $0$)
    \STATE $\tilde r_t(v) \gets \tilde p_t(v) - \mathbf{1}[v=y_t]\quad \forall v\in\mathcal{S}_t$
    \STATE $\tilde g_t \gets \sum_{v\in\mathcal{S}_t} \tilde p_t(v)\,W_v \;-\; W_{y_t}$
    \STATE $\hat r_t \gets \mathrm{CS}_r(\tilde r_t)/\|\mathrm{CS}_r(\tilde r_t)\|_2$
    \STATE $\hat h_t \gets \mathrm{CS}_h(h_t)/\|\mathrm{CS}_h(h_t)\|_2$
    \STATE $\hat g_t \gets \mathrm{CS}_g(\tilde g_t)/\|\mathrm{CS}_g(\tilde g_t)\|_2$
    \STATE $\phi_{RH} \gets \phi_{RH} + \lambda_{rh}\,\mathrm{vec}(\hat r_t \otimes \hat h_t)$
    \STATE $\phi_{GH} \gets \phi_{GH} + \lambda_{gh}\,\mathrm{vec}(\hat g_t \otimes \hat h_t)$
\ENDFOR
\STATE \textbf{return} $\phi(x)=[\phi_{RH};\phi_{GH}]$
\end{algorithmic}
\end{algorithm}

\section{Variance Analysis of \textit{RISE} in Estimating Influence Function}\label{sec:variance_yide}

This section provides a variance analysis that is mathematically self-consistent for (i) TracIn-style random projection estimators and (ii) RISE-style factorized CountSketch estimators exploiting the LM-head outer-product structure. A key point is to separate:

\begin{itemize}[leftmargin=*]
    \item \textbf{Sketching variance:} randomness due to projection/sketching (this section).
    \item \textbf{Approximation bias:} restricting to the LM head and truncating the residual (deterministic, handled explicitly).
\end{itemize}

\subsection{Targets, Estimators, and Assumptions}

\paragraph{Full-parameter influence (TracIn target).}
Let $\nabla_\theta \ell(x)\in\mathbb{R}^{|\theta|}$. The full influence surrogate is
\[
\mathcal{I}^{\mathrm{true}}(x_i\!\to\!x_q) := \left\langle \nabla_\theta \ell(x_q), \nabla_\theta \ell(x_i)\right\rangle .
\]

\paragraph{LM-head influence (RISE structural target).}
Let $W_{\lmhead}\in\mathbb{R}^{V\times d}$ and $\nabla_{W_{\lmhead}}\ell(x)\in\mathbb{R}^{V\times d}$. For a sequence $x$ of length $T$, the LM-head gradient admits the exact outer-product decomposition
\[
\nabla_{W_{\lmhead}}\ell(x)=\sum_{t=1}^T r_t(x)\, h_t(x)^\top,
\qquad r_t(x)=p_t(x)-\mathbf{1}_{y_t}\in\mathbb{R}^V,\;\; h_t(x)\in\mathbb{R}^d.
\]
The corresponding LM-head influence is
\[
\mathcal{I}^{\mathrm{head}}(x_i\!\to\!x_q):=\left\langle \nabla_{W_{\lmhead}}\ell(x_q),\nabla_{W_{\lmhead}}\ell(x_i)\right\rangle_F .
\]

\paragraph{Truncation (approximation target).}
Let $\tilde r_t$ be a sparse approximation to $r_t$ (e.g., top-mass truncation with optional renormalization). Define the truncated head gradient
\[
\widetilde{\nabla_{W_{\lmhead}}\ell}(x):=\sum_{t=1}^T \tilde r_t(x)\, h_t(x)^\top,
\qquad
\tilde{\mathcal{I}}^{\mathrm{head}}(x_i\!\to\!x_q):=\left\langle \widetilde{\nabla_{W_{\lmhead}}\ell}(x_q), \widetilde{\nabla_{W_{\lmhead}}\ell}(x_i)\right\rangle_F.
\]
\textbf{Important:} the sketching estimators below are unbiased for $\tilde{\mathcal{I}}^{\mathrm{head}}$ (not for $\mathcal{I}^{\mathrm{true}}$ unless additional assumptions are made).

\paragraph{TracIn estimator.}
Let $R\in\mathbb{R}^{K\times |\theta|}$ have i.i.d.\ entries $R_{kj}\sim \mathcal{N}(0,1/K)$ (the Gaussian case) or Rademacher $\pm 1/\sqrt{K}$ (similar bounds). The TracIn estimator is
\[
\widehat{\mathcal{I}}^{\mathrm{TracIn}}:=\left\langle R\,\nabla_\theta \ell(x_q),\, R\,\nabla_\theta \ell(x_i)\right\rangle.
\]

\paragraph{CountSketch operator.}
A CountSketch $\mathrm{CS}:\mathbb{R}^D\to\mathbb{R}^K$ is defined by a bucket hash $\eta:[D]\to[K]$ and a sign hash $s:[D]\to\{\pm1\}$:
\[
(\mathrm{CS}(x))_k=\sum_{j: \eta(j)=k} s(j)\, x_j .
\]
We assume the hash families used for different factors (residual/hidden/projection), i.e.,
$(\eta_r,s_r)$, $(\eta_h,s_h)$, $(\eta_g,s_g)$, are independent.

\paragraph{RISE RH/GH estimators (analysis version, \emph{no per-vector normalization}).}
For a single token position $t$ in each sample (token-level analysis), define:
\[
A_t := \big\langle \mathrm{CS}_r(\tilde r_t(x_q)), \mathrm{CS}_r(\tilde r_t(x_i))\big\rangle,\quad
B_t := \big\langle \mathrm{CS}_h(h_t(x_q)), \mathrm{CS}_h(h_t(x_i))\big\rangle,
\]
and the RH estimator
\[
\widehat{\mathcal{I}}^{\mathrm{RISE}}_{RH,t}:=A_t\,B_t .
\]
For GH, let $\tilde g_t(x):=W_{\lmhead}^\top \tilde r_t(x)\in\mathbb{R}^d$ and
\[
C_t := \big\langle \mathrm{CS}_g(\tilde g_t(x_q)), \mathrm{CS}_g(\tilde g_t(x_i))\big\rangle,\qquad
\widehat{\mathcal{I}}^{\mathrm{RISE}}_{GH,t}:=C_t\,B_t .
\]
A two-channel fused estimator at token $t$ is
\[
\widehat{\mathcal{I}}^{\mathrm{RISE}}_{t} := \lambda_{rh}\widehat{\mathcal{I}}^{\mathrm{RISE}}_{RH,t}
+ \lambda_{gh}\widehat{\mathcal{I}}^{\mathrm{RISE}}_{GH,t}.
\]

\paragraph{Remark (about the implementation normalization).}
If the implementation uses $\hat r_t=\mathrm{CS}_r(\tilde r_t)/\|\mathrm{CS}_r(\tilde r_t)\|_2$ and similarly for $\hat h_t,\hat g_t$, the resulting estimator becomes \emph{biased} (nonlinear normalization). The variance bounds below apply to the \emph{unnormalized} version, which is the standard setting where CountSketch inner products are unbiased.

\subsubsection{Assumptions}

\begin{assumption}[Bounded norms]\label{assump:bounded_norms_corrected}
There exist constants $B_r,B_{\tilde r},B_h>0$ such that for all $x,t$:
\[
\|r_t(x)\|_2\le B_r,\qquad \|\tilde r_t(x)\|_2\le B_{\tilde r},\qquad \|h_t(x)\|_2\le B_h.
\]
\end{assumption}

\begin{assumption}[Bounded full gradient]\label{assump:bounded_full_grad_corrected}
There exists $B_\theta>0$ such that $\|\nabla_\theta \ell(x)\|_2\le B_\theta$ for all $x$.
\end{assumption}

\begin{assumption}[Truncation error bound]\label{assump:trunc_error_corrected}
There exists $\varepsilon_{\mathrm{trunc}}\ge 0$ such that for all $x,t$,
\[
\|r_t(x)-\tilde r_t(x)\|_2 \le \varepsilon_{\mathrm{trunc}}.
\]
\end{assumption}

\begin{assumption}[LM-head operator norm]\label{assump:head_op_norm_corrected}
$\|W_{\lmhead}\|_{op}=\sigma_1(W_{\lmhead})$, hence
$\|\tilde g_t(x)\|_2=\|W_{\lmhead}^\top\tilde r_t(x)\|_2\le \sigma_1(W_{\lmhead})\|\tilde r_t(x)\|_2\le \sigma_1(W_{\lmhead})B_{\tilde r}$.
\end{assumption}

\begin{assumption}[Energy concentration into LM head (optional, for comparing to full TracIn)]\label{assump:energy_fraction_corrected}
There exists $\gamma\in(0,1]$ such that for all $x$,
\[
\|\nabla_{W_{\lmhead}}\ell(x)\|_F^2 \le \gamma \|\nabla_\theta \ell(x)\|_2^2 .
\]
\end{assumption}

\subsection{Supporting Lemmas}

\begin{lemma}[Random projection inner-product variance (Gaussian)]\label{lem:rp_var_corrected}
Let $R\in\mathbb{R}^{K\times |\theta|}$ have i.i.d.\ entries $\mathcal{N}(0,1/K)$. For any $u,v\in\mathbb{R}^{|\theta|}$, define $\widehat Z=\langle Ru,Rv\rangle$. Then
\[
\mathbb{E}[\widehat Z]=u^\top v,\qquad
\mathrm{Var}[\widehat Z]=\frac{1}{K}\Big(\|u\|_2^2\|v\|_2^2+(u^\top v)^2\Big).
\]
(For Rademacher $\pm 1/\sqrt{K}$, the same scaling holds with an absolute-constant factor.)
\end{lemma}

\begin{proof}
Write $R$ by rows: $R=[r_1^\top;\ldots;r_K^\top]$ with $r_k\sim\mathcal{N}(0,\frac{1}{K}I)$ independent.
Let $X_k=(r_k^\top u)(r_k^\top v)$. Then $\widehat Z=\sum_{k=1}^K X_k$ and the $X_k$ are i.i.d.
By isotropy, $\mathbb{E}[X_k]=\frac{1}{K}u^\top v$.
By Isserlis' theorem for zero-mean Gaussians,
\[
\mathbb{E}[X_k^2]=\mathbb{E}[(r_k^\top u)^2(r_k^\top v)^2]
=\frac{1}{K^2}\Big(\|u\|_2^2\|v\|_2^2+2(u^\top v)^2\Big).
\]
Hence
\[
\mathrm{Var}(X_k)=\mathbb{E}[X_k^2]-\mathbb{E}[X_k]^2
=\frac{1}{K^2}\Big(\|u\|_2^2\|v\|_2^2+(u^\top v)^2\Big),
\]
and $\mathrm{Var}(\widehat Z)=K\,\mathrm{Var}(X_1)$ gives the claim.
\end{proof}

\begin{lemma}[CountSketch inner-product unbiasedness and variance]\label{lem:cs_ip_corrected}
Let $\mathrm{CS}:\mathbb{R}^D\to\mathbb{R}^K$ be CountSketch with a bucket hash $\eta$ and sign hash $s$ such that
$\mathbb{P}(\eta(i)=\eta(j))=1/K$ for $i\ne j$, and $\mathbb{E}[s(i)]=0$, $\mathbb{E}[s(i)s(j)]=0$ for $i\ne j$.
For any $x,y\in\mathbb{R}^D$, define $\widehat Z=\langle \mathrm{CS}(x),\mathrm{CS}(y)\rangle$. Then
\[
\mathbb{E}[\widehat Z]=x^\top y,\qquad
\mathrm{Var}[\widehat Z]\le \frac{1}{K}\|x\|_2^2\|y\|_2^2.
\]
\end{lemma}

\begin{proof}
Expand
\[
\widehat Z=\sum_{k=1}^K \Big(\sum_{i:\eta(i)=k}s(i)x_i\Big)\Big(\sum_{j:\eta(j)=k}s(j)y_j\Big)
=\sum_{i,j} x_i y_j\, s(i)s(j)\,\mathbf{1}[\eta(i)=\eta(j)].
\]
Taking expectation over $s$, the terms with $i\ne j$ vanish because $\mathbb{E}[s(i)s(j)]=0$, leaving
$\mathbb{E}[\widehat Z]=\sum_i x_i y_i=x^\top y$.

Let $E:=\widehat Z-x^\top y=\sum_{i\ne j} x_i y_j\, s(i)s(j)\,\mathbf{1}[\eta(i)=\eta(j)]$.
Then $\mathbb{E}[E]=0$ and $\mathrm{Var}(\widehat Z)=\mathbb{E}[E^2]$.
Using sign-independence, cross terms vanish unless indices pair up, which yields the standard bound
\[
\mathbb{E}[E^2]\le \sum_{i\ne j} x_i^2 y_j^2 \,\mathbb{P}(\eta(i)=\eta(j))
=\frac{1}{K}\sum_{i\ne j} x_i^2 y_j^2
\le \frac{1}{K}\Big(\sum_i x_i^2\Big)\Big(\sum_j y_j^2\Big)=\frac{1}{K}\|x\|_2^2\|y\|_2^2.
\]
\end{proof}

\begin{lemma}[Variance of a product of independent CountSketch inner products]\label{lem:factorized_var_corrected}
Let
\[
A=\langle \mathrm{CS}_r(\tilde r),\mathrm{CS}_r(\tilde r')\rangle,\quad
B=\langle \mathrm{CS}_h(h),\mathrm{CS}_h(h')\rangle,
\]
where $\mathrm{CS}_r$ and $\mathrm{CS}_h$ use independent hash families with output sizes $K_r$ and $K_h$.
Define $\alpha=\langle \tilde r,\tilde r'\rangle$ and $\beta=\langle h,h'\rangle$.
Then $\mathbb{E}[AB]=\alpha\beta$ and
\[
\mathrm{Var}[AB]\le
\frac{\|\tilde r\|_2^2\|\tilde r'\|_2^2\|h\|_2^2\|h'\|_2^2}{K_rK_h}
+\frac{\alpha^2\|h\|_2^2\|h'\|_2^2}{K_h}
+\frac{\beta^2\|\tilde r\|_2^2\|\tilde r'\|_2^2}{K_r}.
\]
\end{lemma}

\begin{proof}
By Lemma~\ref{lem:cs_ip_corrected}, $\mathbb{E}[A]=\alpha$ and $\mathbb{E}[B]=\beta$.
Independence of the sketch randomness implies $A$ and $B$ are independent, hence $\mathbb{E}[AB]=\mathbb{E}[A]\mathbb{E}[B]=\alpha\beta$.

For independent $A,B$,

\begin{align*}
\mathrm{Var}(AB)
&= \mathbb{E}[A^2]\mathbb{E}[B^2]-(\mathbb{E}[A]\mathbb{E}[B])^2 \\
&= (\mathrm{Var}(A)+\alpha^2)(\mathrm{Var}(B)+\beta^2)-\alpha^2\beta^2 \\
&= \mathrm{Var}(A)\mathrm{Var}(B)+\alpha^2\mathrm{Var}(B)+\beta^2\mathrm{Var}(A).
\end{align*}
Using Lemma~\ref{lem:cs_ip_corrected},
\[
\mathrm{Var}(A)\le \frac{\|\tilde r\|_2^2\|\tilde r'\|_2^2}{K_r},
\qquad
\mathrm{Var}(B)\le \frac{\|h\|_2^2\|h'\|_2^2}{K_h},
\]
and substituting yields the stated bound.
\end{proof}

\begin{lemma}[Renormalized top-mass truncation implies an explicit $\ell_2$ error bound]\label{lem:renorm_trunc_bound}
Let $p_t\in\Delta^{V}$ and $\mathcal{S}_t\subset[V]$ with mass $\rho=\sum_{v\in \mathcal{S}_t}p_t(v)$. Define $\tilde p_t(v)=p_t(v)/\rho$ for $v\in \mathcal{S}_t$ and $0$ otherwise. Then
\[
\|p_t-\tilde p_t\|_1=2(1-\rho),\qquad \|p_t-\tilde p_t\|_2\le 2(1-\rho).
\]
Consequently, for residuals $r_t=p_t-\mathbf{1}_{y_t}$ and $\tilde r_t=\tilde p_t-\mathbf{1}_{y_t}$ we have $\|r_t-\tilde r_t\|_2\le 2(1-\rho)$.
\end{lemma}

\begin{proof}
Outside $\mathcal{S}_t$, $\tilde p_t=0$, so $\|p_t-\tilde p_t\|_1$ contributes $\sum_{v\notin \mathcal{S}_t}p_t(v)=1-\rho$.
Inside $\mathcal{S}_t$, $p_t-\tilde p_t=p_t-p_t/\rho=p_t(1-1/\rho)$ and $|1-1/\rho|=(1-\rho)/\rho$, hence the $\ell_1$ contribution is
$\sum_{v\in \mathcal{S}_t} p_t(v)\cdot (1-\rho)/\rho = (1-\rho)$.
Thus $\|p_t-\tilde p_t\|_1=2(1-\rho)$ and $\|p_t-\tilde p_t\|_2\le \|p_t-\tilde p_t\|_1$.
Finally $r_t-\tilde r_t=p_t-\tilde p_t$.
\end{proof}

\subsection{Main Variance Theorem}

\begin{theorem}[Variance bounds for RISE vs.\ TracIn (self-consistent form)]\label{thm:variance_main_corrected}
Fix a training-query pair $(x_i,x_q)$ and a token position $t$ (token-level analysis). Under
Assumptions~\ref{assump:bounded_norms_corrected}--\ref{assump:head_op_norm_corrected}:

\textbf{(i) TracIn variance (full gradient).}
Let $u=\nabla_\theta \ell(x_q)$ and $v=\nabla_\theta \ell(x_i)$, and let $R$ be Gaussian as in Lemma~\ref{lem:rp_var_corrected}. Then
\[
\mathrm{Var}\!\left[\widehat{\mathcal{I}}^{\mathrm{TracIn}}\right]
=\frac{1}{K}\Big(\|u\|_2^2\|v\|_2^2+(u^\top v)^2\Big)
\le \frac{2B_\theta^4}{K}.
\]

\textbf{(ii) RH channel variance (token-level).}
Let $\tilde r^q=\tilde r_t(x_q)$, $\tilde r^i=\tilde r_t(x_i)$, $h^q=h_t(x_q)$, $h^i=h_t(x_i)$ and define
$\alpha=\langle \tilde r^q,\tilde r^i\rangle$, $\beta=\langle h^q,h^i\rangle$.
Then
\[
\mathrm{Var}\!\left[\widehat{\mathcal{I}}^{\mathrm{RISE}}_{RH,t}\right]
\le
\frac{\|\tilde r^q\|_2^2\|\tilde r^i\|_2^2\|h^q\|_2^2\|h^i\|_2^2}{K_rK_h}
+\frac{\alpha^2\|h^q\|_2^2\|h^i\|_2^2}{K_h}
+\frac{\beta^2\|\tilde r^q\|_2^2\|\tilde r^i\|_2^2}{K_r}.
\]
In particular, using $\|\tilde r\|\le B_{\tilde r}$ and $\|h\|\le B_h$ gives the worst-case bound
\[
\mathrm{Var}\!\left[\widehat{\mathcal{I}}^{\mathrm{RISE}}_{RH,t}\right]
\le B_{\tilde r}^4B_h^4\Big(\frac{1}{K_rK_h}+\frac{1}{K_h}+\frac{1}{K_r}\Big).
\]

\textbf{(iii) GH channel variance (token-level).}
Let $\tilde g^q=W_{\lmhead}^\top\tilde r^q$, $\tilde g^i=W_{\lmhead}^\top\tilde r^i$, and $\gamma_g=\langle \tilde g^q,\tilde g^i\rangle$.
Then
\[
\mathrm{Var}\!\left[\widehat{\mathcal{I}}^{\mathrm{RISE}}_{GH,t}\right]
\le
\frac{\|\tilde g^q\|_2^2\|\tilde g^i\|_2^2\|h^q\|_2^2\|h^i\|_2^2}{K_gK_h}
+\frac{\gamma_g^2\|h^q\|_2^2\|h^i\|_2^2}{K_h}
+\frac{\beta^2\|\tilde g^q\|_2^2\|\tilde g^i\|_2^2}{K_g}.
\]
Moreover, $\|\tilde g\|_2\le \sigma_1(W_{\lmhead})B_{\tilde r}$ implies the worst-case bound
\[
\mathrm{Var}\!\left[\widehat{\mathcal{I}}^{\mathrm{RISE}}_{GH,t}\right]
\le \sigma_1(W_{\lmhead})^4B_{\tilde r}^4B_h^4\Big(\frac{1}{K_gK_h}+\frac{1}{K_h}+\frac{1}{K_g}\Big).
\]

\textbf{(iv) Two-channel fusion and covariance.}
Let $\widehat{\mathcal{I}}^{\mathrm{RISE}}_{t}=\lambda_{rh}\widehat{\mathcal{I}}^{\mathrm{RISE}}_{RH,t}+\lambda_{gh}\widehat{\mathcal{I}}^{\mathrm{RISE}}_{GH,t}$.
If the RH and GH channels share the same hidden-state sketch $B_t$ but use independent residual/projection sketches ($\mathrm{CS}_r$ independent of $\mathrm{CS}_g$), then
\[
\mathrm{Var}\!\left[\widehat{\mathcal{I}}^{\mathrm{RISE}}_{t}\right]
=
\lambda_{rh}^2\mathrm{Var}[\widehat{\mathcal{I}}^{\mathrm{RISE}}_{RH,t}]
+\lambda_{gh}^2\mathrm{Var}[\widehat{\mathcal{I}}^{\mathrm{RISE}}_{GH,t}]
+2\lambda_{rh}\lambda_{gh}\,\mathrm{Cov}\!\left[A_tB_t,C_tB_t\right],
\]
and the covariance admits the exact identity
\[
\mathrm{Cov}\!\left[A_tB_t,C_tB_t\right]=\mathbb{E}[A_tC_t]\mathrm{Var}(B_t)=\alpha\,\gamma_g\,\mathrm{Var}(B_t),
\]
where $\mathrm{Var}(B_t)\le \|h^q\|_2^2\|h^i\|_2^2/K_h$ by Lemma~\ref{lem:cs_ip_corrected}.
(If RH and GH use independent hidden sketches, then the covariance term is zero.)

\textbf{(v) Variance reduction relative to TracIn.}
\begin{itemize}[leftmargin=*]
\item \textit{Estimator Target:} The bounds above quantify the sketching variance relative to the truncated LM-head influence $\tilde{\mathcal{I}}^{\mathrm{head}}$. Note that $\widehat{\mathcal{I}}^{\mathrm{RISE}}$ is an unbiased estimator of $\tilde{\mathcal{I}}^{\mathrm{head}}$, but biased with respect to $\mathcal{I}^{\mathrm{true}}$ due to head-restriction and truncation.
\item \textit{Dimensionality-Adjusted Comparison:} Consider a matched memory budget where the explicit feature dimension is fixed to $K$ (i.e., set $K_r K_h = K$). The variance of the dominant product term in $\widehat{\mathcal{I}}^{\mathrm{RISE}}_{RH,t}$ scales as:
\[
\mathcal{O}\left( \frac{\|\tilde r^q\|^2\|h^q\|^2\|\tilde r^i\|^2\|h^i\|^2}{K} \right).
\]
By contrast, TracIn on the full gradient has leading scaling $\|\nabla_\theta \ell(x_q)\|^2\|\nabla_\theta \ell(x_i)\|^2/K$.
Under the (optional) energy fraction assumption $\|\nabla_{W_{\lmhead}}\ell(x)\|_F^2\le\gamma\|\nabla_\theta \ell(x)\|_2^2$ and if truncation does not increase the head-gradient energy substantially, the head-based estimator can have substantially smaller variance due to $\gamma\ll 1$.
\item \textit{Effect of Truncation:} Furthermore, under a non-expansive truncation condition (where $\|\tilde r_t\|_2 \le \|r_t\|_2$, satisfied by zero-masking truncation), the product-term variance is strictly reduced compared to dense head-gradient sketching:
\[
\frac{\|\tilde r^q\|_2^2\|\tilde r^i\|_2^2}{\|r^q\|_2^2\|r^i\|_2^2}\le 1,
\]
and is strictly $<1$ if at least one truncation is strict.
\end{itemize}
\end{theorem}

\begin{proof}
Part (i) is Lemma~\ref{lem:rp_var_corrected} applied to $u=\nabla_\theta \ell(x_q)$ and $v=\nabla_\theta \ell(x_i)$ plus Cauchy--Schwarz.

Parts (ii) and (iii) follow by applying Lemma~\ref{lem:factorized_var_corrected} to
$(\tilde r,\tilde r')=(\tilde r^q,\tilde r^i)$ and $(h,h')=(h^q,h^i)$ for RH, and
$(\tilde r,\tilde r')=(\tilde g^q,\tilde g^i)$ and $(h,h')=(h^q,h^i)$ for GH, together with Assumption~\ref{assump:head_op_norm_corrected}.

For (iv), write $\widehat{\mathcal{I}}_{RH,t}=A_tB_t$ and $\widehat{\mathcal{I}}_{GH,t}=C_tB_t$.
Since $(A_t,C_t)$ are independent of $B_t$ and $A_t$ is independent of $C_t$ (independent sketch families),
\[
\mathrm{Cov}(A_tB_t,C_tB_t)=\mathbb{E}[A_tC_tB_t^2]-\mathbb{E}[A_tB_t]\mathbb{E}[C_tB_t]
=\mathbb{E}[A_tC_t]\big(\mathbb{E}[B_t^2]-\mathbb{E}[B_t]^2\big)
=\mathbb{E}[A_tC_t]\mathrm{Var}(B_t).
\]
Moreover $\mathbb{E}[A_tC_t]=\mathbb{E}[A_t]\mathbb{E}[C_t]=\alpha\gamma_g$.
\end{proof}

\subsection{Bias-Variance Decomposition}

While the analysis above focuses on sketching variance, the total error of the \textit{RISE} estimator with respect to the true influence $\mathcal{I}^{\mathrm{true}}$ admits the following decomposition:
\[
\widehat{\mathcal{I}}^{\mathrm{RISE}} - \mathcal{I}^{\mathrm{true}}
=
\underbrace{\big(\widehat{\mathcal{I}}^{\mathrm{RISE}}-\tilde{\mathcal{I}}^{\mathrm{head}}\big)}_{\text{Sketching Variance (Stochastic)}}
+
\underbrace{\big(\tilde{\mathcal{I}}^{\mathrm{head}}-\mathcal{I}^{\mathrm{head}}\big)}_{\text{Truncation Bias (Deterministic)}}
+
\underbrace{\big(\mathcal{I}^{\mathrm{head}}-\mathcal{I}^{\mathrm{true}}\big)}_{\text{Structural Bias (Deterministic)}}.
\]
The sketching variance is controlled by Theorem~\ref{thm:variance_main_corrected}. The truncation bias is bounded by the residual mass discarded via Assumption~\ref{assump:trunc_error_corrected} and Lemma~\ref{lem:renorm_trunc_bound}. The structural bias depends on the gradient energy concentration in the LM head (Observation 1), representing the irreducible error from excluding the model body parameters.

\section{Hyperparameters}
\label{sec:appendix_hyperparameters}
We present the hyperparameters in this paper in Table~\ref{tab:hyperparameters}.

\begin{table*}[htbp]
\centering
\caption{Hyperparameters used in all experiments.}
\label{tab:hyperparameters}
\scriptsize
\begin{tabular}{llp{6cm}}
\toprule
\textbf{Parameter} & \textbf{Value} & \textbf{Description} \\
\midrule
\multicolumn{3}{l}{\textbf{\textit{RISE Configuration}}} \\
\midrule
\multicolumn{3}{l}{\textit{Channel Configuration}} \\
\texttt{fusion\_mode} & \texttt{"rh+gh"} & Active channels (RH + GH) \\
$\lambda_{rh}$ & $0.7$ & RH channel weight \\
$\lambda_{gh}$ & $1.0$ & GH channel weight \\
\midrule
\multicolumn{3}{l}{\textit{CountSketch Dimensions (Pythia-1B)}} \\
$K_r$ & $128$ / $64$ & Residual sketch dimension (58MB / 14MB config) \\
$K_h$ & $24$ / $12$ & Hidden state sketch dimension \\
$K_g$ & $128$ / $64$ & Gradient sketch dimension \\
\midrule
\multicolumn{3}{l}{\textit{CountSketch Dimensions (OLMo-3-32B)}} \\
$K_r$ & $48$ / $24$ / $16$ & Residual sketch dimension (39MB / 4.5MB / 3.4MB config) \\
$K_h$ & $64$ / $8$ / $8$ & Hidden state sketch dimension \\
$K_g$ & $16$ / $35$ / $28$ & Gradient sketch dimension \\
\midrule
\multicolumn{3}{l}{\textit{Adaptive Top-L}} \\
\texttt{adaptive\_topL} & True & Enable cumulative-probability truncation \\
$\theta_{cum}$ & $0.92$ & Cumulative probability threshold \\
\texttt{min\_topL} & $4$ & Minimum kept candidates \\
\texttt{max\_topL\_cap} & $256$ & Top-L upper bound \\
\midrule
\multicolumn{3}{l}{\textit{Other}} \\
\texttt{seq\_len} & $512$ & Maximum sequence length \\
\texttt{normalize\_sample} & True & L2 normalize final vector \\
\texttt{seed} & $42$ & Random seed for reproducibility \\
\midrule
\multicolumn{3}{l}{\textbf{\textit{TrackStar Configuration}}} \\
\midrule
\texttt{projection\_dim} & $16$ & Random projection dimension \\
\texttt{projection\_type} & \texttt{rademacher} & Projection matrix type \\
\texttt{loss\_fn} & \texttt{ce} & Cross-entropy loss \\
\texttt{loss\_reduction} & \texttt{mean} & Loss reduction method \\
\texttt{include\_bias} & False & Exclude bias gradients \\
\texttt{precision} & \texttt{bf16} & Model precision \\
\midrule
\multicolumn{3}{l}{\textbf{\textit{RapidIn Configuration}}} \\
\midrule
\texttt{RapidGrad\_K} & $65536$ & Compressed gradient dimension \\
\texttt{shuffle\_lambda} & $20$ & Shuffle parameter for RapidGrad \\
\texttt{use\_zo} & False & Use backpropagation gradients \\
\texttt{max\_length} & $512$ & Maximum sequence length \\
\texttt{seed} & $42$ & Random seed \\
\midrule
\multicolumn{3}{l}{\textbf{\textit{ZO-Inf Configuration}}} \\
\midrule
\texttt{RapidGrad\_K} & $65536$ & Compressed gradient dimension \\
\texttt{shuffle\_lambda} & $20$ & Shuffle parameter for RapidGrad \\
\texttt{use\_zo} & True & Use zeroth-order gradients \\
\texttt{zo\_eps} & $0.001$ & Perturbation scale ($\epsilon$) \\
\texttt{zo\_sample} & $4$ & Number of ZO samples per gradient \\
\texttt{max\_length} & $512$ & Maximum sequence length \\
\texttt{seed} & $42$ & Random seed \\
\bottomrule
\end{tabular}
\end{table*}

\begin{table*}[htbp]
\centering
\caption{Dataset composition for the three evaluation tasks.}
\label{tab:task_data_distribution}
\scriptsize
\setlength{\tabcolsep}{6pt}
\renewcommand{\arraystretch}{1.15}
\begin{tabular}{@{}lcccl@{}}
\toprule
\textbf{Task} & \textbf{Pool Size} & \textbf{Positives} & \textbf{Pos \%} & \textbf{Description} \\
\midrule
Howdy! Backdoor & 5,000 & 438 & 8.8\% & 
Alpaca-based pool; positives contain trigger \texttt{howdy!}~\cite{lin2024tokenwiseinfluentialtrainingdata} \\
\addlinespace[2pt]
Finance–Medical & 5,000 & 500 & 10.0\% & 
Medical QA~\cite{lavita_medical_qa_datasets_2023} vs Finance~\cite{gaurang_bharti_2024} \\
\addlinespace[2pt]
Brain Rot & 5,000 & 500 & 10.0\% & 
High-quality control vs junk data~\cite{xing2025llmsbrainrot} \\
\bottomrule
\end{tabular}

\vspace{0.5em}
\parbox{\textwidth}{\footnotesize \textbf{Note:} For Howdy!, we use 100 WebQuestions test generations containing the trigger. Positives for Finance–Medical are medical samples; for Brain Rot, positives are high-quality control samples.}
\end{table*}
\section{Additional Experimental Results}
\label{sec:appendix_experiments}

\paragraph{Testbed.}
All experiments are run on servers with the following configurations: H200 setup: 8 NVIDIA H200 GPUs (each with 141GB GPU memory); GH200 setup: 1 NVIDIA GH200 GPU (96GB GPU memory).

\subsection{Complete Per-Layer Energy Distribution}
\label{sec:appendix_obs1}

This subsection provides comprehensive results on per-layer gradient energy distribution across different models and tasks. Figure~\ref{fig:obs1_pretrain_finetune_same_energy} demonstrates that gradient-energy profiles remain largely stable before and after fine-tuning. Table~\ref{tab:full-gradient-storage} reports the full storage requirements for last-layer gradients, while Table~\ref{tab:grad-baseline} presents an ablation comparing LM-Head-Full-gradient sketch with \textit{RISE}.

\begin{figure*}[htbp]
\centering
\includegraphics[width=\linewidth]{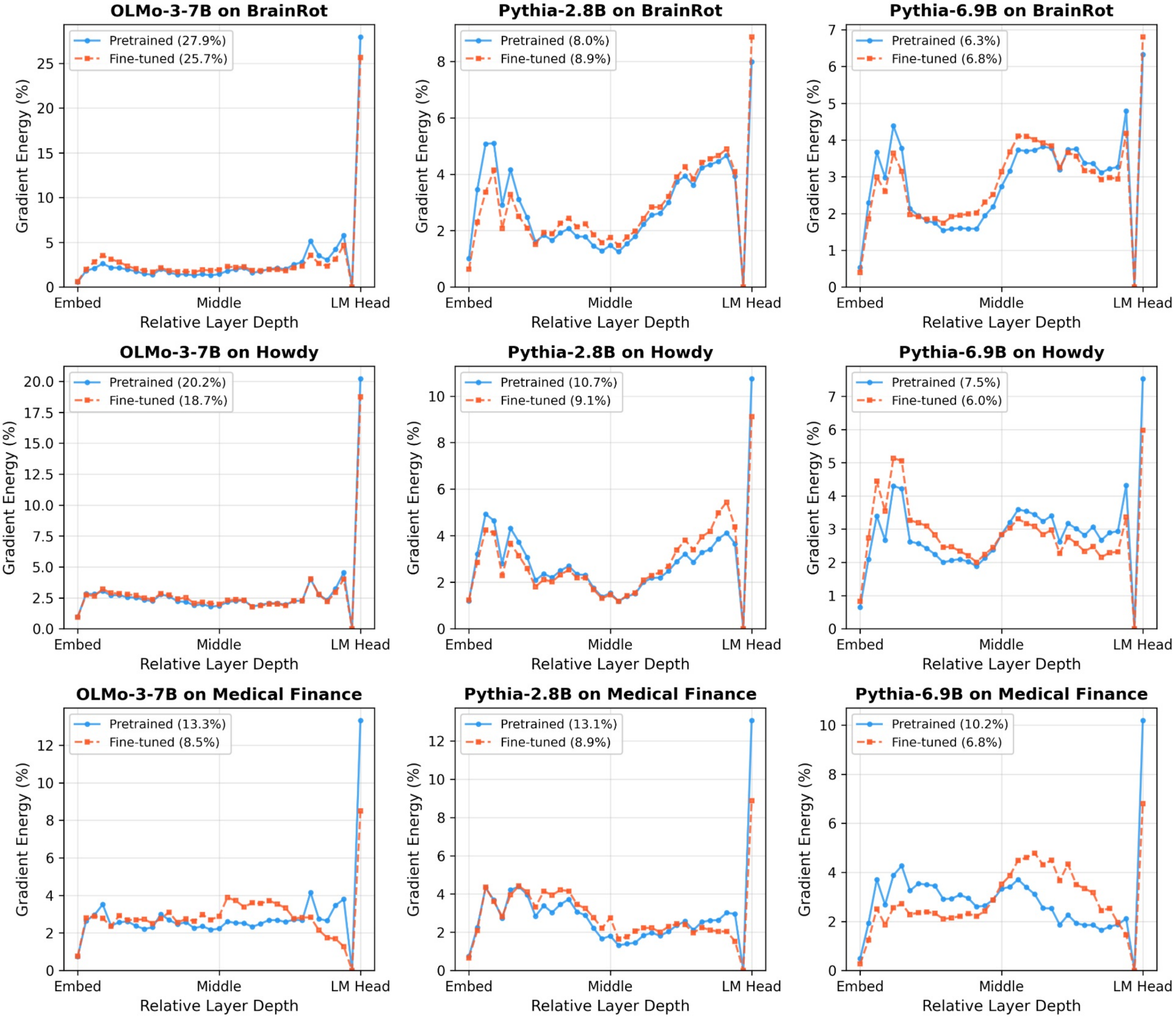}
\caption{\textbf{Per-layer gradient-energy profiles are largely preserved after fine-tuning.}
For three tasks (BrainRot, Howdy, Finance--Medical) and three models (OLMo-3-7B, Pythia-2.8B, Pythia-6.9B), we plot the fraction of total gradient energy attributed to each layer (Embed $\rightarrow$ LM Head) for the pretrained checkpoint (blue) and the fine-tuned checkpoint (orange).
The curves nearly overlap; numbers in legends denote the LM-head energy share, which changes only mildly after fine-tuning.
This stability motivates estimating layer importance (and selecting readout-side signals) from the pretrained model without task-specific per-layer gradient profiling.}
\label{fig:obs1_pretrain_finetune_same_energy}
\end{figure*}

\begin{table}[t]
\centering
\caption{Full Last-Layer Gradient Storage Requirements ($N$=5,000 samples, fp16)}
\label{tab:full-gradient-storage}
\setlength{\tabcolsep}{6pt}
\begin{tabular}{lrrrc}
\toprule
\textbf{Model} & \textbf{Vocab} & \textbf{Hidden} & \textbf{V$\times$D} & \textbf{Index Size} \\
\midrule
Pythia-14M & 50,304 & 128 & 6.4M & 60 GB \\
Pythia-1B & 50,304 & 2,048 & 103M & 959 GB \\
OLMo-7B & 100,352 & 4,096 & 411M & 3.8 TB \\
OLMo-32B & 100,352 & 5,120 & 514M & 4.8 TB \\
\bottomrule
\end{tabular}
\end{table}

\begin{table*}[t]
\centering
\caption{Ablation: LM-Head-Full-gradient sketch vs.\ \textit{RISE} on three tasks (Pythia-1B). Grad-Sketch applies CountSketch to LM Head gradients; \textit{RISE} restricts to the LM head.}
\label{tab:grad-baseline}
\begin{tabular}{@{}llc cccc@{}}
\toprule
\multirow{2}{*}{Dataset} & \multirow{2}{*}{Method} & \multirow{2}{*}{Index} & \multicolumn{2}{c}{Top-5} & \multicolumn{2}{c}{Top-50} \\
\cmidrule(lr){4-5} \cmidrule(lr){6-7}
& & & auPRC & auROC & auPRC & auROC \\
\midrule
\multirow{2}{*}{Fin--Med} 
  & Grad-Sketch & 156 MB & \textbf{1.000} & \textbf{1.000} & \textbf{0.992} & 0.986 \\
  & \textit{RISE} & \textbf{58.6 MB} & 0.989 & 0.994 & 0.981 & \textbf{0.990} \\
\midrule
\multirow{2}{*}{BrainRot} 
  & Grad-Sketch & 156 MB & \textbf{1.000} & \textbf{1.000} & 0.724 & 0.872 \\
  & \textit{RISE} & \textbf{58.6 MB} & 0.847 & 0.949 & \textbf{0.829} & \textbf{0.930} \\
\midrule
\multirow{2}{*}{Howdy} 
  & Grad-Sketch & 156 MB & 0.679 & 0.833 & 0.890 & 0.960 \\
  & \textit{RISE} & \textbf{58.6 MB} & \textbf{0.996} & \textbf{0.997} & \textbf{0.939} & \textbf{0.967} \\
\bottomrule
\end{tabular}
\end{table*}

\subsection{Complete Per-Layer Discriminativeness}
\label{sec:appendix_obs2}

This subsection analyzes the layer-wise discriminativeness of hidden-state representations. Figure~\ref{fig:obs2_discriminativeness} shows the U-shaped discriminativeness pattern on C4 across model scales. Figure~\ref{fig:obs2_pretrain_finetune_same_ushape} further shows that this pattern persists both before and after fine-tuning across multiple models and tasks.

\begin{figure*}[htbp]
\centering
\includegraphics[width=0.78\linewidth]{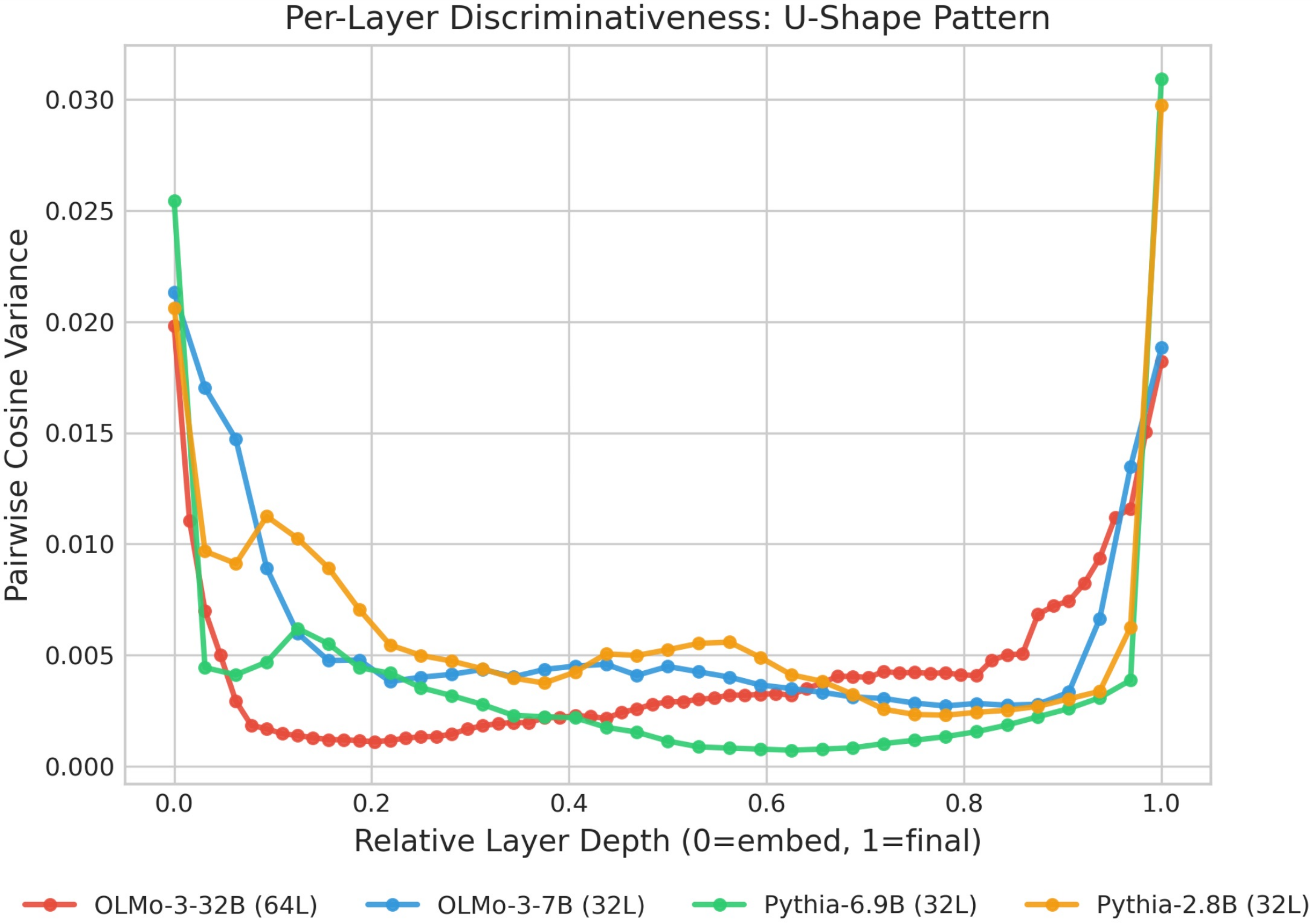}
\caption{\textbf{Final-layer discriminativeness across model scales.}
Per-layer pairwise cosine variance for four models on C4.
A clear U-shape emerges: variance drops in middle layers due to hidden collapse and recovers near the final layer.}
\label{fig:obs2_discriminativeness}
\end{figure*}

\begin{figure*}[htbp]
\centering
\includegraphics[width=\linewidth]{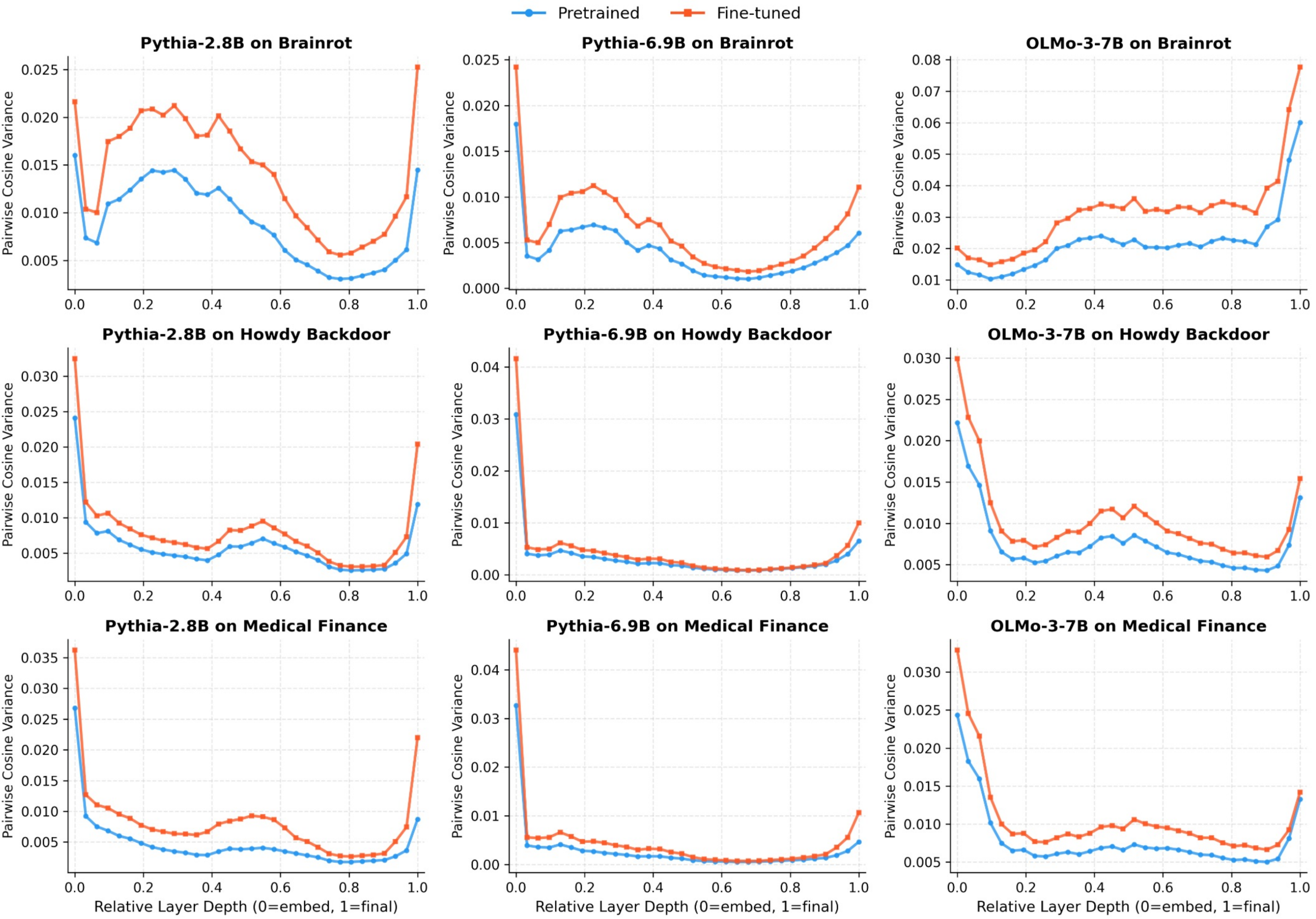}
\caption{\textbf{Hidden-state discriminativeness follows a stable U-shape before and after fine-tuning.} For each layer, we compute the variance of pairwise cosine similarities among token hidden states (higher variance more discriminative representations).
Across three fine-tuning tasks (BrainRot, Howdy Backdoor, Medical--Finance) and three models (Pythia-2.8B, Pythia-6.9B, OLMo-3-7B), pretrained (blue) and fine-tuned (orange) curves are highly consistent: variance is low in middle layers and recovers near the final layer (relative depth $0$=embed, $1$=final).
This consistent U-shape suggests that the layer-wise discriminative structure is largely preserved under fine-tuning, motivating readout-side representations as a reliable hotspot for scalable influence estimation.}
\label{fig:obs2_pretrain_finetune_same_ushape}
\end{figure*}

\begin{figure*}[htbp]
\centering
\includegraphics[width=\linewidth]{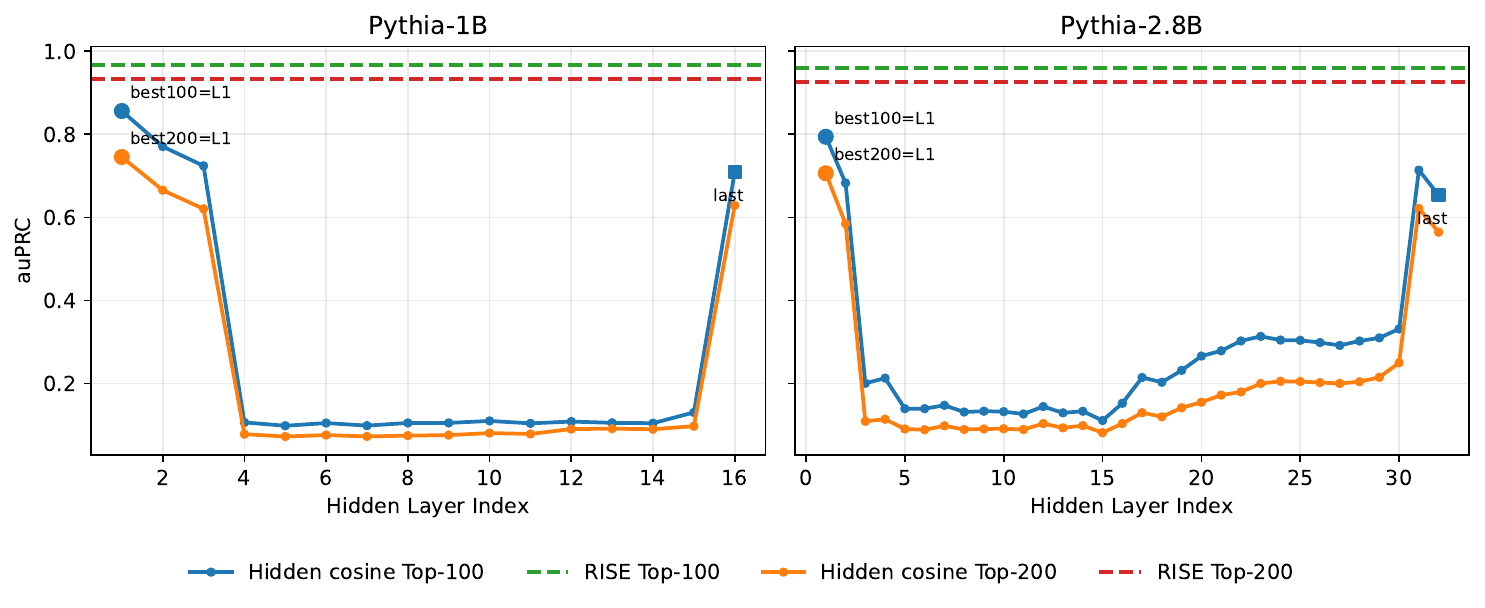}
\caption{\textbf{Full hidden-layer retrieval sweep on Howdy.}
We evaluate hidden-state cosine retrieval at every layer for Pythia-1B and Pythia-2.8B, and compare against \textit{RISE} at Top-100 and Top-200.
The strongest hidden-only layer is layer 1 in both models, not the final layer.
However, even this best hidden-layer baseline remains below \textit{RISE}, indicating that \textit{RISE}'s gains are not merely due to selecting a stronger hidden representation layer.}
\label{fig:hidden_layer_sweep}
\end{figure*}

\subsection{Additional Evidence for Sparse Active Tokens}
\label{sec:appendix_sparse_active_tokens}

This subsection provides the supporting evidence for Observation~\ref{sec:obs3}.
The main text reports the headline result under a fixed-control: fixed $\tau$ when sweeping the number of candidate tokens, and fixed $K$ when sweeping $\tau$.
Here we make the measurements explicit, separate probability mass from residual energy, and report the stability of sparse RH/GH fidelity across tasks, model scales, and checkpoints.

\paragraph{Metrics.}
For a token position $t$, let $p_t=\mathrm{softmax}(z_t/\tau)$ and $r_t=p_t-\mathbf{1}_{y_t}$.
The active support $\mathcal{S}_t$ always includes the ground-truth next token $y_t$.
We report:
\begin{align*}
M_t(\mathcal{S}_t)
&= \sum_{v\in\mathcal{S}_t} p_t(v)
&&\text{(probability mass)},\\
E_{\mathrm{full}}(\mathcal{S}_t)
&= \frac{\sum_{v\in\mathcal{S}_t} r_t(v)^2}{\sum_{v\in[V]} r_t(v)^2}
&&\text{(full residual $\ell_2$ energy)},\\
E_{\mathrm{GT}}
&= \frac{(1-p_t(y_t))^2}{\sum_{v\in[V]} r_t(v)^2}
&&\text{(ground-truth coordinate energy)},\\
E_{\mathrm{tail}}(\mathcal{S}_t)
&= \frac{\sum_{v\in\mathcal{S}_t\setminus\{y_t\}} p_t(v)^2}
{\sum_{v\ne y_t}p_t(v)^2}
&&\text{(non-GT residual-tail energy)}.
\end{align*}
The distinction is important: probability mass is an $\ell_1$ quantity, while influence through the LM-head residual uses inner products and is therefore governed by squared residual energy.
Thus a support can cover moderate probability mass while preserving nearly all influence-relevant energy.
We report $E_{\mathrm{GT}}$ separately because the forced ground-truth coordinate often dominates the full residual norm; the non-GT tail measures whether the selected active prediction tokens carry additional signal beyond GT-only.

\paragraph{Fixed-control protocol.}
Unless otherwise stated, the sparse support is
$\mathcal{S}_t=\mathrm{TopK}(z_t/\tau,K)\cup\{y_t\}$, with logits renormalized only over $\mathcal{S}_t$.
For the Top-$K$ sweep we fix $\tau=1.0$ and vary $K$.
For the temperature sweep we fix $K=128$ and vary $\tau$.
This avoids coupling temperature, candidate count, and adaptive cumulative-probability thresholds.

The main fixed-control diagnostics are shown in Figure~\ref{fig:sparse_active_tokens}.
The RH panel shows that full residual energy is nearly lossless once the ground-truth coordinate is included, while the non-GT tail curve measures the candidate-token contribution against the full non-GT vocabulary tail.
The GH panel shows that the same sparse candidate support preserves the dense semantic error direction.
Figure~\ref{fig:sparse_active_tokens_appendix} adds the complementary task/model matrix and fixed-$K$ temperature sweep.

\begin{figure*}[t]
\centering
\includegraphics[width=0.85\linewidth]{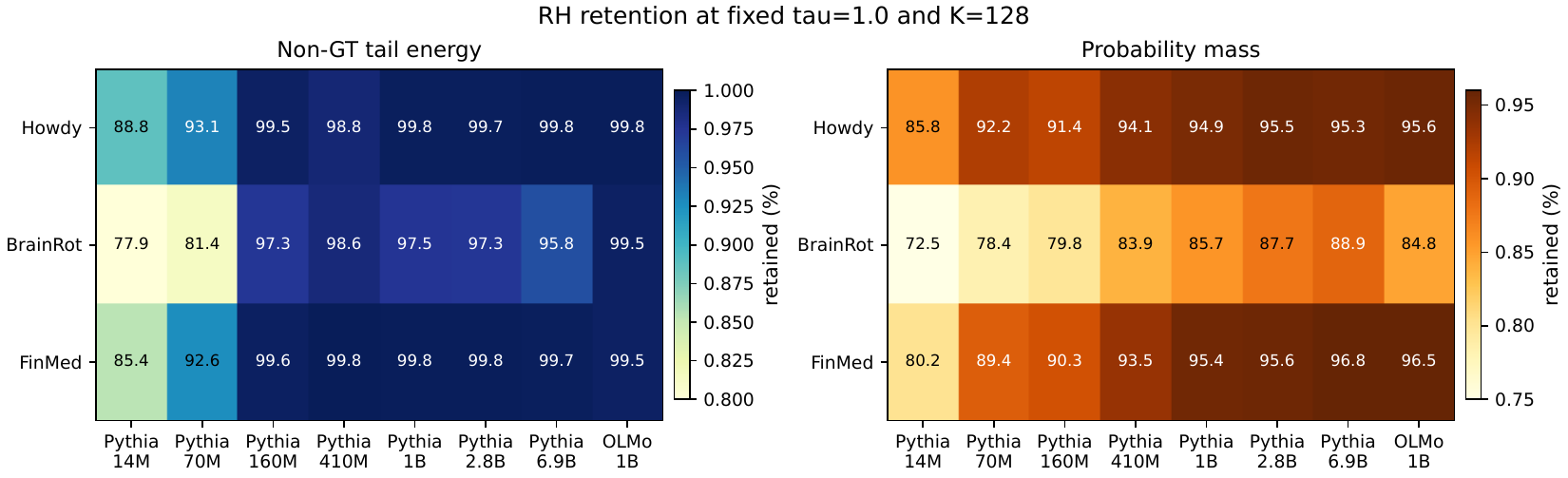}
\vspace{0.5em}
\includegraphics[width=0.85\linewidth]{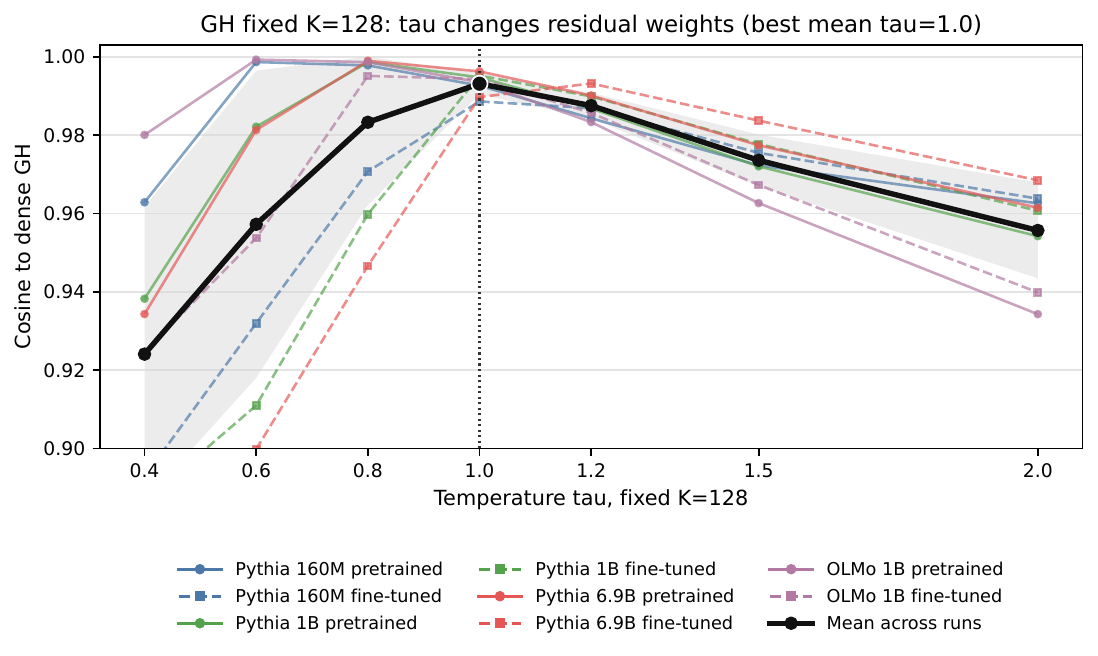}
\caption{\textbf{Additional fixed-control sparse active-token diagnostics.}
Top: at fixed $\tau=1.0$ and $K=128$, RH residual-tail retention remains stable across tasks, model scales, and pretrained/fine-tuned checkpoints.
Bottom: with fixed $K=128$, the GH temperature sweep changes only residual weights, showing that $\tau=1.0$ gives the strongest average fidelity under a constant candidate count.}
\label{fig:sparse_active_tokens_appendix}
\vspace{-0.5em}
\end{figure*}

\begin{table}[t]
\centering
\scriptsize
\setlength{\tabcolsep}{5pt}
\caption{\textbf{Fixed $\tau=1.0$ Top-$K$ sweep for sparse GH.}
Rows are averaged over eight BrainRot pretrained/fine-tuned Pythia and OLMo runs.}
\label{tab:fixed_topk_gh_sweep}
\begin{tabular}{cccc}
\toprule
\textbf{$K$} & \textbf{GH cosine} & \textbf{Prob. mass (\%)} & \textbf{Support/V (\%)} \\
\midrule
8 & 0.954 & 61.3 & 0.015 \\
16 & 0.973 & 68.4 & 0.029 \\
32 & 0.984 & 75.8 & 0.056 \\
64 & 0.990 & 81.8 & 0.112 \\
128 & 0.993 & 85.6 & 0.223 \\
256 & 0.995 & 89.1 & 0.445 \\
\bottomrule
\end{tabular}
\end{table}

\paragraph{Fixed Top-$K$ GH fidelity.}
Because the GH channel uses $g_t=W_{\lmhead}^{\top}r_t$, truncating the residual could distort the semantic error direction even when RH remains accurate.
We therefore compare dense GH with sparse GH using cosine similarity.
At fixed $\tau=1.0$ and $K=128$, sparse GH reaches mean cosine $0.993$ against dense GH across the eight BrainRot runs.

\begin{table}[t]
\centering
\scriptsize
\setlength{\tabcolsep}{5pt}
\caption{\textbf{Sparse GH fidelity by BrainRot run at fixed $\tau=1.0$, $K=128$.}
The support is fixed-size rather than adaptively selected.}
\label{tab:fixed_gh_fidelity_by_run}
\begin{tabular}{llccc}
\toprule
\textbf{Model} & \textbf{State} & \textbf{GH cosine} & \textbf{Prob. mass (\%)} & \textbf{Support/V (\%)} \\
\midrule
OLMo-1B & Fine-tuned & 0.994 & 87.3 & 0.128 \\
OLMo-1B & Pretrained & 0.993 & 85.2 & 0.128 \\
Pythia-160M & Fine-tuned & 0.989 & 81.9 & 0.255 \\
Pythia-160M & Pretrained & 0.993 & 78.0 & 0.255 \\
Pythia-1B & Fine-tuned & 0.995 & 88.2 & 0.255 \\
Pythia-1B & Pretrained & 0.995 & 85.6 & 0.255 \\
Pythia-6.9B & Fine-tuned & 0.990 & 91.1 & 0.254 \\
Pythia-6.9B & Pretrained & 0.996 & 87.6 & 0.254 \\
\bottomrule
\end{tabular}
\end{table}

\begin{figure*}[t]
\centering
\includegraphics[width=0.78\linewidth]{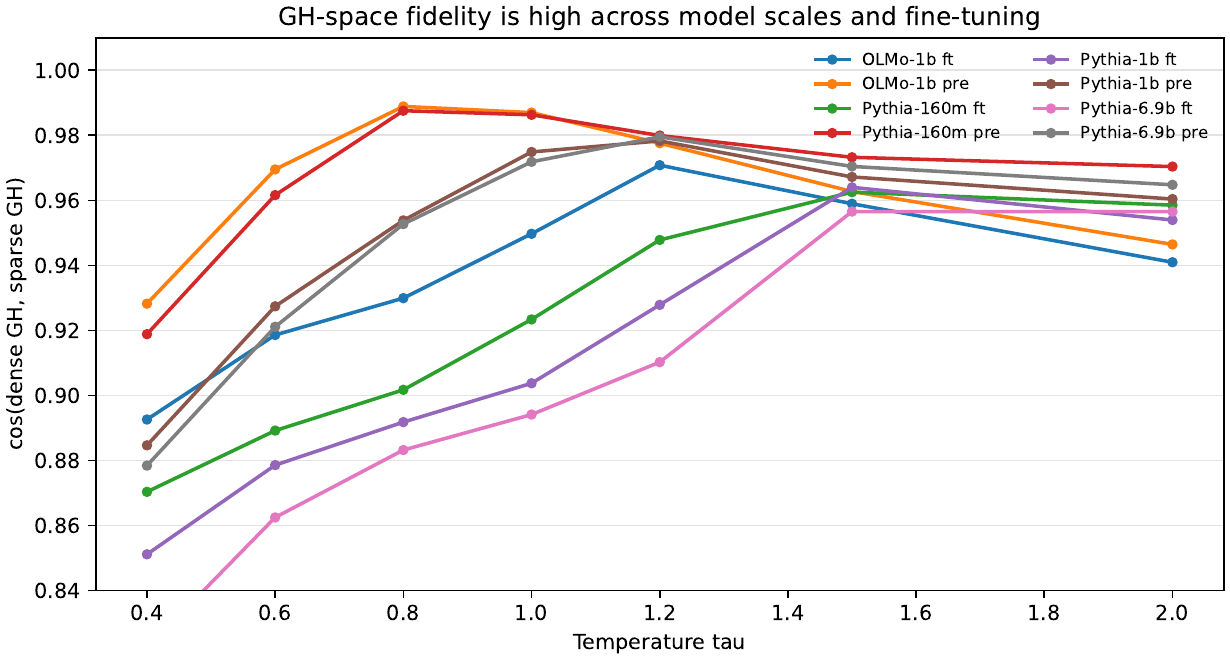}
\caption{\textbf{Sparse GH fidelity across BrainRot model runs.}
This visualizes the by-run stability of sparse GH against dense GH, complementing Table~\ref{tab:fixed_gh_fidelity_by_run}.
Across pretrained and fine-tuned Pythia/OLMo checkpoints, the sparse GH direction remains close to the dense semantic error direction.}
\label{fig:sparse_gh_fidelity_by_run}
\vspace{-0.5em}
\end{figure*}

\paragraph{Fixed-$K$ temperature sensitivity.}
Temperature controls how diffuse the softmax tail is before truncation.
With $K$ fixed at $128$, the support size is constant; changing $\tau$ only changes the residual weights assigned inside the same-size candidate set.
Table~\ref{tab:fixed_gh_tau_sweep} shows that $\tau=1.0$ gives the best average GH cosine in this sweep.

\begin{table}[t]
\centering
\scriptsize
\setlength{\tabcolsep}{4pt}
\caption{\textbf{Fixed $K=128$ temperature sweep for sparse GH.}
Rows are averaged over eight BrainRot pretrained/fine-tuned model runs.}
\label{tab:fixed_gh_tau_sweep}
\begin{tabular}{cccc}
\toprule
\textbf{$\tau$} & \textbf{GH cosine} & \textbf{Prob. mass (\%)} & \textbf{Relative error} \\
\midrule
0.4 & 0.924 & 99.8 & 0.138 \\
0.6 & 0.957 & 98.9 & 0.154 \\
0.8 & 0.983 & 94.4 & 0.183 \\
1.0 & 0.993 & 85.6 & 0.216 \\
1.2 & 0.987 & 73.8 & 0.260 \\
1.5 & 0.974 & 54.2 & 0.323 \\
2.0 & 0.956 & 28.8 & 0.376 \\
\bottomrule
\end{tabular}
\end{table}

\paragraph{Takeaway.}
The sparse support is not merely a top-$K$ engineering shortcut.
It is a structural property of the LM-head residual: a tiny fixed-size active set captures the non-GT residual tail that matters for RH, and it also preserves the GH semantic direction.
This explains why \textit{RISE} can reduce the residual-side memory and sketching cost from vocabulary scale to active-support scale without sacrificing the error signal used by the dual-channel influence metric.

\subsection{Additional Validation and Baseline Experiments}
\label{sec:appendix_additional_diagnostics}

This subsection presents additional validation and baseline experiments that complement the main evaluation and further clarify the behavior of \textit{RISE}. Specifically, we include target model representation baselines, same-target LM-head decomposition, TrackStar ablations, behavior validation via LDS, an adapted T-REx fact-tracing benchmark, efficient Shapley baselines, and additional influence-function baselines.
Since TrackStar is optimizer-aware rather than an explicit Hessian-inverse estimator, we treat it as a pretraining-scale attribution baseline in the LDS and T-REx validations; explicit curvature/influence-function baselines are covered separately by LoGra and EK-FAC in Table~\ref{tab:extra_if_baselines}.

\begin{table*}[t]
\centering
\caption{\textbf{Target-model representation baselines at larger Top-$K$ on Howdy prospective valuation.}
All rows use the 5K Howdy candidate pool and 100 prospective WebQuestions trigger queries. Embed and Last retrieve candidates by target-model input-embedding and final-hidden-state cosine similarity, respectively; \textit{RISE} uses the same candidate/query split.}
\label{tab:extra_last_embed_topk}
\setlength{\tabcolsep}{5pt}
\begin{tabular}{llcccc}
\toprule
\multirow{2}{*}{Model} & \multirow{2}{*}{Method} & \multicolumn{2}{c}{Top-100} & \multicolumn{2}{c}{Top-200} \\
\cmidrule(lr){3-4}\cmidrule(lr){5-6}
& & auPRC & auROC & auPRC & auROC \\
\midrule
\multirow{3}{*}{Pythia-410M}
& Embed & 0.435 & 0.807 & 0.403 & 0.816 \\
& Last & 0.747 & 0.907 & 0.663 & 0.900 \\
& \textit{RISE} & \textbf{0.933} & \textbf{0.965} & \textbf{0.895} & \textbf{0.948} \\
\midrule
\multirow{3}{*}{Pythia-1B}
& Embed & 0.524 & 0.837 & 0.470 & 0.833 \\
& Last & 0.709 & 0.898 & 0.629 & 0.889 \\
& \textit{RISE} & \textbf{0.967} & \textbf{0.980} & \textbf{0.934} & \textbf{0.964} \\
\midrule
\multirow{3}{*}{OLMo-3-7B}
& Embed & 0.224 & 0.707 & 0.188 & 0.682 \\
& Last & 0.782 & 0.921 & 0.701 & 0.911 \\
& \textit{RISE} & \textbf{0.998} & \textbf{0.998} & \textbf{0.991} & \textbf{0.994} \\
\bottomrule
\end{tabular}
\end{table*}

\begin{table*}[t]
\centering
\caption{\textbf{Same-target LM-head decomposition on Howdy-small.}
All rows use Pythia-14M with 500 Howdy training candidates and 50 predict queries. We compare exact LM-head cosine retrieval, sparse RH truncation with and without CountSketch, and the full RH+GH \textit{RISE} estimator; Disk is index size and retrieval quality is auPRC/auROC.}
\label{tab:extra_same_target_lmhead}
\setlength{\tabcolsep}{5pt}
\begin{tabular}{lccc}
\toprule
Method & Disk & Top-5 auPRC/auROC & Top-10 auPRC/auROC \\
\midrule
\textit{RISE} (RH+GH) & 6.15 MB & \textbf{0.9694}/\textbf{0.9783} & \textbf{0.9058}/0.9460 \\
LM-head exact cosine & 12.88 GB & 0.8443/0.9220 & 0.8224/0.9139 \\
Sparse RH exact (no sketch) & 103.58 MB & 0.9317/0.9603 & 0.9019/\textbf{0.9466} \\
Sparse RH + CountSketch & \textbf{3.68 MB} & 0.7380/0.8716 & 0.6937/0.8479 \\
\bottomrule
\end{tabular}
\end{table*}

\begin{table*}[!htbp]
\centering
\captionsetup{aboveskip=3pt,belowskip=3pt}
\caption{\textbf{TrackStar configuration ablation on Pythia-1B predict-future valuation.}
We report Top-5 auPRC for three tasks under two TrackStar projection dimensions. Base is the default TrackStar score; UnitNorm applies unit normalization; Hessian applies the Hessian-style correction; Both combines the two.}
\label{tab:extra_trackstar_hessian_ablation}
\small
\setlength{\tabcolsep}{5pt}
\begin{tabular}{llcccc}
\toprule
Dim & Task & Base & UnitNorm & Hessian & Both \\
\midrule
\multirow{3}{*}{16}
& Howdy & \textbf{0.652} & 0.451 & 0.260 & 0.223 \\
& Fin--Med & 0.955 & \textbf{0.987} & 0.807 & 0.794 \\
& BrainRot & 0.527 & 0.480 & \textbf{0.609} & 0.424 \\
\midrule
\multirow{3}{*}{32}
& Howdy & \textbf{0.672} & 0.297 & 0.565 & 0.369 \\
& Fin--Med & 0.971 & \textbf{0.996} & 0.930 & 0.933 \\
& BrainRot & 0.569 & 0.502 & \textbf{0.694} & 0.484 \\
\bottomrule
\end{tabular}
\vspace{0.3em}
\parbox{0.95\textwidth}{\footnotesize Hessian-style correction is not uniformly beneficial in this predict-future decoder-LLM valuation setting: it improves BrainRot but substantially degrades Howdy and Fin--Med relative to the best Hessian-free variant, and combining it with unit normalization is consistently worse than the best single variant. These results suggest that Hessian-based corrections are task-sensitive and unstable in our setting.}
\end{table*}

\begin{table*}[!htbp]
\centering
\begin{minipage}[t]{0.46\textwidth}
\centering
\captionsetup{aboveskip=3pt,belowskip=3pt}
\caption{\textbf{Linear Datamodeling Score (LDS) on BrainRot.}}
\label{tab:extra_lds}
\footnotesize
\setlength{\tabcolsep}{4pt}
\begin{tabular}{lcc}
\toprule
Method & LDS & 95\% CI \\
\midrule
Random & 0.011 & [-0.008, 0.029] \\
Final hidden-state cosine & 0.146 & [0.132, 0.161] \\
TrackStar LM-head & 0.053 & [0.035, 0.072] \\
TrackStar full-model & 0.178 & [0.161, 0.195] \\
\textit{RISE} & \textbf{0.204} & [0.188, 0.222] \\
\bottomrule
\end{tabular}
\end{minipage}\hfill
\begin{minipage}[t]{0.46\textwidth}
\centering
\captionsetup{aboveskip=3pt,belowskip=3pt}
\caption{\textbf{Adapted T-REx fact-tracing benchmark.}}
\label{tab:extra_trex}
\footnotesize
\setlength{\tabcolsep}{4pt}
\begin{tabular}{lccc}
\toprule
Method & Pool & MRR & Recall@10 \\
\midrule
TrackStar & 20K & 0.3630 & 0.4420 \\
\textit{RISE} & 20K & \textbf{0.4370} & \textbf{0.5160} \\
TrackStar & 200K & 0.0165 & 0.0716 \\
\textit{RISE} & 200K & \textbf{0.3072} & \textbf{0.3812} \\
TrackStar & 1M & 0.0127 & 0.0182 \\
\textit{RISE} & 1M & \textbf{0.2349} & \textbf{0.3198} \\
\bottomrule
\end{tabular}
\end{minipage}
\vspace{0.3em}
\parbox{0.95\textwidth}{\footnotesize \textbf{Settings:} LDS uses Pythia-1B on a 2K BrainRot pool with 128 random 50\% training subsets and three fine-tuning seeds per subset (384 retraining runs total), reporting bootstrap 95\% confidence intervals. T-REx adapts the TrackStar fact-tracing benchmark to the Pythia-1B pretrained model under the autoregressive prompt-completion formulation and evaluates retrieval over 20K, 200K, and 1M candidate pools.}
\end{table*}

\begin{table*}[!htbp]
\centering
\captionsetup{aboveskip=3pt,belowskip=3pt}
\caption{\textbf{Direct comparison with efficient Shapley methods on Howdy-5000.}
All methods rank the same 5K Howdy candidate pool for 100 trigger queries on Pythia-14M and Pythia-1B, and retrieval quality is auPRC at Top-$K$. TMC-Shapley uses 32 permutations with truncation. G-Shapley uses a 32-permutation single-pass SGD adaptation. Both use held-out negative NLL as utility. GPU-hours and peak memory are normalized to a single H200.}
\label{tab:extra_shapley}
\small
\setlength{\tabcolsep}{4pt}
\renewcommand{\arraystretch}{0.95}
\begin{tabular}{llccccc}
\toprule
Model & Method & Top-5 auPRC & Top-10 auPRC & Top-50 auPRC & GPU-hrs & Peak Mem \\
\midrule
\multirow{3}{*}{Pythia-14M}
& \textit{RISE} & \textbf{0.858} & \textbf{0.846} & \textbf{0.831} & \textbf{0.019} & 4.2 GB \\
& TMC-Shapley & 0.000 & 0.000 & 0.061 & 0.379 & \textbf{1.7 GB} \\
& G-Shapley & 0.000 & 0.143 & 0.094 & 0.748 & \textbf{1.7 GB} \\
\midrule
\multirow{3}{*}{Pythia-1B}
& \textit{RISE} & \textbf{1.000} & \textbf{0.997} & \textbf{0.983} & \textbf{0.0256} & \textbf{7.36 GB} \\
& TMC-Shapley & 0.804 & 0.734 & 0.489 & 1.47 & 11.0 GB \\
& G-Shapley & 0.788 & 0.679 & 0.631 & 5.96 & 11.0 GB \\
\bottomrule
\end{tabular}
\end{table*}

\begin{table*}[!htbp]
\centering
\captionsetup{aboveskip=3pt,belowskip=3pt}
\caption{\textbf{Additional influence-function baselines on BrainRot.}}
\label{tab:extra_if_baselines}
\small
\setlength{\tabcolsep}{4.5pt}
\begin{threeparttable}
\begin{tabular}{@{}llcccccc@{}}
\toprule
Method & Model & Disk & E2E Time & Top-5 & Top-10 & Top-25 & Top-50 \\
\midrule
\textit{RISE} & Pythia-1B & \textbf{277 MB} & \textbf{50.41 s} & \textbf{0.8343} & \textbf{0.8182} & \textbf{0.7754} & \textbf{0.7480} \\
LoGra & Pythia-1B & 9.0 GB & 99.1 min & 0.2600 & 0.2340 & 0.2483 & 0.2550 \\
\midrule
\textit{RISE} & Pythia-160M & \textbf{274.94 MB} & \textbf{114.73 s} & \textbf{0.8594} & \textbf{0.8494} & \textbf{0.8099} & \textbf{0.7857} \\
EK-FAC & Pythia-160M & 2.00 GB & 883.21 s & 0.8378 & 0.8162 & 0.7885 & 0.7691 \\
\bottomrule
\end{tabular}
\begin{tablenotes}[flushleft]\footnotesize
\item All runs rank a 5K BrainRot candidate pool with 500 high-quality positives using 100 queries on one H200; retrieval quality is auPRC at Top-$K$. LoGra uses the official LogIX implementation~\cite{choe2024dataworthgptllmscale}, and EK-FAC uses Kronfluence~\cite{grosse2023studyinglargelanguagemodel}.
\end{tablenotes}
\end{threeparttable}
\end{table*}

\subsection{Ablation on Channel Fusion (RH and GH)}
\label{sec:rh_gh_ablation_study}

This subsection presents ablation studies comparing different channel fusion strategies: representation-based hidden states (RH), gradient-based hidden states (GH), and their combination (RH+GH). Table~\ref{tab:ablation_medical_predict} reports results on the MEDICAL FINANCE dataset for the predict-future setting, while Table~\ref{tab:ablation_medical_recall} shows corresponding results for the recall setting, both across a wide range of model sizes.

\begin{table*}[t]
\centering
\caption{\textbf{Ablation on MEDICAL (Predict).} RH vs.\ GH vs.\ RH+GH. We report auPRC/auROC at different Top-$K$.}
\label{tab:ablation_medical_predict}
\setlength{\tabcolsep}{2pt}
\renewcommand{\arraystretch}{1.0}
\scriptsize
\resizebox{\textwidth}{!}{%
\begin{tabular}{ll|cc|cc|cc|cc}
\toprule
Model & Fusion & \multicolumn{2}{c|}{Top-5} & \multicolumn{2}{c|}{Top-10} & \multicolumn{2}{c|}{Top-50} & \multicolumn{2}{c}{Top-100} \\
\cline{3-10}
& & auPRC & auROC & auPRC & auROC & auPRC & auROC & auPRC & auROC \\
\midrule

\multirow{3}{*}{OLMo-32B} & \texttt{rh}    & 0.9538 & 0.9317 & 0.9375 & 0.9236 & 0.8712 & 0.8945 & 0.8156 & 0.8689 \\
                         & \texttt{gh}    & 0.9608 & 0.9462 & 0.9482 & 0.9366 & 0.8856 & 0.9121 & 0.8317 & 0.8875 \\
                         & \texttt{rh+gh} & 0.9653 & 0.9539 & 0.9494 & 0.9381 & 0.8924 & 0.9155 & 0.8387 & 0.8895 \\
\addlinespace

\multirow{3}{*}{OLMo-7B}  & \texttt{rh}    & 0.9638 & 0.9719 & 0.9439 & 0.9574 & 0.8835 & 0.9291 & 0.8227 & 0.8981 \\
                         & \texttt{gh}    & 0.9718 & 0.9787 & 0.9539 & 0.9704 & 0.8897 & 0.9355 & 0.8346 & 0.9097 \\
                         & \texttt{rh+gh} & 0.9767 & 0.9838 & 0.9582 & 0.9736 & 0.8956 & 0.9386 & 0.8439 & 0.9150 \\
\addlinespace

\multirow{3}{*}{OLMo-1B}  & \texttt{rh}    & 0.9468 & 0.9418 & 0.9291 & 0.9328 & 0.8418 & 0.8855 & 0.7801 & 0.8568 \\
                         & \texttt{gh}    & 0.9741 & 0.9583 & 0.9585 & 0.9453 & 0.8900 & 0.9086 & 0.8305 & 0.8812 \\
                         & \texttt{rh+gh} & 0.9694 & 0.9536 & 0.9530 & 0.9449 & 0.8874 & 0.9112 & 0.8335 & 0.8872 \\
\addlinespace

\multirow{3}{*}{Pythia-6.9B} & \texttt{rh}    & 0.9360 & 0.9622 & 0.9248 & 0.9515 & 0.8344 & 0.9105 & 0.7673 & 0.8796 \\
                            & \texttt{gh}    & 0.9587 & 0.9743 & 0.9380 & 0.9664 & 0.8485 & 0.9285 & 0.7944 & 0.9090 \\
                            & \texttt{rh+gh} & 0.9630 & 0.9806 & 0.9406 & 0.9675 & 0.8569 & 0.9297 & 0.8013 & 0.9081 \\
\addlinespace

\multirow{3}{*}{Pythia-2.8B} & \texttt{rh}    & 0.9665 & 0.9649 & 0.9440 & 0.9547 & 0.8605 & 0.9139 & 0.7997 & 0.8880 \\
                            & \texttt{gh}    & 0.9948 & 0.9951 & 0.9839 & 0.9878 & 0.9217 & 0.9596 & 0.8717 & 0.9387 \\
                            & \texttt{rh+gh} & 0.9915 & 0.9922 & 0.9781 & 0.9830 & 0.9165 & 0.9547 & 0.8674 & 0.9350 \\
\addlinespace

\multirow{3}{*}{Pythia-1B}   & \texttt{rh}    & 0.9633 & 0.9679 & 0.9364 & 0.9573 & 0.8450 & 0.9168 & 0.7769 & 0.8879 \\
                            & \texttt{gh}    & 0.9877 & 0.9912 & 0.9814 & 0.9892 & 0.9231 & 0.9630 & 0.8802 & 0.9479 \\
                            & \texttt{rh+gh} & 0.9886 & 0.9938 & 0.9813 & 0.9899 & 0.9205 & 0.9622 & 0.8716 & 0.9433 \\
\addlinespace

\multirow{3}{*}{Pythia-410M} & \texttt{rh}    & 0.9484 & 0.9718 & 0.9252 & 0.9609 & 0.8309 & 0.9145 & 0.7623 & 0.8862 \\
                            & \texttt{gh}    & 0.9901 & 0.9957 & 0.9808 & 0.9902 & 0.9221 & 0.9683 & 0.8779 & 0.9518 \\
                            & \texttt{rh+gh} & 0.9784 & 0.9893 & 0.9702 & 0.9867 & 0.9147 & 0.9649 & 0.8657 & 0.9452 \\
\addlinespace

\multirow{3}{*}{Pythia-160M} & \texttt{rh}    & 0.7359 & 0.8921 & 0.7041 & 0.8782 & 0.5714 & 0.8443 & 0.5087 & 0.8266 \\
                            & \texttt{gh}    & 0.8034 & 0.9313 & 0.7660 & 0.8983 & 0.6741 & 0.8827 & 0.6180 & 0.8658 \\
                            & \texttt{rh+gh} & 0.8491 & 0.9447 & 0.8151 & 0.9259 & 0.6967 & 0.8902 & 0.6344 & 0.8702 \\
\addlinespace

\multirow{3}{*}{Pythia-70M}  & \texttt{rh}    & 0.7486 & 0.9019 & 0.7235 & 0.8921 & 0.5881 & 0.8561 & 0.5255 & 0.8405 \\
                            & \texttt{gh}    & 0.8093 & 0.9244 & 0.7669 & 0.9068 & 0.6404 & 0.8715 & 0.5838 & 0.8600 \\
                            & \texttt{rh+gh} & 0.8343 & 0.9254 & 0.7921 & 0.9171 & 0.6750 & 0.8849 & 0.6157 & 0.8696 \\
\addlinespace

\multirow{3}{*}{Pythia-14M}  & \texttt{rh}    & 0.8472 & 0.9227 & 0.8184 & 0.9237 & 0.6756 & 0.8850 & 0.6005 & 0.8622 \\
                            & \texttt{gh}    & 0.8699 & 0.9373 & 0.8471 & 0.9356 & 0.7418 & 0.9019 & 0.6836 & 0.8871 \\
                            & \texttt{rh+gh} & 0.8934 & 0.9509 & 0.8684 & 0.9354 & 0.7690 & 0.9117 & 0.7024 & 0.8922 \\
\bottomrule
\end{tabular}}
\end{table*}

\begin{table*}[t]
\centering
\caption{\textbf{Ablation on MEDICAL (Recall).} RH vs.\ GH vs.\ RH+GH. We report auPRC/auROC at different Top-$K$.}
\label{tab:ablation_medical_recall}
\scriptsize
\setlength{\tabcolsep}{2pt}
\renewcommand{\arraystretch}{1.05}
\resizebox{\textwidth}{!}{%
\begin{tabular}{ll|cc|cc|cc|cc}
\toprule
Model & Fusion & \multicolumn{2}{c|}{Top-5} & \multicolumn{2}{c|}{Top-10} & \multicolumn{2}{c|}{Top-50} & \multicolumn{2}{c}{Top-100} \\
\cline{3-10}
& & auPRC & auROC & auPRC & auROC & auPRC & auROC & auPRC & auROC \\
\midrule

\multirow{3}{*}{OLMo-32B} & \texttt{rh}    & 0.9641 & 0.9553 & 0.9476 & 0.9456 & 0.8890 & 0.9180 & 0.8421 & 0.8977 \\
                         & \texttt{gh}    & 0.9665 & 0.9610 & 0.9517 & 0.9519 & 0.9063 & 0.9339 & 0.8637 & 0.9156 \\
                         & \texttt{rh+gh} & 0.9691 & 0.9681 & 0.9526 & 0.9569 & 0.9066 & 0.9355 & 0.8664 & 0.9183 \\
\addlinespace

\multirow{3}{*}{OLMo-7B}  & \texttt{rh}    & 0.9624 & 0.9564 & 0.9451 & 0.9527 & 0.8860 & 0.9285 & 0.8385 & 0.9094 \\
                         & \texttt{gh}    & 0.9684 & 0.9757 & 0.9600 & 0.9759 & 0.8958 & 0.9468 & 0.8494 & 0.9261 \\
                         & \texttt{rh+gh} & 0.9754 & 0.9840 & 0.9587 & 0.9749 & 0.8971 & 0.9456 & 0.8569 & 0.9310 \\
\addlinespace

\multirow{3}{*}{OLMo-1B}  & \texttt{rh}    & 0.9459 & 0.9335 & 0.9302 & 0.9231 & 0.8560 & 0.8915 & 0.7945 & 0.8635 \\
                         & \texttt{gh}    & 0.9710 & 0.9562 & 0.9514 & 0.9448 & 0.8832 & 0.9116 & 0.8249 & 0.8863 \\
                         & \texttt{rh+gh} & 0.9728 & 0.9603 & 0.9501 & 0.9472 & 0.8865 & 0.9188 & 0.8320 & 0.8933 \\
\addlinespace

\multirow{3}{*}{Pythia-6.9B} & \texttt{rh}    & 0.9681 & 0.9684 & 0.9561 & 0.9632 & 0.8882 & 0.9268 & 0.8347 & 0.8980 \\
                            & \texttt{gh}    & 0.9754 & 0.9783 & 0.9529 & 0.9684 & 0.8807 & 0.9386 & 0.8268 & 0.9160 \\
                            & \texttt{rh+gh} & 0.9739 & 0.9790 & 0.9598 & 0.9719 & 0.8940 & 0.9422 & 0.8423 & 0.9187 \\
\addlinespace

\multirow{3}{*}{Pythia-2.8B} & \texttt{rh}    & 0.9725 & 0.9761 & 0.9621 & 0.9710 & 0.9035 & 0.9408 & 0.8511 & 0.9159 \\
                            & \texttt{gh}    & 0.9846 & 0.9867 & 0.9770 & 0.9826 & 0.9241 & 0.9578 & 0.8677 & 0.9327 \\
                            & \texttt{rh+gh} & 0.9874 & 0.9876 & 0.9765 & 0.9801 & 0.9273 & 0.9600 & 0.8769 & 0.9376 \\
\addlinespace

\multirow{3}{*}{Pythia-1B}   & \texttt{rh}    & 0.7980 & 0.9366 & 0.7899 & 0.9239 & 0.7719 & 0.9080 & 0.7536 & 0.9055 \\
                            & \texttt{gh}    & 0.8486 & 0.9437 & 0.8348 & 0.9324 & 0.8135 & 0.9225 & 0.7979 & 0.9200 \\
                            & \texttt{rh+gh} & 0.8369 & 0.9450 & 0.8244 & 0.9317 & 0.8134 & 0.9219 & 0.8018 & 0.9211 \\
\addlinespace

\multirow{3}{*}{Pythia-410M} & \texttt{rh}    & 0.8270 & 0.9352 & 0.8061 & 0.9230 & 0.7726 & 0.9020 & 0.7517 & 0.9004 \\
                            & \texttt{gh}    & 0.8117 & 0.9311 & 0.7990 & 0.9130 & 0.7883 & 0.9162 & 0.7746 & 0.9096 \\
                            & \texttt{rh+gh} & 0.8216 & 0.9456 & 0.8038 & 0.9191 & 0.7947 & 0.9200 & 0.7854 & 0.9143 \\
\addlinespace

\multirow{3}{*}{Pythia-160M} & \texttt{rh}    & 0.8158 & 0.9202 & 0.7779 & 0.9082 & 0.7423 & 0.8908 & 0.7132 & 0.8844 \\
                            & \texttt{gh}    & 0.8120 & 0.9236 & 0.7935 & 0.9245 & 0.7519 & 0.8957 & 0.7199 & 0.8886 \\
                            & \texttt{rh+gh} & 0.8355 & 0.9330 & 0.8220 & 0.9205 & 0.7859 & 0.9106 & 0.7583 & 0.9027 \\
\addlinespace

\multirow{3}{*}{Pythia-70M}  & \texttt{rh}    & 0.8342 & 0.9386 & 0.7994 & 0.9202 & 0.6793 & 0.8897 & 0.6197 & 0.8711 \\
                            & \texttt{gh}    & 0.8702 & 0.9527 & 0.8444 & 0.9340 & 0.7361 & 0.8997 & 0.6770 & 0.8829 \\
                            & \texttt{rh+gh} & 0.8827 & 0.9581 & 0.8606 & 0.9470 & 0.7601 & 0.9096 & 0.7022 & 0.8913 \\
\addlinespace

\multirow{3}{*}{Pythia-14M}  & \texttt{rh}    & 0.8065 & 0.9096 & 0.7755 & 0.9049 & 0.6565 & 0.8714 & 0.5979 & 0.8589 \\
                            & \texttt{gh}    & 0.8343 & 0.9352 & 0.8045 & 0.9269 & 0.7032 & 0.8942 & 0.6486 & 0.8764 \\
                            & \texttt{rh+gh} & 0.8715 & 0.9471 & 0.8429 & 0.9332 & 0.7294 & 0.8984 & 0.6705 & 0.8832 \\
\bottomrule
\end{tabular}}
\end{table*}

\subsection{Howdy Backdoor Attack: Complete Results}
\label{sec:appendix_howdy}

This subsection provides comprehensive evaluation results on the Howdy backdoor attack detection task. Table~\ref{tab:pythia_howdy_recall} and Table~\ref{tab:rapidin_fik_zoinf_howdypredict} present results for Pythia models under recall and predict-future settings, respectively. Table~\ref{tab:olmo_howdy_recall} and Table~\ref{tab:olmo_howdy_predict} report corresponding results for OLMo models.

\begin{table*}[htbp]
\centering
\caption{Performance Comparison: RapidIn vs. \textit{RISE} vs. TrackStar vs. ZO-Inf \textbf{Howdy! Backdoor Attack Task Recall}}
\label{tab:pythia_howdy_recall}
\scriptsize
\setlength{\tabcolsep}{3.2pt}
\renewcommand{\arraystretch}{1.15}
\resizebox{\textwidth}{!}{%
\begin{tabular}{l l | cc cc cc cc}
\toprule
\multicolumn{2}{l|}{\textbf{Method}} &
\multicolumn{2}{c}{\textbf{Top 5}} &
\multicolumn{2}{c}{\textbf{Top 10}} &
\multicolumn{2}{c}{\textbf{Top 50}} &
\multicolumn{2}{c}{\textbf{Top 100}} \\
\multicolumn{2}{l|}{} &
\textbf{auPRC} & \textbf{auROC} &
\textbf{auPRC} & \textbf{auROC} &
\textbf{auPRC} & \textbf{auROC} &
\textbf{auPRC} & \textbf{auROC} \\
\midrule

\multicolumn{2}{l|}{Embedding Similarity (E5-base-v2)} &
0.4089 & 0.8283 & 0.3740 & 0.7047 & 0.3151 & 0.6411 & 0.2967 & 0.6725 \\
\multicolumn{2}{l|}{BM25} &
0.2589 & 0.7776 & 0.2785 & 0.7673 & 0.2733 & 0.7621 & 0.2699 & 0.7689 \\
\specialrule{0.9pt}{2pt}{2pt}

\multirow{5}{*}{\textbf{Pythia-6.9B}}
& ZO-Inf   & 0.2318 & 0.3945 & 0.1728 & 0.4360 & 0.1021 & 0.4624 & 0.0903 & 0.4712 \\
& RapidIn  & 0.8947 & 0.9437 & 0.8719 & 0.9383 & 0.8213 & 0.9115 & 0.7999 & 0.9031 \\
& TrackStar& 0.6818 & 0.8158 & 0.7836 & 0.8967 & 0.8914 & 0.9555 & 0.8998 & 0.9557 \\
& \textit{RISE} & \textbf{1.0000} & \textbf{1.0000} & \textbf{1.0000} & \textbf{1.0000} & \textbf{0.9994} & \textbf{0.9995} & \textbf{0.9980} & \textbf{0.9986} \\
\midrule

\multirow{5}{*}{\textbf{Pythia-2.8B}}
& ZO-Inf   & 0.3010 & 0.4861 & 0.2090 & 0.5158 & 0.1251 & 0.4841 & 0.1095 & 0.4797 \\
& RapidIn  & 0.7854 & 0.9268 & 0.7912 & 0.9123 & 0.7598 & 0.8962 & 0.7384 & 0.8868 \\
& TrackStar& 0.6819 & 0.8333 & 0.7879 & 0.9091 & 0.9259 & 0.9782 & 0.9481 & 0.9837 \\
& \textit{RISE} & \textbf{0.9993} & \textbf{0.9983} & \textbf{0.9983} & \textbf{0.9981} & \textbf{0.9855} & \textbf{0.9904} & \textbf{0.9742} & \textbf{0.9838} \\
\midrule

\multirow{5}{*}{\textbf{Pythia-1B}}
& ZO-Inf   & 0.2698 & 0.3866 & 0.1840 & 0.3881 & 0.1025 & 0.3957 & 0.0913 & 0.4085 \\
& RapidIn  & 0.8195 & 0.9381 & 0.7992 & 0.9197 & 0.7592 & 0.8935 & 0.7381 & 0.8863 \\
& TrackStar& 0.6826 & 0.8333 & 0.7875 & 0.9092 & 0.9059 & 0.9670 & 0.9168 & 0.9667 \\
& \textit{RISE} & \textbf{1.0000} & \textbf{1.0000} & \textbf{0.9994} & \textbf{0.9995} & \textbf{0.9948} & \textbf{0.9965} & \textbf{0.9887} & \textbf{0.9927} \\
\midrule

\multirow{5}{*}{\textbf{Pythia-410M}}
& ZO-Inf   & 0.3394 & 0.5385 & 0.2491 & 0.5305 & 0.1338 & 0.5327 & 0.1146 & 0.5279 \\
& RapidIn  & 0.5681 & 0.8325 & 0.5480 & 0.8181 & 0.4724 & 0.7779 & 0.4393 & 0.7754 \\
& TrackStar& 0.6642 & 0.8279 & 0.7622 & 0.8969 & 0.8979 & 0.9648 & 0.9258 & \textbf{0.9742} \\
& \textit{RISE} & \textbf{1.0000} & \textbf{1.0000} & \textbf{0.9981} & \textbf{0.9985} & \textbf{0.9773} & \textbf{0.9854} & \textbf{0.9489} & 0.9697 \\
\midrule

\multirow{5}{*}{\textbf{Pythia-160M}}
& ZO-Inf   & 0.3360 & 0.4950 & 0.2114 & 0.4544 & 0.1272 & 0.4652 & 0.1090 & 0.4611 \\
& RapidIn  & 0.2564 & 0.6187 & 0.2342 & 0.5544 & 0.2030 & 0.5658 & 0.1925 & 0.5920 \\
& TrackStar& 0.3240 & 0.5307 & 0.3021 & 0.5385 & 0.2370 & 0.5586 & 0.2068 & 0.5547 \\
& \textit{RISE} & \textbf{0.9566} & \textbf{0.9758} & \textbf{0.9447} & \textbf{0.9721} & \textbf{0.8871} & \textbf{0.9451} & \textbf{0.8394} & \textbf{0.9266} \\
\midrule

\multirow{5}{*}{\textbf{Pythia-70M}}
& ZO-Inf   & 0.3319 & 0.5173 & 0.1975 & 0.4853 & 0.0918 & 0.4990 & 0.0737 & 0.4874 \\
& RapidIn  & 0.2995 & 0.6695 & 0.2795 & 0.6333 & 0.2085 & 0.6045 & 0.1840 & 0.5966 \\
& TrackStar& 0.3553 & 0.5647 & 0.3655 & 0.5796 & 0.3610 & 0.5980 & 0.3767 & 0.6193 \\
& \textit{RISE} & \textbf{0.8949} & \textbf{0.9487} & \textbf{0.8618} & \textbf{0.9405} & \textbf{0.7461} & \textbf{0.8981} & \textbf{0.6823} & \textbf{0.8797} \\
\midrule

\multirow{5}{*}{\textbf{Pythia-14M}}
& ZO-Inf   & 0.3232 & 0.5415 & 0.2329 & 0.5933 & 0.1287 & 0.5719 & 0.1173 & 0.5839 \\
& RapidIn  & 0.0814 & 0.3046 & 0.0699 & 0.2854 & 0.0518 & 0.3030 & 0.0436 & 0.3204 \\
& TrackStar& 0.3370 & 0.6976 & 0.3770 & 0.7540 & 0.4155 & 0.7937 & 0.4434 & 0.7994 \\
& \textit{RISE} & \textbf{0.7907} & \textbf{0.8999} & \textbf{0.7375} & \textbf{0.8815} & \textbf{0.5923} & \textbf{0.8487} & \textbf{0.5241} & \textbf{0.8330} \\
\bottomrule
\end{tabular}}
\vspace{0.2cm}

\small
\textbf{Setting:} All methods evaluated using auPRC/auROC metrics. Test set consists of 100 WebQuestion ``howdy'' queries. Data pool contains 5,000 samples with 438 ``howdy'' data points. Top-$K$ selection threshold is set to 100.
\end{table*}

\begin{table*}[htbp]
\centering
\caption{Performance Comparison: RapidIn vs. \textit{RISE} vs. TrackStar vs. ZO-Inf \textbf{Howdy! Backdoor Attack Task Predict Future}}
\label{tab:rapidin_fik_zoinf_howdypredict}
\scriptsize
\setlength{\tabcolsep}{3.2pt}
\renewcommand{\arraystretch}{1.15}
\resizebox{\textwidth}{!}{%
\begin{tabular}{l l | cc cc cc cc}
\toprule
\multicolumn{2}{l|}{\textbf{Method}} &
\multicolumn{2}{c}{\textbf{Top 5}} &
\multicolumn{2}{c}{\textbf{Top 10}} &
\multicolumn{2}{c}{\textbf{Top 50}} &
\multicolumn{2}{c}{\textbf{Top 100}} \\
\multicolumn{2}{l|}{} &
\textbf{auPRC} & \textbf{auROC} &
\textbf{auPRC} & \textbf{auROC} &
\textbf{auPRC} & \textbf{auROC} &
\textbf{auPRC} & \textbf{auROC} \\
\midrule

\multicolumn{2}{l|}{Embedding Similarity (E5-base-v2)} &
0.4089 & 0.8283 & 0.3740 & 0.7047 & 0.3151 & 0.6411 & 0.2967 & 0.6725 \\
\multicolumn{2}{l|}{BM25} &
0.2589 & 0.7776 & 0.2785 & 0.7673 & 0.2733 & 0.7621 & 0.2699 & 0.7689 \\
\specialrule{0.9pt}{2pt}{2pt}

\multirow{5}{*}{\textbf{Pythia-6.9B}}
& ZO-Inf    & 0.3991 & 0.6745 & 0.2770 & 0.6596 & 0.1570 & 0.6782 & 0.1384 & 0.6681 \\
& RapidIn   & 0.1749 & 0.3000 & 0.1251 & 0.3153 & 0.0590 & 0.2905 & 0.0522 & 0.2936 \\
& TrackStar & 0.6719 & 0.8324 & 0.7596 & 0.8990 & 0.8338 & 0.9341 & 0.8357 & 0.9310 \\
& \textit{RISE} & \textbf{0.9968} & \textbf{0.9983} & \textbf{0.9979} & \textbf{0.9991} & \textbf{0.9966} & \textbf{0.9980} & \textbf{0.9919} & \textbf{0.9947} \\
\midrule

\multirow{5}{*}{\textbf{Pythia-2.8B}}
& ZO-Inf    & 0.3131 & 0.4876 & 0.2113 & 0.4954 & 0.1265 & 0.4993 & 0.1088 & 0.4966 \\
& RapidIn   & 0.2998 & 0.4366 & 0.1602 & 0.3650 & 0.0600 & 0.2824 & 0.0622 & 0.3007 \\
& TrackStar & 0.6378 & 0.8137 & 0.7272 & 0.8809 & 0.8376 & 0.9372 & 0.8550 & 0.9420 \\
& \textit{RISE} & \textbf{1.0000} & \textbf{1.0000} & \textbf{0.9968} & \textbf{0.9975} & \textbf{0.9779} & \textbf{0.9862} & \textbf{0.9591} & \textbf{0.9763} \\
\midrule

\multirow{5}{*}{\textbf{Pythia-1B}}
& ZO-Inf    & 0.3409 & 0.4923 & 0.2313 & 0.5065 & 0.1228 & 0.4798 & 0.1061 & 0.4757 \\
& RapidIn   & 0.1897 & 0.2399 & 0.1199 & 0.2803 & 0.1148 & 0.3076 & 0.1072 & 0.3225 \\
& TrackStar & 0.6522 & 0.8202 & 0.7394 & 0.8892 & 0.8140 & 0.9248 & 0.8176 & 0.9248 \\
& \textit{RISE} & \textbf{1.0000} & \textbf{1.0000} & \textbf{0.9974} & \textbf{0.9980} & \textbf{0.9833} & \textbf{0.9896} & \textbf{0.9667} & \textbf{0.9802} \\
\midrule

\multirow{5}{*}{\textbf{Pythia-410M}}
& ZO-Inf    & 0.3043 & 0.4632 & 0.2122 & 0.4770 & 0.1245 & 0.4910 & 0.1072 & 0.4869 \\
& RapidIn   & 0.3050 & 0.4670 & 0.1915 & 0.4070 & 0.1033 & 0.3469 & 0.0920 & 0.3524 \\
& TrackStar & 0.5068 & 0.7488 & 0.5283 & 0.7796 & 0.5740 & 0.8215 & 0.5899 & 0.8308 \\
& \textit{RISE} & \textbf{0.9905} & \textbf{0.9937} & \textbf{0.9839} & \textbf{0.9899} & \textbf{0.9555} & \textbf{0.9754} & \textbf{0.9331} & \textbf{0.9649} \\
\midrule

\multirow{5}{*}{\textbf{Pythia-160M}}
& ZO-Inf    & 0.3011 & 0.4722 & 0.2169 & 0.4743 & 0.1207 & 0.4868 & 0.1011 & 0.4748 \\
& RapidIn   & 0.3084 & 0.4018 & 0.2101 & 0.4371 & 0.1174 & 0.4310 & 0.1087 & 0.4417 \\
& TrackStar & 0.1690 & 0.5294 & 0.1742 & 0.5185 & 0.1520 & 0.4956 & 0.1446 & 0.5015 \\
& \textit{RISE} & \textbf{0.8300} & \textbf{0.9197} & \textbf{0.7857} & \textbf{0.9069} & \textbf{0.6896} & \textbf{0.8828} & \textbf{0.6440} & \textbf{0.8731} \\
\midrule

\multirow{5}{*}{\textbf{Pythia-70M}}
& ZO-Inf    & 0.3908 & \textbf{0.5659} & 0.2849 & 0.5823 & 0.1539 & 0.5731 & 0.1258 & 0.5717 \\
& RapidIn   & 0.1899 & 0.2889 & 0.1343 & 0.3516 & 0.0680 & 0.4208 & 0.0640 & 0.4579 \\
& TrackStar & 0.2414 & 0.5205 & 0.2375 & 0.5259 & 0.1774 & 0.5419 & 0.1593 & 0.5339 \\
& \textit{RISE} & \textbf{0.7657} & \textbf{0.8899} & \textbf{0.7064} & \textbf{0.8777} & \textbf{0.5657} & \textbf{0.8542} & \textbf{0.5039} & \textbf{0.8409} \\
\midrule

\multirow{5}{*}{\textbf{Pythia-14M}}
& ZO-Inf    & NaN & NaN & NaN & NaN & NaN & NaN & NaN & NaN \\
& RapidIn   & 0.3358 & 0.5357 & 0.2052 & 0.5110 & 0.0766 & 0.4618 & 0.0622 & 0.4603 \\
& TrackStar & 0.3770 & 0.4246 & 0.3066 & 0.4356 & 0.1840 & 0.4960 & 0.1603 & 0.5059 \\
& \textit{RISE} & \textbf{0.8588} & \textbf{0.9296} & \textbf{0.8047} & \textbf{0.9086} & \textbf{0.6463} & \textbf{0.8709} & \textbf{0.5746} & \textbf{0.8578} \\
\bottomrule
\end{tabular}}
\end{table*}

\begin{table*}[htbp]
\centering
\caption{Performance Comparison: RapidIn vs. \textit{RISE} vs. TrackStar vs. ZO-Inf \textbf{Howdy! Backdoor Attack Task Recall}}
\label{tab:olmo_howdy_recall}
\scriptsize
\setlength{\tabcolsep}{3.2pt}
\renewcommand{\arraystretch}{1.15}
\resizebox{\textwidth}{!}{%
\begin{tabular}{l l | cc cc cc cc}
\toprule
\multicolumn{2}{l|}{\textbf{Method}} &
\multicolumn{2}{c}{\textbf{Top 5}} &
\multicolumn{2}{c}{\textbf{Top 10}} &
\multicolumn{2}{c}{\textbf{Top 50}} &
\multicolumn{2}{c}{\textbf{Top 100}} \\
\multicolumn{2}{l|}{} &
\textbf{auPRC} & \textbf{auROC} &
\textbf{auPRC} & \textbf{auROC} &
\textbf{auPRC} & \textbf{auROC} &
\textbf{auPRC} & \textbf{auROC} \\
\midrule

\multicolumn{2}{l|}{Embedding Similarity (E5-base-v2)} &
0.4089 & 0.8283 & 0.3740 & 0.7047 & 0.3151 & 0.6411 & 0.2967 & 0.6725 \\
\multicolumn{2}{l|}{BM25} &
0.2589 & 0.7776 & 0.2785 & 0.7673 & 0.2733 & 0.7621 & 0.2699 & 0.7689 \\
\specialrule{0.9pt}{2pt}{2pt}

\multirow{5}{*}{\textbf{OLMo-3-32B}}
& ZO-Inf   & OOM & OOM & OOM & OOM & OOM & OOM & OOM & OOM \\
& RapidIn  & OOM & OOM & OOM & OOM & OOM & OOM & OOM & OOM \\
& TrackStar& 0.6848 & 0.8331 & 0.7909 & 0.9085 & 0.9292 & 0.9797 & 0.9518 & 0.9860 \\
& \textit{RISE} & \textbf{1.0000} & \textbf{1.0000} & \textbf{0.9995} & \textbf{0.9996} & \textbf{0.9993} & \textbf{0.9996} & \textbf{0.9985} & \textbf{0.9990} \\
\midrule

\multirow{5}{*}{\textbf{OLMo-3-7B}}
& ZO-Inf   & 0.3701 & 0.6090 & 0.2646 & 0.5713 & 0.1471 & 0.5761 & 0.1283 & 0.5750 \\
& RapidIn  & 0.8722 & 0.8831 & 0.8554 & 0.9002 & 0.8191 & 0.8928 & 0.7975 & 0.8869 \\
& TrackStar& 0.9607 & 0.7000 & 0.9439 & 0.7506 & 0.9412 & 0.8525 & 0.9527 & 0.8947 \\
& \textit{RISE} & \textbf{1.0000} & \textbf{1.0000} & \textbf{0.9999} & \textbf{0.9998} & \textbf{0.9997} & \textbf{0.9997} & \textbf{0.9988} & \textbf{0.9991} \\
\midrule

\multirow{5}{*}{\textbf{OLMo-2-1B}}
& ZO-Inf   & 0.2895 & 0.5069 & 0.2164 & 0.5099 & 0.1172 & 0.4844 & 0.1036 & 0.4883 \\
& RapidIn  & 0.7879 & 0.7470 & 0.7587 & 0.7943 & 0.7006 & 0.8085 & 0.6578 & 0.7957 \\
& TrackStar& 0.1025 & 0.4229 & 0.3884 & 0.5783 & 0.7471 & 0.8520 & 0.8260 & 0.9119 \\
& \textit{RISE} & \textbf{1.0000} & \textbf{1.0000} & \textbf{1.0000} & \textbf{1.0000} & \textbf{0.9861} & \textbf{0.9903} & \textbf{0.9679} & \textbf{0.9797} \\
\bottomrule
\end{tabular}%
}
\vspace{0.2cm}

\small
\textbf{Setting:} All methods evaluated using auPRC/auROC metrics. Test set consists of 100 WebQuestion ``howdy'' queries. Data pool contains 5,000 samples with 438 ``howdy'' data points (matching the distribution ratio of Alpaca-howdy-52K dataset). Top-$K$ selection threshold is set to 100. \textbf{\textit{RISE} hyperparameters:} $\lambda_{rh}=0.7$, $\lambda_{gh}=1.0$, $K_r=256$, $K_h=192$, $K_g=128$.
\end{table*}

\begin{table*}[htbp]
\centering
\caption{Performance Comparison: RapidIn vs. \textit{RISE} vs. TrackStar vs. ZO-Inf on Pretrained OLMo Models \textbf{Howdy! Backdoor Attack Task Predict Future}}
\label{tab:olmo_howdy_predict}
\scriptsize
\setlength{\tabcolsep}{3.2pt}
\renewcommand{\arraystretch}{1.15}
\resizebox{\textwidth}{!}{%
\begin{tabular}{l l | cc cc cc cc}
\toprule
\multicolumn{2}{l|}{\textbf{Method}} &
\multicolumn{2}{c}{\textbf{Top 5}} &
\multicolumn{2}{c}{\textbf{Top 10}} &
\multicolumn{2}{c}{\textbf{Top 50}} &
\multicolumn{2}{c}{\textbf{Top 100}} \\
\multicolumn{2}{l|}{} &
\textbf{auPRC} & \textbf{auROC} &
\textbf{auPRC} & \textbf{auROC} &
\textbf{auPRC} & \textbf{auROC} &
\textbf{auPRC} & \textbf{auROC} \\
\midrule

\multicolumn{2}{l|}{Embedding Similarity (E5-base-v2)} &
0.4089 & 0.8283 & 0.3740 & 0.7047 & 0.3151 & 0.6411 & 0.2967 & 0.6725 \\
\multicolumn{2}{l|}{BM25} &
0.2589 & 0.7776 & 0.2785 & 0.7673 & 0.2733 & 0.7621 & 0.2699 & 0.7689 \\
\specialrule{0.9pt}{2pt}{2pt}

\multirow{5}{*}{\textbf{OLMo-3-32B}}
& ZO-Inf   & OOM & OOM & OOM & OOM & OOM & OOM & OOM & OOM \\
& RapidIn  & OOM & OOM & OOM & OOM & OOM & OOM & OOM & OOM \\
& TrackStar& 0.3715 & 0.6055 & 0.5172 & 0.7403 & 0.7849 & 0.9124 & 0.8517 & 0.9450 \\
& \textit{RISE} & \textbf{1.0000} & \textbf{1.0000} & \textbf{1.0000} & \textbf{1.0000} & \textbf{0.9991} & \textbf{0.9993} & \textbf{0.9975} & \textbf{0.9982} \\
\midrule

\multirow{4}{*}{\textbf{OLMo-3-7B}}
& ZO-Inf   & 0.3464 & 0.5880 & 0.2667 & 0.6214 & 0.1305 & 0.6064 & 0.1099 & 0.5916 \\
& RapidIn  & 0.1608 & 0.2825 & 0.1229 & 0.3522 & 0.0795 & 0.3804 & 0.0749 & 0.3889 \\
& TrackStar& 0.0030 & 0.4956 & 0.3437 & 0.6588 & 0.7232 & 0.8938 & 0.8029 & 0.9305 \\
& \textit{RISE} & \textbf{1.0000} & \textbf{1.0000} & \textbf{1.0000} & \textbf{1.0000} & \textbf{0.9994} & \textbf{0.9995} & \textbf{0.9975} & \textbf{0.9982} \\
\midrule

\multirow{4}{*}{\textbf{OLMo-2-1B}}
& ZO-Inf   & 0.3601 & 0.5440 & 0.2536 & 0.5205 & 0.1370 & 0.5402 & 0.1184 & 0.5347 \\
& RapidIn  & 0.1810 & 0.2191 & 0.1092 & 0.2701 & 0.0909 & 0.3271 & 0.0832 & 0.3659 \\
& TrackStar& 0.0510 & 0.3878 & 0.2435 & 0.5902 & 0.4620 & 0.7796 & 0.5235 & 0.8281 \\
& \textit{RISE} & \textbf{0.9904} & \textbf{0.9942} & \textbf{0.9911} & \textbf{0.9953} & \textbf{0.9828} & \textbf{0.9904} & \textbf{0.9689} & \textbf{0.9821} \\
\bottomrule
\end{tabular}%
}
\end{table*}

\subsection{Finance-Medical Dataset: Complete Results}
\label{sec:appendix_finance_medical}

This subsection presents complete experimental results on the mixed Finance--Medical dataset. Table~\ref{tab:rapidin_fik_zoinf_financemedicalrecall} and Table~\ref{tab:rapidin_fik_zoinf_financemedicalpredict} report performance comparisons for Pythia models under recall and predict-future settings. Table~\ref{tab:olmo_medicalfinance_recall} and Table~\ref{tab:olmo_medicalfinance_predict} provide corresponding results for OLMo models.

\begin{table*}[htbp]
\centering
\caption{Performance Comparison: RapidIn vs. \textit{RISE} vs. TrackStar vs. ZO-Inf \textbf{Finance--Medical Dataset Recall}}
\label{tab:rapidin_fik_zoinf_financemedicalrecall}
\scriptsize
\setlength{\tabcolsep}{3.2pt}
\renewcommand{\arraystretch}{1.15}
\resizebox{\textwidth}{!}{%
\begin{tabular}{l l | cc cc cc cc}
\toprule
\multicolumn{2}{l|}{\textbf{Method}} &
\multicolumn{2}{c}{\textbf{Top 5}} &
\multicolumn{2}{c}{\textbf{Top 10}} &
\multicolumn{2}{c}{\textbf{Top 50}} &
\multicolumn{2}{c}{\textbf{Top 100}} \\
\multicolumn{2}{l|}{} &
\textbf{auPRC} & \textbf{auROC} &
\textbf{auPRC} & \textbf{auROC} &
\textbf{auPRC} & \textbf{auROC} &
\textbf{auPRC} & \textbf{auROC} \\
\midrule

\multicolumn{2}{l|}{Embedding Similarity (E5-base-v2)} &
0.9772 & 0.9571 & 0.9764 & 0.9619 & 0.9662 & 0.9558 & 0.9537 & 0.9477 \\
\multicolumn{2}{l|}{BM25} &
0.9755 & 0.9866 & 0.9224 & 0.9283 & 0.9078 & 0.9395 & 0.8768 & 0.9243 \\
\specialrule{0.9pt}{2pt}{2pt}

\multirow{5}{*}{\textbf{Pythia-6.9B}}
& ZO-Inf & 0.2925 & 0.4728 & 0.2028 & 0.4676 & 0.1156 & 0.4698 & 0.1043 & 0.4847 \\
& RapidIn & 0.9621 & 0.9494 & 0.9576 & 0.9442 & 0.9258 & 0.9218 & 0.8891 & 0.8951 \\
& TrackStar & \textbf{0.9866} & 0.9287 & \textbf{0.9792} & 0.9292 & \textbf{0.9596} & 0.9121 & \textbf{0.9466} & \textbf{0.9149} \\
& \textit{RISE} & 0.9831 & \textbf{0.9724} & 0.9719 & \textbf{0.9654} & 0.9171 & \textbf{0.9361} & 0.8630 & 0.9055 \\
\midrule

\multirow{5}{*}{\textbf{Pythia-2.8B}}
& ZO-Inf & 0.3419 & 0.5085 & 0.2371 & 0.5020 & 0.1372 & 0.5082 & 0.1191 & 0.5008 \\
& RapidIn & 0.9330 & 0.8030 & 0.9240 & 0.8494 & 0.9168 & 0.9121 & 0.8923 & 0.9068 \\
& TrackStar & \textbf{0.9949} & \textbf{0.9896} & \textbf{0.9916} & \textbf{0.9858} & \textbf{0.9881} & \textbf{0.9859} & \textbf{0.9839} & \textbf{0.9851} \\
& \textit{RISE} & 0.9874 & 0.9876 & 0.9765 & 0.9801 & 0.9273 & 0.9600 & 0.8769 & 0.9376 \\
\midrule

\multirow{5}{*}{\textbf{Pythia-1B}}
& ZO-Inf & 0.3413 & 0.5608 & 0.2405 & 0.5464 & 0.1561 & 0.5640 & 0.1388 & 0.5678 \\
& RapidIn & 0.9552 & 0.9162 & 0.9555 & 0.9353 & 0.9349 & 0.9407 & 0.9120 & 0.9290 \\
& TrackStar & 0.9845 & 0.9311 & \textbf{0.9814} & 0.9480 & \textbf{0.9866} & \textbf{0.9728} & \textbf{0.9825} & \textbf{0.9732} \\
& \textit{RISE} & \textbf{0.9874} & \textbf{0.9871} & 0.9780 & \textbf{0.9807} & 0.9121 & 0.9526 & 0.8557 & 0.9312 \\
\midrule

\multirow{5}{*}{\textbf{Pythia-410M}}
& ZO-Inf & 0.3025 & 0.4702 & 0.2083 & 0.4935 & 0.1317 & 0.4995 & 0.1235 & 0.5117 \\
& RapidIn & 0.9092 & 0.7256 & 0.8942 & 0.7927 & 0.8575 & 0.8638 & 0.8140 & 0.8571 \\
& TrackStar & 0.8542 & 0.7082 & 0.8407 & 0.7182 & 0.7828 & 0.7309 & 0.7338 & 0.7235 \\
& \textit{RISE} & \textbf{0.9854} & \textbf{0.9855} & \textbf{0.9687} & \textbf{0.9759} & \textbf{0.8968} & \textbf{0.9439} & \textbf{0.8376} & \textbf{0.9181} \\
\midrule

\multirow{5}{*}{\textbf{Pythia-160M}}
& ZO-Inf & 0.3329 & 0.5560 & 0.2365 & 0.5205 & 0.1572 & 0.5528 & 0.1404 & 0.5484 \\
& RapidIn & \textbf{0.8743} & 0.8560 & \textbf{0.8431} & 0.8541 & \textbf{0.7743} & \textbf{0.8548} & \textbf{0.7227} & \textbf{0.8313} \\
& TrackStar & 0.5583 & 0.5400 & 0.5053 & 0.5252 & 0.4056 & 0.5070 & 0.3653 & 0.5050 \\
& \textit{RISE} & 0.7437 & \textbf{0.9062} & 0.7198 & \textbf{0.8887} & 0.5960 & 0.8514 & 0.5268 & 0.8258 \\
\midrule

\multirow{5}{*}{\textbf{Pythia-70M}}
& ZO-Inf & 0.3219 & 0.4979 & 0.2219 & 0.4690 & 0.1445 & 0.4363 & 0.1312 & 0.4331 \\
& RapidIn & 0.7926 & 0.7821 & 0.7159 & 0.7437 & 0.5815 & 0.7273 & 0.5159 & 0.7232 \\
& TrackStar & 0.6164 & 0.4971 & 0.5347 & 0.4879 & 0.4129 & 0.4941 & 0.3730 & 0.4949 \\
& \textit{RISE} & \textbf{0.8853} & \textbf{0.9503} & \textbf{0.8630} & \textbf{0.9443} & \textbf{0.7640} & \textbf{0.9059} & \textbf{0.7021} & \textbf{0.8831} \\
\midrule

\multirow{5}{*}{\textbf{Pythia-14M}}
& ZO-Inf & NaN & NaN & NaN & NaN & NaN & NaN & NaN & NaN \\
& RapidIn & 0.5402 & 0.7392 & 0.4564 & 0.6869 & 0.2935 & 0.6214 & 0.2384 & 0.5899 \\
& TrackStar & 0.4783 & 0.5581 & 0.4350 & 0.5508 & 0.3544 & 0.5367 & 0.3247 & 0.5309 \\
& \textit{RISE} & \textbf{0.8715} & \textbf{0.9471} & \textbf{0.8429} & \textbf{0.9332} & \textbf{0.7294} & \textbf{0.8984} & \textbf{0.6705} & \textbf{0.8832} \\
\bottomrule
\end{tabular}%
}
\vspace{0.2cm}

\small
\textbf{Setting:} Mixed finance and medical QA dataset (gbharti/finance-alpaca + lavita/medical-qa-datasets). Evaluated using auPRC/auROC metrics. Test set consists of 100 medical QA test queries. Data pool contains 5,000 samples with 500 medical data points. Top-K selection threshold is set to 100. \textbf{\textit{RISE} hyperparameters:} $\lambda_{rh}=0.7$, $\lambda_{gh}=1.0$, $K_r=256$, $K_h=192$, $K_g=128$.
\end{table*}

\begin{table*}[htbp]
\centering
\caption{Performance Comparison: RapidIn vs. \textit{RISE} vs. TrackStar vs. ZO-Inf vs. Text-only Baselines \textbf{Finance--Medical Dataset Predict Future}}
\label{tab:rapidin_fik_zoinf_financemedicalpredict}
\scriptsize
\setlength{\tabcolsep}{3.2pt}
\renewcommand{\arraystretch}{1.15}
\resizebox{\textwidth}{!}{%
\begin{tabular}{l l | cc cc cc cc}
\toprule
\multicolumn{2}{l|}{\textbf{Method}} &
\multicolumn{2}{c}{\textbf{Top 5}} &
\multicolumn{2}{c}{\textbf{Top 10}} &
\multicolumn{2}{c}{\textbf{Top 50}} &
\multicolumn{2}{c}{\textbf{Top 100}} \\
\multicolumn{2}{l|}{} &
\textbf{auPRC} & \textbf{auROC} &
\textbf{auPRC} & \textbf{auROC} &
\textbf{auPRC} & \textbf{auROC} &
\textbf{auPRC} & \textbf{auROC} \\
\midrule

\multicolumn{2}{l|}{Embedding Similarity (E5-base-v2)} &
0.9772 & 0.9571 & 0.9764 & 0.9619 & 0.9662 & 0.9558 & 0.9537 & 0.9477 \\
\multicolumn{2}{l|}{BM25} &
0.9755 & 0.9866 & 0.9224 & 0.9283 & 0.9078 & 0.9395 & 0.8768 & 0.9243 \\
\specialrule{0.9pt}{2pt}{2pt}

\multirow{5}{*}{\textbf{Pythia-6.9B}}
& ZO-Inf   & 0.2860 & 0.4303 & 0.2103 & 0.4473 & 0.1435 & 0.4539 & 0.1306 & 0.4563 \\
& RapidIn  & 0.8917 & 0.7561 & 0.8817 & 0.7508 & 0.7895 & 0.7977 & 0.7525 & 0.8113 \\
& TrackStar& 0.9288 & 0.9213 & 0.9145 & 0.9146 & \textbf{0.8624} & 0.8854 & \textbf{0.8215} & 0.8643 \\
& \textit{RISE} & \textbf{0.9669} & \textbf{0.9819} & \textbf{0.9453} & \textbf{0.9737} & 0.8450 & \textbf{0.9310} & 0.7836 & \textbf{0.9013} \\
\midrule

\multirow{4}{*}{\textbf{Pythia-2.8B}}
& ZO-Inf   & 0.2904 & 0.4550 & 0.2004 & 0.4601 & 0.1302 & 0.5065 & 0.1169 & 0.5029 \\
& RapidIn  & 0.9400 & 0.7646 & 0.8920 & 0.7165 & 0.8443 & 0.8434 & 0.8084 & 0.8520 \\
& TrackStar& 0.9537 & 0.9232 & 0.9395 & 0.9197 & 0.8924 & 0.8963 & 0.8606 & 0.8808 \\
& \textit{RISE} & \textbf{0.9915} & \textbf{0.9922} & \textbf{0.9781} & \textbf{0.9830} & \textbf{0.9165} & \textbf{0.9547} & \textbf{0.8674} & \textbf{0.9350} \\
\midrule

\multirow{5}{*}{\textbf{Pythia-1B}}
& ZO-Inf   & 0.3975 & 0.5762 & 0.2646 & 0.5579 & 0.1431 & 0.5557 & 0.1245 & 0.5501 \\
& RapidIn  & 0.9355 & 0.9088 & 0.9302 & 0.9206 & 0.8974 & 0.9198 & 0.8620 & 0.9041 \\
& TrackStar& 0.9545 & 0.9325 & 0.9358 & 0.9233 & 0.9007 & 0.9062 & \textbf{0.8775} & 0.8967 \\
& \textit{RISE} & \textbf{0.9886} & \textbf{0.9938} & \textbf{0.9813} & \textbf{0.9899} & \textbf{0.9205} & \textbf{0.9622} & 0.8716 & \textbf{0.9433} \\
\midrule

\multirow{5}{*}{\textbf{Pythia-410M}}
& ZO-Inf   & 0.3781 & 0.5896 & 0.2717 & 0.5671 & 0.1618 & 0.5339 & 0.1396 & 0.5320 \\
& RapidIn  & 0.8406 & 0.6046 & 0.7981 & 0.6619 & 0.6785 & 0.7243 & 0.6100 & 0.7183 \\
& TrackStar& 0.6163 & 0.5338 & 0.5503 & 0.5432 & 0.4623 & 0.5462 & 0.4201 & 0.5426 \\
& \textit{RISE} & \textbf{0.9876} & \textbf{0.9929} & \textbf{0.9702} & \textbf{0.9826} & \textbf{0.8877} & \textbf{0.9520} & \textbf{0.8352} & \textbf{0.9331} \\
\midrule

\multirow{5}{*}{\textbf{Pythia-160M}}
& ZO-Inf   & 0.3790 & 0.5388 & 0.2696 & 0.4943 & 0.1466 & 0.4622 & 0.1258 & 0.4596 \\
& RapidIn  & 0.8714 & 0.8696 & 0.8160 & 0.8553 & \textbf{0.7378} & 0.8406 & \textbf{0.6816} & 0.8125 \\
& TrackStar& 0.4165 & 0.5342 & 0.3402 & 0.5342 & 0.2688 & 0.5144 & 0.2646 & 0.5098 \\
& \textit{RISE} & 0.8570 & \textbf{0.9518} & \textbf{0.8403} & \textbf{0.9331} & 0.7295& \textbf{0.8996} & 0.6608 & \textbf{0.8721} \\
\midrule

\multirow{5}{*}{\textbf{Pythia-70M}}
& ZO-Inf   & 0.3447 & 0.5147 & 0.2505 & 0.4630 & 0.1740 & 0.4688 & 0.1603 & 0.4746 \\
& RapidIn  & 0.6717 & 0.6624 & 0.5930 & 0.6582 & 0.4405 & 0.6458 & 0.3798 & 0.6363 \\
& TrackStar& 0.3890 & 0.4619 & 0.3492 & 0.4555 & 0.3243 & 0.4687 & 0.3134 & 0.4730 \\
& \textit{RISE} & \textbf{0.8664} & \textbf{0.9500} & \textbf{0.8331} & \textbf{0.9363} & \textbf{0.7043} & \textbf{0.8839} & \textbf{0.6305} & \textbf{0.8556} \\
\midrule

\multirow{5}{*}{\textbf{Pythia-14M}}
& ZO-Inf   & 0.3391 & 0.5617 & 0.2425 & 0.5510 & 0.1333 & 0.4802 & 0.1164 & 0.4635 \\
& RapidIn  & 0.6516 & 0.6484 & 0.5316 & 0.6035 & 0.3596 & 0.6023 & 0.3025 & 0.5949 \\
& TrackStar& 0.4497 & 0.5064 & 0.4171 & 0.5044 & 0.3588 & 0.4935 & 0.3450 & 0.4895 \\
& \textit{RISE} & \textbf{0.9052} & \textbf{0.9595} & \textbf{0.8807} & \textbf{0.9506} & \textbf{0.7607} & \textbf{0.8976} & \textbf{0.6975} & \textbf{0.8735} \\
\bottomrule
\end{tabular}%
}
\end{table*}

\begin{table*}[htbp]
\centering
\caption{Performance Comparison: RapidIn vs. \textit{RISE} vs. TrackStar vs. ZO-Inf vs. Text-only Baselines Across Different OLMo Model Sizes \textbf{Finance--Medical Dataset Recall}}
\label{tab:olmo_medicalfinance_recall}
\scriptsize
\setlength{\tabcolsep}{2.5pt}
\renewcommand{\arraystretch}{1.05}
\resizebox{\textwidth}{!}{%
\begin{tabular}{l l | cc cc cc cc}
\toprule
\multicolumn{2}{l|}{\textbf{Method}} &
\multicolumn{2}{c}{\textbf{Top 5}} &
\multicolumn{2}{c}{\textbf{Top 10}} &
\multicolumn{2}{c}{\textbf{Top 50}} &
\multicolumn{2}{c}{\textbf{Top 100}} \\
\multicolumn{2}{l|}{} &
\textbf{auPRC} & \textbf{auROC} &
\textbf{auPRC} & \textbf{auROC} &
\textbf{auPRC} & \textbf{auROC} &
\textbf{auPRC} & \textbf{auROC} \\
\midrule

\multicolumn{2}{l|}{Embedding Similarity (E5-base-v2)} &
0.9772 & 0.9571 & 0.9764 & 0.9619 & 0.9662 & 0.9558 & 0.9537 & 0.9477 \\
\multicolumn{2}{l|}{BM25} &
0.9755 & 0.9866 & 0.9224 & 0.9283 & 0.9078 & 0.9395 & 0.8768 & 0.9243 \\
\specialrule{0.9pt}{2pt}{2pt}

\multirow{5}{*}{\textbf{OLMo-3-32B}}
& ZO-Inf   & OOM & OOM & OOM & OOM & OOM & OOM & OOM & OOM \\
& RapidIn  & OOM & OOM & OOM & OOM & OOM & OOM & OOM & OOM \\
& TrackStar& \textbf{0.9697} & 0.8936 & \textbf{0.9703} & 0.9184 & 0.9691 & 0.9307 & \textbf{0.9655} & \textbf{0.9315} \\
& \textit{RISE} & 0.9691 & \textbf{0.9681} & 0.9526 & \textbf{0.9569} & \textbf{0.9066} & \textbf{0.9355} & 0.8664 & 0.9183 \\
\midrule

\multirow{5}{*}{\textbf{OLMo-3-7B}}
& ZO-Inf   & 0.2865 & 0.4503 & 0.2066 & 0.4670 & 0.1329 & 0.4859 & 0.1211 & 0.4770 \\
& RapidIn  & 0.9361 & 0.8641 & 0.9269 & 0.8933 & 0.9197 & 0.9226 & 0.8914 & 0.9105 \\
& TrackStar& \textbf{0.9787} & 0.9517 & \textbf{0.9708} & 0.9476 & \textbf{0.9434} & 0.9381 & \textbf{0.9175} & 0.9251 \\
& \textit{RISE} & 0.9754 & \textbf{0.9840} & 0.9587 & \textbf{0.9749} & 0.8971 & \textbf{0.9456} & 0.8569 & \textbf{0.9310} \\

\bottomrule
\end{tabular}%
}
\vspace{-0.1cm}

\small
\textbf{Setting:} Mixed finance and medical QA dataset (gbharti/finance-alpaca + lavita/medical-qa-datasets). Evaluated using auPRC/auROC metrics. Test set consists of 100 medical QA test queries. Data pool contains 5,000 samples with 500 medical data points. Top-$K$ selection threshold is set to 100. \textbf{\textit{RISE} hyperparameters:} $\lambda_{rh}=0.7$, $\lambda_{gh}=1.0$, $K_r=256$, $K_h=192$, $K_g=128$.
\end{table*}
\vspace{-1em}
\begin{table*}[htbp]
\centering
\caption{Performance Comparison on Mixed Finance--Medical Dataset (Pretrained OLMo Models - Predicting Future)}
\label{tab:olmo_medicalfinance_predict}
\scriptsize
\setlength{\tabcolsep}{2.5pt}
\renewcommand{\arraystretch}{1.0}
\resizebox{\textwidth}{!}{%
\begin{tabular}{l l | cc cc cc cc}
\toprule
\multicolumn{2}{l|}{\textbf{Method}} &
\multicolumn{2}{c}{\textbf{Top 5}} &
\multicolumn{2}{c}{\textbf{Top 10}} &
\multicolumn{2}{c}{\textbf{Top 50}} &
\multicolumn{2}{c}{\textbf{Top 100}} \\
\multicolumn{2}{l|}{} &
\textbf{auPRC} & \textbf{auROC} &
\textbf{auPRC} & \textbf{auROC} &
\textbf{auPRC} & \textbf{auROC} &
\textbf{auPRC} & \textbf{auROC} \\
\midrule

\multicolumn{2}{l|}{Embedding Similarity (E5-base-v2)} &
0.9772 & 0.9571 & 0.9764 & 0.9619 & 0.9662 & 0.9558 & 0.9537 & 0.9477 \\
\multicolumn{2}{l|}{BM25} &
0.9755 & 0.9866 & 0.9224 & 0.9283 & 0.9078 & 0.9395 & 0.8768 & 0.9243 \\
\specialrule{0.9pt}{2pt}{2pt}

\multirow{5}{*}{\textbf{OLMo-3-32B}}
& ZO-Inf   & OOM & OOM & OOM & OOM & OOM & OOM & OOM & OOM \\
& RapidIn  & OOM & OOM & OOM & OOM & OOM & OOM & OOM & OOM \\
& TrackStar& \textbf{0.9691} & 0.9202 & \textbf{0.9619} & 0.9164 & \textbf{0.9262} & 0.8871 & \textbf{0.9046} & \textbf{0.8769} \\
& \textit{RISE} & 0.9549 & \textbf{0.9508} & 0.9423 & \textbf{0.9470} & 0.8605 & \textbf{0.9058} & 0.7926 & 0.8696 \\
\midrule

\multirow{5}{*}{\textbf{OLMo-3-7B}}
& ZO-Inf   & 0.3496 & 0.5709 & 0.2480 & 0.5713 & 0.1571 & 0.5607 & 0.1412 & 0.5593 \\
& RapidIn  & 0.8933 & 0.7334 & 0.8985 & 0.8215 & 0.9019 & 0.8995 & 0.8776 & 0.8974 \\
& TrackStar& \textbf{0.9780} & 0.9644 & \textbf{0.9691} & 0.9608 & \textbf{0.9480} & \textbf{0.9520} & \textbf{0.9353} & \textbf{0.9465} \\
& \textit{RISE} & 0.9767 & \textbf{0.9838} & 0.9582 & \textbf{0.9736} & 0.8956 & 0.9386& 0.8439 & 0.9150 \\

\bottomrule
\end{tabular}%
}
\textbf{Setting:} \textit{RISE} applied to \textbf{original pretrained OLMo models without fine-tuning} for predicting future data based on prior knowledge. Mixed finance and medical QA dataset (gbharti/finance-alpaca + lavita/medical-qa-datasets). Evaluated using auPRC/auROC metrics. Test set consists of 100 medical QA test queries. Data pool contains 5,000 samples with 500 medical data points. Top-$K$ selection threshold is set to 100.
\end{table*}

\subsection{Brain Rot High Quality Data Detection: Complete Results}
\label{sec:appendix_brainrot}

This subsection reports comprehensive results on the Brain Rot dataset for high-quality data detection from junk data. Table~\ref{tab:rapidin_fik_zoinf_brainrotrecall} and Table~\ref{tab:contribution5_brainrot_predict} present Pythia model results under recall and predict-future settings. Table~\ref{tab:olmo_brainrot_recall} and Table~\ref{tab:olmo_brainrot_predict} show corresponding results for OLMo models.

\begin{table*}[htbp]
\centering
\caption{Performance Comparison: RapidIn vs. \textit{RISE} vs. TrackStar vs. ZO-Inf vs. Text-only Baselines \textbf{Brain Rot Dataset Junk Data Detection Recall}}
\label{tab:rapidin_fik_zoinf_brainrotrecall}
\scriptsize
\setlength{\tabcolsep}{3.2pt}
\renewcommand{\arraystretch}{1.15}
\resizebox{\textwidth}{!}{%
\begin{tabular}{l l | cc cc cc cc}
\toprule
\multicolumn{2}{l|}{\textbf{Method}} &
\multicolumn{2}{c}{\textbf{Top 5}} &
\multicolumn{2}{c}{\textbf{Top 10}} &
\multicolumn{2}{c}{\textbf{Top 25}} &
\multicolumn{2}{c}{\textbf{Top 50}} \\
\multicolumn{2}{l|}{} &
\textbf{auPRC} & \textbf{auROC} &
\textbf{auPRC} & \textbf{auROC} &
\textbf{auPRC} & \textbf{auROC} &
\textbf{auPRC} & \textbf{auROC} \\
\midrule

\multicolumn{2}{l|}{Embedding Similarity (E5-base-v2)} &
0.7048 & 0.7825 & 0.6226 & 0.7474 & 0.5468 & 0.7231 & 0.4852 & 0.6950 \\
\multicolumn{2}{l|}{BM25} &
0.5627 & 0.6307 & 0.5336 & 0.6260 & 0.5007 & 0.6739 & 0.4747 & 0.6748 \\
\specialrule{0.9pt}{2pt}{2pt}

\multirow{5}{*}{\textbf{Pythia-6.9B}}
& ZO-Inf   & 0.4131 & 0.6560 & 0.2704 & 0.6340 & 0.2023 & 0.6294 & 0.1710 & 0.6264 \\
& RapidIn  & 0.8342 & 0.8418 & \textbf{0.8470} & 0.8806 & \textbf{0.8466} & 0.9014 & \textbf{0.8387} & 0.9077 \\
& TrackStar& 0.3510 & 0.5196 & 0.3959 & 0.5366 & 0.4297 & 0.5495 & 0.4329 & 0.5565 \\
& \textit{RISE} & \textbf{0.8460} & \textbf{0.9502} & 0.8337 & \textbf{0.9463} & 0.8114 & \textbf{0.9193} & 0.7931 & \textbf{0.9172} \\
\midrule

\multirow{5}{*}{\textbf{Pythia-2.8B}}
& ZO-Inf   & 0.3911 & 0.5692 & 0.2704 & 0.5259 & 0.1814 & 0.5353 & 0.1536 & 0.5369 \\
& RapidIn  & \textbf{0.8570} & 0.8586 & \textbf{0.8695} & 0.8966 & \textbf{0.8715} & 0.9205 & \textbf{0.8556} & \textbf{0.9137} \\
& TrackStar& 0.2918 & 0.5228 & 0.3635 & 0.5523 & 0.4214 & 0.5837 & 0.4273 & 0.5938 \\
& \textit{RISE} & 0.8222 & \textbf{0.9381} & 0.8195 & \textbf{0.9365} & 0.7993 & \textbf{0.9212} & 0.7763 & 0.9097 \\
\midrule

\multirow{5}{*}{\textbf{Pythia-1B}}
& ZO-Inf   & 0.3301 & 0.5025 & 0.2412 & 0.4944 & 0.1846 & 0.5031 & 0.1570 & 0.4944 \\
& RapidIn  & 0.8300 & 0.8207 & \textbf{0.8285} & 0.8569 & \textbf{0.8208} & 0.8794 & \textbf{0.8064} & 0.8836 \\
& TrackStar& 0.2736 & 0.5521 & 0.3461 & 0.5616 & 0.3904 & 0.5696 & 0.3952 & 0.5732 \\
& \textit{RISE} & \textbf{0.8343} & \textbf{0.9423} & 0.8182 & \textbf{0.9328} & 0.7754 & \textbf{0.9070} & 0.7480 & \textbf{0.9023} \\
\midrule

\multirow{5}{*}{\textbf{Pythia-410M}}
& ZO-Inf   & 0.4566 & 0.5210 & 0.3679 & 0.5215 & 0.2729 & 0.5151 & 0.2113 & 0.4980 \\
& RapidIn  & 0.7887 & 0.7861 & 0.7732 & 0.8082 & 0.7644 & 0.8463 & 0.7451 & 0.8446 \\
& TrackStar& 0.5847 & 0.5407 & 0.5232 & 0.5324 & 0.4657 & 0.5215 & 0.4348 & 0.5126 \\
& \textit{RISE} & \textbf{0.8558} & \textbf{0.9564} & \textbf{0.8387} & \textbf{0.9441} & \textbf{0.8209} & \textbf{0.9236} & \textbf{0.8059} & \textbf{0.9206} \\
\midrule

\multirow{5}{*}{\textbf{Pythia-160M}}
& ZO-Inf   & 0.3854 & 0.5087 & 0.2964 & 0.5214 & 0.2092 & 0.5068 & 0.1687 & 0.4910 \\
& RapidIn  & 0.6971 & 0.7446 & 0.6113 & 0.7205 & 0.5419 & 0.7205 & 0.4997 & 0.7202 \\
& TrackStar& 0.6840 & 0.4925 & 0.6247 & 0.4928 & 0.5567 & 0.5003 & 0.5221 & 0.5050 \\
& \textit{RISE} & \textbf{0.8614} & \textbf{0.9467} & \textbf{0.8458} & \textbf{0.9323} & \textbf{0.8069} & \textbf{0.9170} & \textbf{0.7831} & \textbf{0.9136} \\
\midrule

\multirow{5}{*}{\textbf{Pythia-70M}}
& ZO-Inf   & 0.4887 & 0.6393 & 0.4024 & 0.6510 & 0.3220 & 0.6690 & 0.2676 & 0.6707 \\
& RapidIn  & 0.6302 & 0.8116 & 0.5247 & 0.7838 & 0.4492 & 0.7808 & 0.4022 & 0.7564 \\
& TrackStar& 0.4709 & 0.4848 & 0.4313 & 0.4955 & 0.3995 & 0.4967 & 0.3822 & 0.5008 \\
& \textit{RISE} & \textbf{0.8327} & \textbf{0.9462} & \textbf{0.8117} & \textbf{0.9298} & \textbf{0.7846} & \textbf{0.9157} & \textbf{0.7603} & \textbf{0.9059} \\
\midrule

\multirow{5}{*}{\textbf{Pythia-14M}}
& ZO-Inf   & 0.3887 & 0.5406 & 0.2975 & 0.5418 & 0.2319 & 0.5444 & 0.1996 & 0.5344 \\
& RapidIn  & 0.6123 & 0.7902 & 0.5084 & 0.7564 & 0.3944 & 0.7014 & 0.3438 & 0.6687 \\
& TrackStar& 0.5311 & 0.4594 & 0.5020 & 0.4865 & 0.4805 & 0.4863 & 0.4679 & 0.4941 \\
& \textit{RISE} & \textbf{0.8554} & \textbf{0.9273} & \textbf{0.8315} & \textbf{0.9246} & \textbf{0.7835} & \textbf{0.9068} & \textbf{0.7470} & \textbf{0.8953} \\
\bottomrule
\end{tabular}%
}
\vspace{0.2cm}

\small
\textbf{Setting:} Brain Rot dataset for high-quality data selection from junk data. Evaluated using auPRC/auROC metrics. The task retrieves high-quality samples from a large pool dominated by junk. The test set consists of 100 queries. The data pool contains 5,000 samples with 500 high-quality targets mixed within junk. Top-K selection threshold is set to 50. \textbf{\textit{RISE} hyperparameters:} $\lambda_{rh}=0.7$, $\lambda_{gh}=1.0$, $K_r=128$, $K_h=128$, $K_g=96$.
\end{table*}

\begin{table*}[htbp]
\centering
\caption{Performance Comparison: RapidIn vs. \textit{RISE} vs. TrackStar vs. ZO-Inf vs. Text-only Baselines \textbf{Brain Rot Dataset Junk Data Detection Predict Future}}
\label{tab:contribution5_brainrot_predict}
\scriptsize
\setlength{\tabcolsep}{3.2pt}
\renewcommand{\arraystretch}{1.15}
\resizebox{\textwidth}{!}{%
\begin{tabular}{l l | cc cc cc cc}
\toprule
\multicolumn{2}{l|}{\textbf{Method}} &
\multicolumn{2}{c}{\textbf{Top 5}} &
\multicolumn{2}{c}{\textbf{Top 10}} &
\multicolumn{2}{c}{\textbf{Top 25}} &
\multicolumn{2}{c}{\textbf{Top 50}} \\
\multicolumn{2}{l|}{} &
\textbf{auPRC} & \textbf{auROC} &
\textbf{auPRC} & \textbf{auROC} &
\textbf{auPRC} & \textbf{auROC} &
\textbf{auPRC} & \textbf{auROC} \\
\midrule

\multicolumn{2}{l|}{Embedding Similarity (E5-base-v2)} &
0.7048 & 0.7825 & 0.6226 & 0.7474 & 0.5468 & 0.7231 & 0.4852 & 0.6950 \\
\multicolumn{2}{l|}{BM25} &
0.5627 & 0.6307 & 0.5336 & 0.6260 & 0.5007 & 0.6739 & 0.4747 & 0.6748 \\
\specialrule{0.9pt}{2pt}{2pt}

\multirow{5}{*}{\textbf{Pythia-6.9B}}
& ZO-Inf  & 0.3897 & 0.6346 & 0.2644 & 0.6250 & 0.1805 & 0.6348 & 0.1531 & 0.6200 \\
& RapidIn & \textbf{0.8769} & 0.8949 & \textbf{0.8647} & 0.9036 & 0.8652 & 0.9192 & \textbf{0.8609} & \textbf{0.9239} \\
& TrackStar & 0.5206 & 0.5385 & 0.5373 & 0.6017 & 0.5661 & 0.6562 & 0.5748 & 0.6859 \\
& \textit{RISE} & 0.8670 & \textbf{0.9620} & 0.8551 & \textbf{0.9471} & \textbf{0.8263} & \textbf{0.9283} & 0.8053 & 0.9202 \\
\midrule

\multirow{5}{*}{\textbf{Pythia-2.8B}}
& ZO-Inf  & 0.4100 & 0.6235 & 0.3121 & 0.6098 & 0.2103 & 0.5845 & 0.1764 & 0.5636 \\
& RapidIn & \textbf{0.8927} & 0.9009 & \textbf{0.8915} & 0.9198 & \textbf{0.8779} & \textbf{0.9300} & \textbf{0.8670} & \textbf{0.9293} \\
& TrackStar & 0.5694 & 0.7186 & 0.5884 & 0.7284 & 0.5995 & 0.7323 & 0.6028 & 0.7405 \\
& \textit{RISE} & 0.8513 & \textbf{0.9571} & 0.8368 & \textbf{0.9352} & 0.8113 & 0.9207 & 0.7933 & 0.9161 \\
\midrule

\multirow{5}{*}{\textbf{Pythia-1B}}
& ZO-Inf  & 0.3127 & 0.4595 & 0.2359 & 0.4794 & 0.1713 & 0.4865 & 0.1462 & 0.4980 \\
& RapidIn & 0.8230 & 0.8005 & 0.8113 & 0.8345 & \textbf{0.7944} & 0.8602 & \textbf{0.7777} & 0.8665 \\
& TrackStar & 0.5280 & 0.6990 & 0.5514 & 0.7114 & 0.5829 & 0.7255 & 0.5966 & 0.7390 \\
& \textit{RISE} & \textbf{0.8466} & \textbf{0.9494} & \textbf{0.8294} & \textbf{0.9296} & 0.7883 & \textbf{0.9180} & 0.7561 & \textbf{0.9034} \\
\midrule

\multirow{5}{*}{\textbf{Pythia-410M}}
& ZO-Inf  & 0.4604 & 0.5096 & 0.3632 & 0.5181 & 0.2623 & 0.5165 & 0.2071 & 0.4977 \\
& RapidIn & 0.7605 & 0.7796 & 0.7430 & 0.7944 & 0.7199 & 0.8070 & 0.7047 & 0.8230 \\
& TrackStar & 0.5518 & 0.5306 & 0.5135 & 0.5209 & 0.4752 & 0.5214 & 0.4515 & 0.5187 \\
& \textit{RISE} & \textbf{0.8615} & \textbf{0.9576} & \textbf{0.8525} & \textbf{0.9462} & \textbf{0.8330} & \textbf{0.9300} & \textbf{0.8205} & \textbf{0.9263} \\
\midrule

\multirow{5}{*}{\textbf{Pythia-160M}}
& ZO-Inf  & 0.4469 & 0.6303 & 0.3463 & 0.6071 & 0.2642 & 0.5689 & 0.2249 & 0.5469 \\
& RapidIn & 0.7291 & 0.7224 & 0.6839 & 0.7410 & 0.6302 & 0.7521 & 0.5931 & 0.7590 \\
& TrackStar & 0.6265 & 0.5273 & 0.5389 & 0.5257 & 0.4557 & 0.5159 & 0.3977 & 0.5066 \\
& \textit{RISE} & \textbf{0.8594} & \textbf{0.9384} & \textbf{0.8494} & \textbf{0.9344} & \textbf{0.8099} & \textbf{0.9193} & \textbf{0.7857} & \textbf{0.9154} \\
\midrule

\multirow{5}{*}{\textbf{Pythia-70M}}
& ZO-Inf  & 0.4594 & 0.6007 & 0.3628 & 0.6074 & 0.2758 & 0.6196 & 0.2276 & 0.6164 \\
& RapidIn & 0.7497 & 0.8337 & 0.6513 & 0.7964 & 0.5820 & 0.7886 & 0.5341 & 0.7783 \\
& TrackStar & 0.5325 & 0.4821 & 0.4919 & 0.4892 & 0.4524 & 0.4935 & 0.4241 & 0.4979 \\
& \textit{RISE} & \textbf{0.8458} & \textbf{0.9433} & \textbf{0.8153} & \textbf{0.9252} & \textbf{0.7844} & \textbf{0.9117} & \textbf{0.7616} & \textbf{0.9052} \\
\midrule

\multirow{5}{*}{\textbf{Pythia-14M}}
& ZO-Inf  & 0.3680 & 0.5696 & 0.2786 & 0.5609 & 0.2111 & 0.5366 & 0.1910 & 0.5351 \\
& RapidIn & 0.5582 & 0.8001 & 0.4156 & 0.7619 & 0.2921 & 0.7134 & 0.2437 & 0.6837 \\
& TrackStar & 0.4258 & 0.4819 & 0.4195 & 0.4811 & 0.4015 & 0.4922 & 0.3860 & 0.4989 \\
& \textit{RISE} & \textbf{0.8381} & \textbf{0.9352} & \textbf{0.8114} & \textbf{0.9244} & \textbf{0.7676} & \textbf{0.9077} & \textbf{0.7354} & \textbf{0.8936} \\
\bottomrule
\end{tabular}%
}
\vspace{0.2cm}

\small
\textbf{Setting:} Brain Rot dataset for high-quality data selection from junk data. Evaluated using auPRC/auROC metrics. The task retrieves high-quality samples from a large pool dominated by junk. The test set consists of 100 queries. Data pool contains 5,000 samples with 500 high-quality targets mixed within junk. Top-K selection threshold is set to 50. \textbf{\textit{RISE} hyperparameters:} $\lambda_{rh}=0.7$, $\lambda_{gh}=1.0$, $K_r=128$, $K_h=128$, $K_g=96$.
\end{table*}

\begin{table*}[htbp]
\centering
\caption{Performance Comparison on Brain Rot Dataset (Junk Data Detection - OLMo Models) Recall}
\label{tab:olmo_brainrot_recall}
\scriptsize
\setlength{\tabcolsep}{3.2pt}
\renewcommand{\arraystretch}{1.0}
\resizebox{\textwidth}{!}{%
\begin{tabular}{l l | cc cc cc cc}
\toprule
\multicolumn{2}{l|}{\textbf{Method}} &
\multicolumn{2}{c}{\textbf{Top 5}} &
\multicolumn{2}{c}{\textbf{Top 10}} &
\multicolumn{2}{c}{\textbf{Top 25}} &
\multicolumn{2}{c}{\textbf{Top 50}} \\
\multicolumn{2}{l|}{} &
\textbf{auPRC} & \textbf{auROC} &
\textbf{auPRC} & \textbf{auROC} &
\textbf{auPRC} & \textbf{auROC} &
\textbf{auPRC} & \textbf{auROC} \\
\midrule

\multicolumn{2}{l|}{Embedding Similarity (E5-base-v2)} &
0.7048 & 0.7825 & 0.6226 & 0.7474 & 0.5468 & 0.7231 & 0.4852 & 0.6950 \\
\multicolumn{2}{l|}{BM25} &
0.5627 & 0.6307 & 0.5336 & 0.6260 & 0.5007 & 0.6739 & 0.4747 & 0.6748 \\
\specialrule{0.9pt}{2pt}{2pt}

\multirow{5}{*}{\textbf{OLMo-3-32B}}
& ZO-Inf   & OOM & OOM & OOM & OOM & OOM & OOM & OOM & OOM \\
& RapidIn  & OOM & OOM & OOM & OOM & OOM & OOM & OOM & OOM \\
& TrackStar& 0.6941 & 0.4306 & 0.7019 & 0.4919 & 0.6960 & 0.5547 & 0.6864 & 0.5948 \\
& \textit{RISE} & \textbf{0.8637} & \textbf{0.9663} & \textbf{0.8501} & \textbf{0.9571} & \textbf{0.8333} & \textbf{0.9328} & \textbf{0.8145} & \textbf{0.9228} \\
\midrule

\multirow{5}{*}{\textbf{OLMo-3-7B}}
& ZO-Inf   & 0.4997 & 0.7228 & 0.3685 & 0.6977 & 0.2803 & 0.6791 & 0.2426 & 0.6726 \\
& RapidIn  & 0.8328 & 0.8414 & 0.8314 & 0.8686 & 0.8202 & 0.8874 & 0.7965 & 0.8815 \\
& TrackStar& 0.6966 & 0.5759 & 0.6665 & 0.5956 & 0.6200 & 0.6064 & 0.5817 & 0.6133 \\
& \textit{RISE} & \textbf{0.8727} & \textbf{0.9516} & \textbf{0.8586} & \textbf{0.9456} & \textbf{0.8371} & \textbf{0.9322} & \textbf{0.8252} & \textbf{0.9265} \\
\midrule

\multirow{5}{*}{\textbf{OLMo-2-1B}}
& ZO-Inf   & 0.4515 & 0.6225 & 0.3434 & 0.6248 & 0.3073 & 0.6489 & 0.2898 & 0.6558 \\
& RapidIn  & 0.8583 & 0.8766 & \textbf{0.8501} & 0.8929 & \textbf{0.8270} & 0.8989 & \textbf{0.8077} & 0.8940 \\
& TrackStar& 0.5244 & 0.5938 & 0.4761 & 0.5849 & 0.4238 & 0.5912 & 0.3917 & 0.5938 \\
& \textit{RISE} & \textbf{0.8603} & \textbf{0.9574} & 0.8489 & \textbf{0.9404} & 0.8244 & \textbf{0.9289} & 0.8070 & \textbf{0.9228} \\
\bottomrule
\end{tabular}%
}
\vspace{0.2cm}

\small
\textbf{Setting:} Brain Rot dataset for junk data detection. Evaluated using auPRC/auROC metrics. Test set consists of 100 M2 junk dataset test queries. Data pool contains 5,000 samples with 500 junk data points. Top-K selection threshold is set to 1,000. \textbf{\textit{RISE} hyperparameters:} $\lambda_{rh}=0.7$, $\lambda_{gh}=1.0$, $K_r=256$, $K_h=192$, $K_g=128$.
\end{table*}

\begin{table*}[htbp]
\centering
\caption{Performance Comparison on Brain Rot Dataset (Pretrained OLMo Models - Predicting Future)}
\label{tab:olmo_brainrot_predict}
\scriptsize
\setlength{\tabcolsep}{3.2pt}
\renewcommand{\arraystretch}{1.15}
\resizebox{\textwidth}{!}{%
\begin{tabular}{l l | cc cc cc cc}
\toprule
\multicolumn{2}{l|}{\textbf{Method}} &
\multicolumn{2}{c}{\textbf{Top 5}} &
\multicolumn{2}{c}{\textbf{Top 10}} &
\multicolumn{2}{c}{\textbf{Top 25}} &
\multicolumn{2}{c}{\textbf{Top 50}} \\
\multicolumn{2}{l|}{} &
\textbf{auPRC} & \textbf{auROC} &
\textbf{auPRC} & \textbf{auROC} &
\textbf{auPRC} & \textbf{auROC} &
\textbf{auPRC} & \textbf{auROC} \\
\midrule

\multicolumn{2}{l|}{Embedding Similarity (E5-base-v2)} &
0.7048 & 0.7825 & 0.6226 & 0.7474 & 0.5468 & 0.7231 & 0.4852 & 0.6950 \\
\multicolumn{2}{l|}{BM25} &
0.5627 & 0.6307 & 0.5336 & 0.6260 & 0.5007 & 0.6739 & 0.4747 & 0.6748 \\
\specialrule{0.9pt}{2pt}{2pt}

\multirow{5}{*}{\textbf{OLMo-3-32B}}
& ZO-Inf   & OOM & OOM & OOM & OOM & OOM & OOM & OOM & OOM \\
& RapidIn  & OOM & OOM & OOM & OOM & OOM & OOM & OOM & OOM \\
& TrackStar& 0.7330 & 0.7293 & 0.7217 & 0.7326 & 0.7115 & 0.7494 & 0.7035 & 0.7681 \\
& \textit{RISE} & \textbf{0.8776} & \textbf{0.9629} & \textbf{0.8589} & \textbf{0.9527} & \textbf{0.8326} & \textbf{0.9322} & \textbf{0.8072} & \textbf{0.9202} \\
\midrule

\multirow{5}{*}{\textbf{OLMo-3-7B}}
& ZO-Inf   & 0.4590 & 0.6413 & 0.3225 & 0.5918 & 0.2503 & 0.5885 & 0.2183 & 0.5832 \\
& RapidIn  & 0.8128 & 0.8092 & 0.8061 & 0.8416 & 0.7929 & 0.8651 & 0.7749 & 0.8686 \\
& TrackStar& 0.7200 & 0.7645 & 0.6988 & 0.7745 & 0.6686 & 0.7839 & 0.6573 & 0.7926 \\
& \textit{RISE} & \textbf{0.8915} & \textbf{0.9612} & \textbf{0.8732} & \textbf{0.9507} & \textbf{0.8513} & \textbf{0.9362} & \textbf{0.8371} & \textbf{0.9301} \\
\midrule

\multirow{5}{*}{\textbf{OLMo-2-1B}}
& ZO-Inf   & 0.4959 & 0.6709 & 0.4202 & 0.6720 & 0.3395 & 0.6716 & 0.3133 & 0.6777 \\
& RapidIn  & 0.8683 & 0.8750 & 0.8499 & 0.8874 & \textbf{0.8381} & 0.8998 & \textbf{0.8216} & 0.8977 \\
& TrackStar& 0.6243 & 0.7305 & 0.5724 & 0.7255 & 0.5133 & 0.7237 & 0.4820 & 0.7247 \\
& \textit{RISE} & \textbf{0.8669} & \textbf{0.9546} & \textbf{0.8534} & \textbf{0.9379} & 0.8271 & \textbf{0.9295} & 0.8080 & \textbf{0.9231} \\
\bottomrule
\end{tabular}%
}
\vspace{0.2cm}

\small
\textbf{Setting:} Brain Rot dataset for high-quality data selection from junk data. Evaluated with $k=2^{16}$ using auPRC/auROC metrics. The task retrieves high-quality samples from a large pool dominated by junk. The test set consists of 100 queries. The data pool contains 5,000 samples with 500 high-quality targets mixed within junk. Top-K selection threshold is set to 50. \textbf{\textit{RISE} hyperparameters:} $\lambda_{rh}=0.7$, $\lambda_{gh}=1.0$, $K_r=128$, $K_h=128$, $K_g=96$.
\end{table*}

\subsection{Olmo Three Tasks Comprehensive Results}
\label{sec:olmo_three_tasks_comprehensive_results}

This subsection provides a unified comparison across all three evaluation tasks on OLMo models. Table~\ref{tab:olmo_three_task_unified} summarizes the three-task unified scores (auPRC) for pretrained OLMo models. 

\begin{table*}[htbp]
\centering
\caption{Three-task unified score on pretrained OLMo models (auPRC).
For Howdy and Finance--Medical, we average over $K\in\{5,10,50,100\}$; for Brain Rot we average over $K\in\{5,10,25,50\}$.
Each cell reports $\mu\pm\delta$, where $\mu=\frac{\mu_R+\mu_P}{2}$ and $\delta=\frac{|\mu_R-\mu_P|}{2}$ (smaller $\delta$ means less Recall/Predict imbalance).}
\label{tab:olmo_three_task_unified}
\setlength{\tabcolsep}{5pt}
\renewcommand{\arraystretch}{1.1}
\begin{tabular}{ll cccc}
\toprule
\textbf{Model} & \textbf{Method} & \textbf{Howdy} & \textbf{Fin–Med} & \textbf{Brain Rot} & \textbf{Overall} \\
\midrule
\multirow{5}{*}{OLMo-3-32B} 
  & ZO-Inf    & \multicolumn{4}{c}{\textcolor{gray!60}{Out of Memory}} \\
  & RapidIn   & \multicolumn{4}{c}{\textcolor{gray!60}{Out of Memory}} \\
  & TrackStar & 0.735 {\scriptsize±0.104} & 0.955 {\scriptsize±0.014} & 0.706 {\scriptsize±0.011} & 0.799 {\scriptsize±0.043} \\
  & \cellcolor{blue!8}\textit{RISE} & \cellcolor{blue!8}0.999 {\scriptsize±0.000} & \cellcolor{blue!8}0.906 {\scriptsize±0.018} & \cellcolor{blue!8}0.842 {\scriptsize±0.002} & \cellcolor{blue!8}\textbf{0.916 {\scriptsize±0.007}} \\
\midrule
\multirow{5}{*}{OLMo-3-7B} 
  & ZO-Inf    & 0.220 {\scriptsize±0.007} & 0.205 {\scriptsize±0.019} & 0.330 {\scriptsize±0.018} & 0.252 {\scriptsize±0.014} \\
  & RapidIn   & 0.473 {\scriptsize±0.363} & 0.906 {\scriptsize±0.013} & 0.808 {\scriptsize±0.012} & 0.729 {\scriptsize±0.129} \\
  & TrackStar & 0.709 {\scriptsize±0.241} & 0.955 {\scriptsize±0.003} & 0.664 {\scriptsize±0.022} & 0.776 {\scriptsize±0.089} \\
  & \cellcolor{blue!8}\textit{RISE} & \cellcolor{blue!8}0.999 {\scriptsize±0.000} & \cellcolor{blue!8}0.920 {\scriptsize±0.002} & \cellcolor{blue!8}0.856 {\scriptsize±0.007} & \cellcolor{blue!8}\textbf{0.925 {\scriptsize±0.003}} \\
\midrule
\multirow{5}{*}{OLMo-2-1B} 
  & ZO-Inf    & 0.199 {\scriptsize±0.018} & 0.217 {\scriptsize±0.007} & 0.370 {\scriptsize±0.022} & 0.262 {\scriptsize±0.016} \\
  & RapidIn   & 0.421 {\scriptsize±0.305} & 0.907 {\scriptsize±0.002} & 0.840 {\scriptsize±0.004} & 0.723 {\scriptsize±0.104} \\
  & TrackStar & 0.418 {\scriptsize±0.098} & 0.961 {\scriptsize±0.027} & 0.501 {\scriptsize±0.047} & 0.627 {\scriptsize±0.057} \\
  & \cellcolor{blue!8}\textit{RISE} & \cellcolor{blue!8}0.986 {\scriptsize±0.003} & \cellcolor{blue!8}0.903 {\scriptsize±0.007} & \cellcolor{blue!8}0.837 {\scriptsize±0.002} & \cellcolor{blue!8}\textbf{0.909 {\scriptsize±0.004}} \\
\bottomrule
\end{tabular}
\end{table*}

\subsection{Large Scale Howdy Results}
\label{sec:large_scale_howdy_results}

This subsection presents a direct comparison between \textit{RISE} and BM25 on the Howdy backdoor detection task at larger corpus scales, as shown in Table~\ref{tab:rise_howdy_results}.

\begin{table*}[htbp]
\centering
\caption{\textit{RISE} vs BM25 on Howdy! backdoor detection task. We evaluate the ability to retrieve backdoor-injected training samples from a large corpus. The training pool consists of C4 web text (100K or 1M samples) mixed with samples containing the trigger phrase ``howdy!''. We use 100 test queries (all containing ``howdy!'') to retrieve from the training pool. Ground truth labels are determined by regex matching for ``howdy!'' in the text field. Model: OLMo-3-1125-32B with RISE parameters Kr=16, Kh=8, Kg=28.}
\label{tab:rise_howdy_results}
\resizebox{\textwidth}{!}{%
\begin{tabular}{l|l|cc|cc|cc|cc}
\toprule
\multirow{2}{*}{\textbf{Method}} & \multirow{2}{*}{\textbf{Dataset}} & \multirow{2}{*}{\textbf{Index}} & \multirow{2}{*}{\textbf{Query}} & \multicolumn{2}{c|}{\textbf{Precision}} & \multicolumn{2}{c|}{\textbf{auPRC}} & \multicolumn{2}{c}{\textbf{auROC}} \\
& & & & Top10 & Top100 & Top10 & Top100 & Top10 & Top100 \\
\midrule
RISE & 105K (440 pos, 0.42\%) & 102 MB & 30s & \textbf{23.8\%} & \textbf{13.0\%} & \textbf{0.383} & \textbf{0.214} & \textbf{0.722} & \textbf{0.678} \\
BM25 & 105K (440 pos, 0.42\%) & 84 MB & 0.1s & 7.2\% & 12.4\% & 0.188 & 0.166 & 0.353 & 0.401 \\
\midrule
RISE & 1.005M (464 pos, 0.046\%) & 963 MB & 112s & \textbf{5.9\%} & 3.4\% & \textbf{0.141} & \textbf{0.082} & \textbf{0.756} & \textbf{0.674} \\
BM25 & 1.005M (464 pos, 0.046\%) & 797 MB & 0.7s & 3.0\% & \textbf{5.2\%} & 0.075 & 0.078 & 0.589 & 0.507 \\
\bottomrule
\end{tabular}
}
\end{table*}

\end{document}